\documentclass[twoside]{article}

 \usepackage[preprint]{aistats2027}
\usepackage{amsmath,amssymb,amsfonts,amsthm}
\usepackage{mathtools}
\usepackage{bm}
\usepackage{bbm}
\usepackage{booktabs}
\usepackage{graphicx}
\usepackage{subcaption}
\usepackage{enumitem}
 \usepackage{natbib}
\usepackage{microtype}
\usepackage{hyperref}
\usepackage{tikz}
\usetikzlibrary{arrows.meta,calc,angles,quotes}
\usepackage{algorithm}
\usepackage{algpseudocode}

\newtheorem{theorem}{Theorem}
\newtheorem{proposition}[theorem]{Proposition}
\newtheorem{lemma}[theorem]{Lemma}
\newtheorem{corollary}[theorem]{Corollary}
\newtheorem{assumption}[theorem]{Assumption}
\theoremstyle{definition}

\newcommand{\R}{\mathbb R}
\newcommand{\E}{\mathbb E}
\newcommand{\Prob}{\mathbb P}

\newcommand{\1}{\mathbbm 1}

\begin{document}

%

%

\twocolumn[

\aistatstitle{Elicitation and Decision Geometry in Single-Index Bandits}

\aistatsauthor{Sakshi Arya \And Cheng Soon Ong }

\aistatsaddress{Case Western Reserve University, USA \And  CSIRO, Australia } ]

\begin{abstract}
We study two-arm contextual bandits with arm-specific single indices and a shared unknown monotone link. Monotonicity makes the optimal action depend only on the contrast between the index directions, hence arm-specific reward functions need not be estimated. We introduce \emph{Natural Boundary Learning} (NBL), a greedy procedure that uses a sequential Stein contrast to learn the optimal boundary directly, without estimating the reward functions or the common link. We characterize the local Riemannian dynamics of NBL through a \emph{decision stability coefficient} balancing arm separation, link geometry, and the context distribution. We show that this stability is connected to the elicitation geometry of the underlying convex potential. Under local decision stability, NBL contracts toward the optimal boundary and achieves $O(\log n)$ expected regret. Numerical experiments illustrate the predicted stability regimes
and compare NBL with a parametric greedy benchmark under link misspecification.
\end{abstract}
\section{Introduction}
Contextual bandits seek to make effective sequential decisions by
learning from observed contexts and rewards obtained under previous
actions. A central challenge is to balance two competing goals: imposing
enough structure to learn efficiently while retaining enough flexibility
to accommodate complex reward mechanisms. Linear and generalized linear
bandits impose parametric reward models that enable efficient learning
\cite{goldenshluger2013linear,Filippi2010_GLM_Bandits}, but their performance can be sensitive to model misspecification \cite{ghosh2017misspecified}. At the other end, nonparametric contextual bandits allow greater flexibility in the reward functions
\cite{yang2002randomized,rigollet2010nonparametric}, but generally face
a more difficult learning problem and need not retain a simple interpretation of how covariates determine treatment decisions. Single-index bandits \cite{arya2026BIDS, kang2026single} provide a natural middle ground, retaining an interpretable low-dimensional index while leaving its mapping to the mean reward unspecified.

In this paper, we study a two-arm single-index contextual bandit in which
the arms have different index directions but share an unknown strictly
increasing link, $\E(Y_{t,a}\mid X_t)=g(\beta_a^\top X_t).$
The shared-link assumption has a natural interpretation from the
perspective of elicitation \cite{gneiting2011making}: the arm-specific
indices may differ, but they are translated to the mean-reward scale
through a common elicitation geometry. 
Motivated by the success of greedy policies in parametric contextual
bandits under suitable context distributions \cite{bastani2021mostly},
we ask: \emph{what does it mean to be greedy when the shared monotone
link is unknown?} Monotonicity provides a simple answer: $g(\beta_+^\top x)\geq g(\beta_-^\top x)$
if and only if
$(\beta_+-\beta_-)^\top x\geq0.$
Thus the optimal policy is determined by the decision boundary
$\{x:(v^\star)^\top x=0\},$
whose normal direction is $v^\star=(\beta_+-\beta_-)/\|\beta_+-\beta_-\|_2.$
Rather than estimating the two reward functions, a greedy learner can therefore target only the decision boundary, opening the possibility of faster regret rates under suitable structure.
 This motivates
\emph{Natural Boundary Learning} (NBL), which uses a sequential
\emph{Stein contrast} to update the boundary direction on the unit sphere. Related ideas arise in the supervised individualized treatment-learning
literature, where single-index structure has been imposed on treatment
contrasts to obtain interpretable treatment rules without fully modeling
outcome surfaces \cite{xiong2017treatment,liang2023relative}.

Learning the boundary, however, is only part of the problem. The regret incurred from a boundary error depends on how differences on the index scale
are translated to the mean-reward scale, bringing the unknown link back
into the analysis. We show that the local dynamics of NBL are governed
by a \emph{decision stability coefficient} balancing arm separation,
the sensitivity and curvature of the shared link, and the context
distribution. Through the elicitation representation
$g=(\varphi')^{-1}$, this stability is further connected to the geometry
of the underlying convex potential. Thus, although monotonicity makes
the optimal boundary independent of the unknown link, the geometry of
that link determines whether and how the boundary can be learned
greedily.

Our contributions are threefold. First, we formulate the shared-link
single-index bandit as a direct boundary-learning problem and develop
NBL based on a sequential Stein contrast. Second, we characterize its
local Riemannian dynamics and connect decision stability to the
elicitation geometry of the shared link. Third, under local decision
stability, we establish contraction to $v^\star$ and a logarithmic
expected regret. Numerical experiments illustrate the
stability regimes and compare NBL with a parametric greedy logistic
benchmark.

\paragraph{Notation.}
Let $(z)_+=\max\{z,0\}$ and let $\1\{\cdot\}$ denote the indicator.
Write $\mathbb S^{d-1}=\{v\in\R^d:\|v\|_2=1\}$ and
$P_v^\perp=I_d-vv^\top$ for the tangent-space projection at $v$.
Throughout, $\|\cdot\|_2$ denotes the Euclidean norm.
\section{Setup and Related Work}
\label{sec:setup}
We consider a stochastic contextual bandit problem with two actions,
$\mathcal A=\{-1,+1\}$. At each decision time $t=1,\ldots,n$, a
context vector $X_t\in\mathbb R^d$ is observed, an action
$A_t\in\mathcal A$ is selected, and a corresponding reward $Y_t$ is revealed. Let $Y_{t,a}$ denote the potential reward associated with action $a$, so that $Y_t=Y_{t,A_t}$. We let $\mathcal H_{t-1}=\sigma(X_s,A_s,Y_s:1\leq s\leq t-1)$ denote the history available prior to observing $X_t$. 
Let $\mathcal F$ denote a class of probability distributions on
$\mathbb R$, and let $T:\mathcal F\rightarrow\mathbb R$ denote a
statistical functional of interest. For each action $a$ and context
$x$, let $F_{a,x}\in\mathcal F$ denote the conditional distribution
of $Y_{t,a}$ given $X_t=x$. We focus on the conditional mean
functional
\[
T(F_{a,x})
=
\E[Y_{t,a}\mid X_t=x]
=
m_a(x).
\]
We assume that the conditional mean admits the shared-link
single-index representation
\begin{equation}\label{eq:model}
m_a(x)=g(\beta_a^\top x),
\qquad a\in\{-1,+1\}.
\end{equation}
where $g:\mathbb R\rightarrow\mathbb R$ is an unknown common link
function and $\beta_a\in\mathbb R^d$ is an unknown action-specific
index direction satisfying $\|\beta_a\|_2=1$. Thus, the nonlinear
relationship between the index and the conditional mean is shared
across actions, while the index directions are allowed to differ, allowing the effects of the covariates/contexts to vary
across actions.
Throughout, the common link $g$ is assumed to be strictly increasing.
Consequently, $m_+(x)\geq m_-(x)$ if and only if 
$(\beta_+-\beta_-)^\top x\geq0$, so that the optimal decision boundary is linear even though the
conditional mean reward functions need not be. Define the index separation and corresponding unit direction by
\begin{equation}\label{eq:separation_direction}
\Delta_\beta=\|\beta_+-\beta_-\|_2,
\qquad
v^\star=\frac{\beta_+-\beta_-}{\Delta_\beta},
\end{equation}
where we assume $\beta_+\neq\beta_-$.
The optimal arm at context $x$ is then given by, 
$
a^\star(x)
=
\operatorname{sgn}\{(v^\star)^\top x\}.
$
Thus, although the reward model involves an unknown nonparametric
link and two unknown index directions, the optimal action depends
only on the decision-relevant direction $v^\star$ and not on the
unknown shared link $g$.
We measure the performance of a policy through its cumulative regret
relative to the oracle policy,
\begin{align}\label{eq:regret}
R_n(\pi)
=
\sum_{t=1}^n
\left\{
m_{a^\star(X_t)}(X_t)
-
m_{A_t}(X_t)
\right\}.
\end{align}
Our goal is therefore to learn $v^\star$ directly from the
sequentially collected bandit observations, without first estimating
the unknown link $g$ or the individual index directions
$\beta_+$ and $\beta_-$.

\paragraph{Related work.}
Single-index contextual bandits lie between parametric linear and
generalized linear bandits \cite{yadkori2011_linear,Filippi2010_GLM_Bandits, zhang2026generalized}
and fully nonparametric approaches
\cite{rigollet2010nonparametric,perchet2013multi,krause2011contextual,wanigasekara2019nonparametric}.
Within single-index models, the link and index direction may each be
shared or action-specific, yielding
$
g(\beta^\top x),
g_a(\beta^\top x),
g(\beta_a^\top x),
g_a(\beta_a^\top x).
$
Existing work has studied the shared-index setting with
action-specific links \cite{arya2026BIDS}, the shared-link,
shared-index setting \cite{kang2026single}, and models with both
action-specific links and indices
\cite{ma2025nonparametric,arya2026KSIB}. These approaches estimate the
relevant index direction(s) and construct the policy from them. In
contrast, for the shared-link, action-specific-index model
$g(\beta_a^\top x)$, we directly learn the decision-boundary direction
$v^\star\propto\beta_+-\beta_-$ without separately estimating the
arm-specific indices or the common link.
Our approach is also related to greedy contextual bandit methods,
which can learn effectively without continued randomized exploration
under suitable conditions on the context distribution
\cite{bastani2021mostly}.
Finally, our interpretation of the shared link through elicitation
geometry builds on the theory of elicitable statistical functionals
and consistent scoring functions \cite{gneiting2011making,osband1985providing, savage1971elicitation}.

\section{Stein Identification of the Decision Direction}
\label{sec:stein}

The optimal policy depends on the reward model only through the  direction $v^\star$ in \eqref{eq:separation_direction}. 
We therefore seek to estimate $v^\star$ directly.
Our construction builds on score-function estimators for single-index
models based on Stein's identity \cite{yangBala2017}. In general,
for a density $p$ with score $S(x)=-\nabla\log p(x)$, the first-order
Stein identity takes the form
$\E\{f(X)S(X)\}=\E\{\nabla f(X)\}$ under suitable regularity
conditions. For the theoretical development in this paper, we specialize to
Gaussian contexts.

\begin{assumption}[Gaussian contexts]
\label{assum:gaussian_context}
The contexts satisfy $X_t\overset{\mathrm{iid}}{\sim}N(0,I_d)$ and
are independent of the past history $\mathcal H_{t-1}$.
\end{assumption}

\begin{assumption}[Regularity of the common link]
\label{assum:polygrowth_g}
The common link $g:\R\to\R$ is twice continuously differentiable and
there exist constants $C_g>0$ and $q\geq0$ such that
$|g^{(j)}(z)|\leq C_g(1+|z|^q),\,$
$z\in\R,\, j=0,1,2,$
where $g^{(0)}=g$.
\end{assumption}

For Gaussian contexts, $S(X)=X$, and
Assumption~\ref{assum:polygrowth_g} ensures the integrability and boundary conditions required by Stein's identity. We work with isotropic Gaussian contexts, since whitening need not preserve the shared-link unit-norm parameterization. We therefore obtain the following specialization.

\begin{proposition}[First-order Stein identity]
\label{prop:first-order-stein}
Under Assumptions~\ref{assum:gaussian_context} and
\ref{assum:polygrowth_g}, for any
$\beta\in\mathbb S^{d-1}$,
$
\E\{g(\beta^\top X)X\}
=
\beta\,\E\{g'(\beta^\top X)\}.
$
\end{proposition}
Since $\|\beta_a\|_2=1$, we have $\beta_a^\top X\sim N(0,1)$ for
both actions. Let $\mu_g=\E\{g'(Z)\}$, where $Z\sim N(0,1)$.
By Proposition~\ref{prop:first-order-stein},
$\E\{m_a(X)X\}=\mu_g\beta_a$ for $a\in\{-1,+1\}$.
Since $g$ is strictly increasing and continuously differentiable,
$\mu_g>0$. Consequently,
\begin{equation}\label{eq:Stein_for_vstar}
\E\{m_+(X)X-m_-(X)X\}
=
\mu_g(\beta_+-\beta_-)
=
\Delta_\beta\mu_g v^\star.
\end{equation}
Thus, Stein's identity identifies the decision-relevant direction
$v^\star$ up to a positive scalar, without requiring estimation of
the common link $g$ or the individual index directions. The remaining challenge is to exploit this identity sequentially,
when only the reward corresponding to the selected action is observed
and adaptive arm selection changes the distribution of the contexts
conditional on the selected action. This motivates learning the
decision boundary directly rather than estimating the two index
directions separately.

\section{Natural Boundary Learning (NBL)}
\label{sec:nbl}
Equation~\eqref{eq:Stein_for_vstar} suggests directly learning the boundary direction from a Stein-type contrast. In the bandit
setting, however, only the reward $Y_t=Y_{t,A_t}$ corresponding to the
selected action is observed. NBL therefore proceeds in two stages:
an initial randomized sampling period is used to estimate the
boundary direction $v^\star$, followed by greedy sampling that continually
updates the boundary direction from the observed bandit feedback.
We make the following assumptions on the reward process.
\begin{assumption}[Sequential reward model]
\label{assum:sequential_rewards}
For each $t\geq1$ and $a\in\{-1,+1\}$,
$
\E[Y_{t,a}\mid X_t,\mathcal H_{t-1}]=m_a(X_t).
$
\end{assumption}
\begin{assumption}[Bounded rewards]
\label{assum:bounded_rewards}
There exists $B_Y>0$ such that
$|Y_{t,a}|\leq B_Y$ almost surely for all
$t\geq1$ and $a\in\{-1,+1\}$.
\end{assumption}
Assumption~\ref{assum:sequential_rewards} allows sequential dependence
while fixing the conditional mean reward, and boundedness in
Assumption~\ref{assum:bounded_rewards} provides the concentration needed
in our analysis.
For the first $n_0$ rounds, let
$A_t\sim\operatorname{Unif}\{-1,+1\}$ independently of $X_t$ and
$\mathcal H_{t-1}$, and define the \emph{Stein contrast}
$W_t=A_tY_tX_t$. It is easy to check that under Assumptions~\ref{assum:gaussian_context},
\ref{assum:polygrowth_g}, and \ref{assum:sequential_rewards} (see \eqref{eq:Wt-conditional-mean-proof}),
\[
\E[W_t\mid\mathcal H_{t-1}]
=
\frac{\Delta_\beta\mu_g}{2}v^\star.
\]
We therefore initialize
$v_{n_0}=\overline W_{n_0}/\|\overline W_{n_0}\|_2$, where
$\overline W_{n_0}=n_0^{-1}\sum_{t=1}^{n_0}W_t$, so that the unknown
scale $\Delta_\beta\mu_g/2$ disappears upon normalization.
Proposition~\ref{prop:randomized_initialization} in the Appendix
further shows that, for any $\delta\in(0,1)$, with probability at
least $1-\delta$,
\[
\|v_{n_0}-v^\star\|_2
\lesssim
\frac{B_Y}{\Delta_\beta\mu_g}
\sqrt{\frac{d+\log(1/\delta)}{n_0}}.
\]

For $t>n_0$, NBL selects
$
A_t=\operatorname{sgn}(v_{t-1}^\top X_t).
$
Since only the tangent component of $W_t$ changes the direction of
$v_{t-1}$, let
\begin{align}\label{eq:tangent_update}
P_v^\perp=I_d-vv^\top,
\qquad
G_t=P_{v_{t-1}}^\perp W_t.
\end{align}
NBL updates
\begin{equation}\label{eq:nbl_update}
v_t
=
\frac{v_{t-1}+\eta_tG_t}
{\|v_{t-1}+\eta_tG_t\|_2},
\qquad t>n_0.
\end{equation}
Algorithm~\ref{alg:nbl} summarizes the resulting procedure.
Figure~\ref{fig:tangent-update} illustrates the tangent projection
underlying the update.
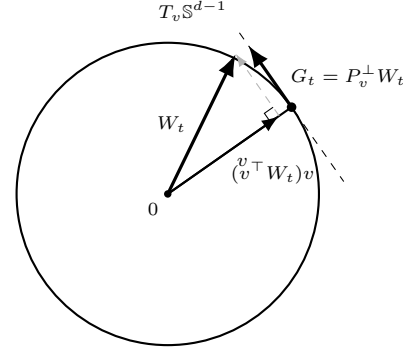
\begin{figure}[t]
\centering

\begin{tikzpicture}[
    scale=2.0,
    >=Latex,
    every node/.style={font=\scriptsize}
]

\draw[thick] (0,0) circle (1);

\coordinate (O) at (0,0);
\coordinate (v) at (0.819,0.574);

\fill (O) circle (0.7pt);
\node[below left] at (O) {$0$};

\fill (v) circle (0.9pt);
\draw[-,thick] (O) -- (v)
    node[midway,below right=-1pt] {$v$};

\coordinate (Tleft)  at (0.474,1.066);
\coordinate (Tright) at (1.164,0.082);

\draw[dashed] (Tleft) -- (Tright);
\node[above left=1pt] at (Tleft)
    {$T_v\mathbb S^{d-1}$};

\coordinate (R) at (0.737,0.516);
\coordinate (W) at (0.450,0.926);

\draw[->,very thick] (O) -- (W)
    node[midway,left=2pt] {$W_t$};

\draw[->,thick] (O) -- (R)
    node[midway,below right=-1pt]
    {$(v^\top W_t)v$};

\draw[->,thin,dashed,gray!60] (R) -- (W);

\pic[
    draw,
    angle radius=4pt
] {right angle = O--R--W};

\coordinate (Gend) at (0.532,0.984);

\draw[->,very thick] (v) -- (Gend)
    node[midway,right=4pt]
    {$G_t=P_v^\perp W_t$};

\end{tikzpicture}
\caption{
Tangent projection in NBL. The Stein signal decomposes as
$W_t=(v^\top W_t)v+P_v^\perp W_t$. The orthogonal component
$G_t=P_v^\perp W_t$ defines the tangent update direction at
$v\in\mathbb S^{d-1}$.
}
\label{fig:tangent-update}
\end{figure}
Having defined the NBL algorithm, we now study whether its greedy
updates continue to learn the true boundary direction after
randomized initialization. The key object is the population tangent
field associated with one greedy NBL update. We study two questions:
is the true boundary direction $v^\star$ an equilibrium of the
population dynamics, and do small perturbations away from $v^\star$
induce updates that move the boundary back toward it? These are
distinct properties. Equilibrium identifies $v^\star$ as a stationary point of the population dynamics, while local stability determines whether nearby population updates are restoring. We begin by identifying the population field generated by the greedy policy.
\paragraph{Population geometry.}
For $v\in\mathbb S^{d-1}$, define
\begin{align}\label{eq:Fdef}
F(v)
&=
\E[
m_+(X)X\1\{v^\top X\geq0\}\nonumber\\
&\qquad-
m_-(X)X\1\{v^\top X<0\}
].
\end{align}
Let $h(v)=P_v^\perp F(v)$. Under greedy allocation,
Lemma~\ref{lem:conditional_population_field} in the Appendix gives
$
\E[W_t\mid\mathcal H_{t-1}]=F(v_{t-1})
$
and
$
\E[G_t\mid\mathcal H_{t-1}]=h(v_{t-1}).
$
Thus, $h$ is the population tangent field underlying the NBL update.
Moreover, defining
$
\mathcal J(v)
=
\E[
\{m_+(X)+m_-(X)\}(v^\top X)_+
]
-
v^\top\E\{m_-(X)X\},
$
gives $\nabla\mathcal J(v)=F(v)$. Hence,
$
\operatorname{grad}_{\mathbb S^{d-1}}\mathcal J(v)=h(v),
$
so $h$ is the Riemannian gradient field of $\mathcal J$ on
$\mathbb S^{d-1}$.


\begin{proposition}[Population equilibrium]
\label{prop:population_equilibrium}
Under Assumptions~\ref{assum:gaussian_context} and
\ref{assum:polygrowth_g},
$
F(v^\star)=\kappa^\star v^\star
$
for some $\kappa^\star\in\R$. Consequently,
$
h(v^\star)=0.
$
\end{proposition}

Thus, $v^\star$ is a stationary point of the population Riemannian
gradient field. NBL can therefore be viewed as a stochastic
Riemannian gradient-ascent procedure on the sphere
\cite{absil2008optimization}. This reflects the natural-gradient principle of adapting an update to
the geometry of the parameter space \cite{amari2016information},
with the normalization in \eqref{eq:nbl_update} providing the standard
normalization retraction onto $\mathbb S^{d-1}$. This motivates the term
\emph{Natural Boundary Learning}.

\begin{algorithm}[t]
\caption{Natural Boundary Learning (NBL)}
\label{alg:nbl}
\begin{algorithmic}[1]
\Require $n_0$, step sizes $\{\eta_t\}_{t>n_0}$
\For{$t=1,\ldots,n_0$}
    \State Observe $X_t$; draw $A_t\sim\operatorname{Unif}\{-1,+1\}$
    \State Observe $Y_t$; set $W_t=A_tY_tX_t$
\EndFor
\State $\displaystyle
v_{n_0}\gets
\frac{\sum_{t=1}^{n_0}W_t}
{\|\sum_{t=1}^{n_0}W_t\|_2}$
\For{$t=n_0+1,n_0+2,\ldots, n$}
    \State Observe $X_t$; set
    $A_t\gets\operatorname{sgn}(v_{t-1}^\top X_t)$
    \State Observe $Y_t$; set $W_t\gets A_tY_tX_t$
    \State $G_t\gets(I_d-v_{t-1}v_{t-1}^\top)W_t$
    \State $\displaystyle
    v_t\gets
    \frac{v_{t-1}+\eta_tG_t}
    {\|v_{t-1}+\eta_tG_t\|_2}$
\EndFor
\end{algorithmic}
\end{algorithm}
The role of the randomized initialization is therefore to place the
iterate in a neighborhood of $v^\star$ where the population field is
contractive. We next characterize this local stability, which will
provide the contraction condition used in the stochastic analysis.
For the local analysis, write
$\beta_a=\theta+a(\Delta_\beta/2)v^\star$, $a\in\{-1,+1\}$,
where $\theta^\top v^\star=0$ and
$\|\theta\|_2^2=1-\Delta_\beta^2/4$. Let
$U=\theta^\top X$ and $Z=(v^\star)^\top X$. Under
Assumption~\ref{assum:gaussian_context},
$U\sim N(0,1-\Delta_\beta^2/4)$, $Z\sim N(0,1)$, and $U$ and $Z$
are independent. Define
\begin{align}\label{eq:mu_star_m2_def}
\mu^\star
=
\E\!\left[
g'\!\left(U+\frac{\Delta_\beta}{2}Z\right)\1\{Z>0\}
\right],\,
m_2=\E\{g''(U)\}.
\end{align}
For $u\in T_{v^\star}\mathbb S^{d-1}$, let
$Dh(v^\star)[u]$ denote the directional derivative of $h$ at
$v^\star$ along $u$.

Local stability simply means that when the
candidate direction moves slightly away from $v^\star$, the
population update moves it back toward $v^\star$. For geometric
intuition, we refer the reader to
Appendix~\ref{app:local_stability_in_2D}, where this behavior is
illustrated in two dimensions. We next characterize this
local behavior in arbitrary dimension.

\begin{proposition}[Local population dynamics]
\label{prop:local_jacobian}
Let $\phi_0$ denote the standard Gaussian density.
Under Assumptions~\ref{assum:gaussian_context} and
\ref{assum:polygrowth_g}, for every
$u\in T_{v^\star}\mathbb S^{d-1}$,
\begin{align}\label{eq:Jacobian}
Dh(v^\star)[u]
=
-\Delta_\beta\mu^\star u
+
2\phi_0(0)m_2\theta(\theta^\top u).
\end{align}
\end{proposition}

Since $(\theta^\top u)^2\leq
(1-\Delta_\beta^2/4)\|u\|_2^2$, the proposition implies
$\langle u,Dh(v^\star)[u]\rangle
\leq-\lambda_\star\|u\|_2^2$, where the
\emph{decision stability coefficient} is
\begin{equation}\label{eq:decision_stability}
\lambda_\star
=
\Delta_\beta\mu^\star
-
2\phi_0(0)(m_2)_+
\left(1-\frac{\Delta_\beta^2}{4}\right),
\end{equation}
and $(m_2)_+=\max\{m_2,0\}$, $\mu^\star$ and $m_2$ as defined in \eqref{eq:mu_star_m2_def}. Since positive $\lambda_\star$ ensures that the linearized population
field is strictly contractive in every tangent direction, we impose the following local stability condition.

\begin{assumption}[Positive decision stability]
\label{assum:positive-stability}
The decision stability coefficient satisfies
$
\lambda_\star>0.
$
\end{assumption}

\begin{corollary}[Local decision stability]
\label{cor:local_stability}
Under Assumption~\ref{assum:positive-stability}, $v^\star$ is a locally attracting
equilibrium of the population NBL dynamics. In particular, there
exist $r_0,\lambda_0>0$ such that
$\langle v-v^\star,h(v)\rangle
\leq-\lambda_0\|v-v^\star\|_2^2$
whenever $\|v-v^\star\|_2\leq r_0$.
\end{corollary}

Corollary~\ref{cor:local_stability} shows that positive
$\lambda_\star$ yields a locally contractive population field around
$v^\star$. This local contraction will be the key ingredient in the
stochastic analysis of Section~\ref{sec:regret}. To understand when it holds, we next examine the two competing terms in the local dynamics. The first term in \eqref{eq:Jacobian},
$-\Delta_\beta\mu^\star u$, always acts against a tangent
perturbation $u$, since $\mu^\star>0$. The second term,
$2\phi_0(0)m_2\theta(\theta^\top u)$, depends on the component of the perturbation along $\theta$ and can either reinforce or weaken this restoring force. In particular, when $u\perp\theta$, the second term vanishes and the population dynamics are locally restoring at rate $\Delta_\beta\mu^\star$. While strict monotonicity makes the optimal decision boundary
independent of the shape of $g$, the stability coefficient depends on the shared link through $\mu^\star$ and $m_2$. Thus, the link does not determine which boundary is optimal, but it does determine the local dynamics by which that boundary is learned. We next use elicitation geometry to make this dependence explicit.

\section{Elicitation Geometry and Decision Stability}
\label{sec:elicitation}
We first we give a decision-theoretic interpretation of the shared link through elicitation theory, and then use the resulting geometry to interpret the link-dependent quantities $\mu^\star$ and $m_2$ governing decision
stability.
\subsection{Elicitation geometry of the shared link}
\label{sec:elicitation_geometry}

 Recall from Section~\ref{sec:setup} that $T(F_{a,x})=m_a(x)$ is the conditional mean functional.
The mean is elicitable: there exist strictly consistent losses for which $T(F_{a,x})$ is the unique optimal point prediction under $F_{a,x}$. In particular, a classical class of consistent losses for the mean is given by Bregman losses
\cite{savage1971elicitation,gneiting2011making}. Let
$\varphi:\mathcal I\to\R$ be a twice differentiable, strictly convex
potential. The associated Bregman loss is
\begin{equation}\label{eq:Bregman_loss}
\ell_\varphi(r,y)
=
\varphi(y)-\varphi(r)-\varphi'(r)(y-r).
\end{equation}
Under standard regularity conditions,
$T(F_{a,x})= m_a(x) = \arg\min_r \E_{F_{a,x}}[\ell_\varphi(r,Y)]$. Different strictly convex potentials can elicit the same mean
functional while inducing different local geometries; see Appendix~\ref{app:elicitation} for examples. Moreover,
\[
\frac{\partial}{\partial r}\ell_\varphi(r,y)
=
\varphi''(r)(r-y),
\]
so the curvature $\varphi''(r)$ determines how prediction errors are locally
weighted at the prediction value $r$.
The same convex potential induces a natural dual coordinate
$\vartheta=\varphi'(r)$. Geometrically, $\vartheta$ is the slope of
the tangent to $\varphi$ at $r$. Since $\varphi$ is strictly convex,
$\varphi'$ is strictly increasing and hence invertible on its range.
Writing
$
g_\varphi=(\varphi')^{-1},
$
suppose that the action-specific index represents the conditional
mean on this dual scale, so that
$
\varphi'\{m_a(x)\}=\beta_a^\top x.
$
Mapping back to the primal mean scale gives
$
m_a(x)=g_\varphi(\beta_a^\top x).
$
Under this interpretation,
\[
\boxed{
\underbrace{m_a(x)}_{\text{primal mean}}
\quad
\xrightleftharpoons[
g_\varphi=(\varphi')^{-1}
]{\varphi'}
\quad
\underbrace{\beta_a^\top x}_{\substack{\text{dual coordinate}\\
\text{tangent slope}}}
}
\]
and the single-index assumption can be viewed as imposing linear
dimension reduction in the dual coordinate. The high-dimensional
context is first reduced to the scalar $\beta_a^\top x$, and the
inverse-gradient map $(\varphi')^{-1}$ then returns this scalar to the
conditional mean scale.

If the two actions share a common convex potential $\varphi$, then
they share the same response geometry $g_\varphi=(\varphi')^{-1}$,
while retaining action-specific linear representations through
$\beta_+$ and $\beta_-$. Thus, the shared-link model in \eqref{eq:model} separates a
common response geometry from action-specific dimension reduction.
Moreover, although this geometry determines how the dual coordinate
is mapped to the conditional mean, its monotonicity preserves the
ordering of the two actions. This preservation of ordering is exactly
what makes the optimal action depend only on  $v^\star$.


This interpretation does not require the unknown link in our model to
be specified through a known potential. Conversely, under suitable
regularity, any strictly increasing link $g$ admits such a
representation: defining $\varphi'=g^{-1}$ on the range of $g$
yields a strictly convex potential satisfying
$g=(\varphi')^{-1}$; see Appendix~\ref{app:link_to_potential} for the
formal construction. Thus, $\varphi$ provides an interpretive
description of the geometry induced by the unknown shared link,
rather than an object that must be known or estimated by NBL. This perspective clarifies how link geometry affects learning dynamics while leaving the optimal boundary unchanged.
\subsection{Geometry of decision stability}
\label{sec:decision_stability_geometry}

The geometry does, however, affect how changes in the dual coordinate
are transmitted to the conditional mean. Since
$\varphi'\{g_\varphi(z)\}=z$, differentiation gives
\[
g_\varphi'(z)
=
\frac{1}{\varphi''\{g_\varphi(z)\}},
\qquad
g_\varphi''(z)
=
-
\frac{\varphi'''\{g_\varphi(z)\}}
{\left[\varphi''\{g_\varphi(z)\}\right]^3}.
\]
Thus, $g_\varphi'(z)$ is the reciprocal curvature of the eliciting
potential and measures the local sensitivity of the primal mean to
changes in the dual coordinate, while $g_\varphi''(z)$ describes how
this sensitivity varies across the response surface.

The two terms in \eqref{eq:decision_stability} therefore have distinct
geometric roles. The first term measures the
strength of the restoring signal generated by the dual-to-primal
sensitivity: changes in the index must translate into sufficiently
large changes in the conditional mean for a perturbation of the
boundary to generate a meaningful restoring update. The second term
captures the effect of variation in this sensitivity across the
population. Because a curved link responds differently at different index values,
variation in the dual-to-primal sensitivity can create an asymmetry in
the population boundary update.
Importantly, what matters for stability is not curvature at a
particular index value, but its signed average under the population
distribution,
$
m_2
=
-\E[
\varphi'''\{g(U)\}/
\{\varphi''\{g(U)\}\}^3
].
$
Thus, local curvature may cancel across the population. When $m_2>0$,
the resulting average curvature effect can weaken the restoring
component of the population dynamics, whereas when $m_2\leq0$, it
cannot oppose the restoring movement. Appendix~\ref{app:elicitation_stability} expresses both $\mu^\star$ and $m_2$, and hence the full decision
stability coefficient $\lambda_\star$, directly in terms of the curvature and variation of the eliciting potential.

The preceding interpretation also connects decision stability to
self-concordance conditions used in the generalized linear bandit
literature \cite{russac2021self,zhang2026generalized}. In particular, consider
the generalized self-concordance condition
\begin{equation}
   |g''(z)|\leq Rg'(z),
 z\in\R, \label{eq:self-concordance-link}
\end{equation}
for some $R\geq0$.  Since $g'(z)>0$, this is
equivalent to $|(\log g')'(z)|\leq R$ and, in terms of the eliciting
potential, to
$
|\varphi'''(m)|\leq R\{\varphi''(m)\}^2.
$
Thus, self-concordance uniformly controls the variation of the
dual-to-primal sensitivity relative to its magnitude.

\begin{proposition}[Self-concordance and decision stability]
\label{prop:self-concordance-stability}
Suppose that the shared link satisfies
\eqref{eq:self-concordance-link}. Define
$
\kappa(a)=e^{a^2/2}\Phi(-a),
$
where $\Phi$ denotes the standard normal distribution function. Then the decision stability coefficient satisfies
\[
\lambda_\star
\geq
2\mu^\star
\left\{
\frac{\Delta_\beta}{2}
-
\phi_0(0)
\left(1-\frac{\Delta_\beta^2}{4}\right)
\frac{R}{\kappa(R\Delta_\beta/2)}
\right\}.
\]
Consequently, a sufficient condition for local decision stability is
$\Delta_\beta/2 >
\phi_0(0)(1-\Delta_\beta^2/4)
R/\kappa(R\Delta_\beta/2)$.
\end{proposition}

Proposition~\ref{prop:self-concordance-stability} (proof in Appendix ~\ref{app:self_concordance}) should be viewed as
a sufficient condition rather than as a replacement for the exact
stability criterion $\lambda_\star>0$. The latter depends on the
average slope and signed curvature of the shared link under the
population distribution, whereas self-concordance imposes uniform
control on curvature relative to slope. The self-concordance bound can
therefore be conservative, and decision stability may hold even when
the sufficient condition above is not satisfied.

This distinction is also useful in comparing NBL with existing
analyses of greedy generalized linear bandits. Self-concordance
controls the regularity of the response link, while conditions such
as covariate diversity \cite{bastani2021mostly} ensure that greedy
allocation continues to provide sufficient information for estimating
arm-specific parameters. Decision stability plays a different role: it determines whether the
population dynamics of a direct boundary estimator exhibit restoring
drift toward the optimal decision boundary, accounting jointly for
the geometry of the boundary and the magnitude of the reward gap
across it.

Taken together, these results reveal a distinction between the
geometry of the decision problem and the geometry of learning.
Monotonicity makes the optimal boundary independent of the shape of
the shared link, while the dual-to-primal sensitivity and its
variation determine the local dynamics by which that decision boundary is
learned. These dynamics depend jointly on the separation between the
arm-specific indices and on how the shared link translates this
separation into a reward gap. Thus, different elicitation geometries
can induce the same optimal decision boundary while producing
different reward gaps and different local learning behavior. We make
this distinction explicit in Section~\ref{sec:numerics}, where we
vary the geometry of the shared link while holding the arm-specific
indices, and hence the optimal decision boundary, fixed.

\section{Regret Guarantees}
\label{sec:regret}
The population analysis in Section~\ref{sec:nbl} shows that, when
$\lambda_\star>0$, the mean NBL update is locally contractive around
$v^\star$. We now show that this contraction persists for the
stochastic iterates following a sufficiently accurate randomized
initialization and yields logarithmic cumulative regret. Detailed
stochastic arguments are deferred to
Appendix~\ref{app:stochastic_analysis}.
\paragraph{Local stochastic convergence.}
Let
$
e_t=\|v_t-v^\star\|_2^2.
$
By the definition of the Stein contrast $G_t$ in \eqref{eq:tangent_update},
$G_t\in T_{v_{t-1}}\mathbb S^{d-1}$. Hence, normalization of the
tentative update cannot increase its distance from $v^\star$, giving
\begin{equation}
e_t
\leq
e_{t-1}
+
2\eta_t
\langle v_{t-1}-v^\star,G_t\rangle
+
\eta_t^2\|G_t\|_2^2.
\label{eq:pathwise-error-recursion}
\end{equation}

To separate the population drift from the stochastic fluctuation,
define
$
\xi_t=G_t-h(v_{t-1}).
$
By Lemma~\ref{lem:conditional_population_field},
$
\E[\xi_t\mid\mathcal H_{t-1}]=0.
$
Substituting
$
G_t=h(v_{t-1})+\xi_t
$
into \eqref{eq:pathwise-error-recursion} gives
\begin{align}
e_t
\leq{}&
e_{t-1}
+
2\eta_t
\langle v_{t-1}-v^\star,h(v_{t-1})\rangle
\nonumber\\
&+
2\eta_t
\langle v_{t-1}-v^\star,\xi_t\rangle
+
\eta_t^2\|G_t\|_2^2.
\label{eq:stochastic-error-decomposition}
\end{align}
Thus, the evolution of the directional error consists of a population
drift term, a martingale fluctuation, and a second-order stochastic
term. Under Assumption~\ref{assum:positive-stability}, the population
drift is locally contractive by
Corollary~\ref{cor:local_stability}. The remaining terms determine
whether the stochastic iterates remain within this stability region.
Combining the pathwise decomposition with the local contraction gives
the following one-step conditional error bound.

\begin{lemma}[Local one-step contraction]
\label{lem:local_one_step}
Suppose that Assumptions~\ref{assum:gaussian_context},
\ref{assum:sequential_rewards}, \ref{assum:bounded_rewards}, and
\ref{assum:positive-stability} hold. If
$
\|v_{t-1}-v^\star\|_2\leq r_0,
$
then
\begin{equation}
\E[e_t\mid\mathcal H_{t-1}]
\leq
(1-2\lambda_0\eta_t)e_{t-1}
+
B_Y^2d\,\eta_t^2.
\label{eq:local-conditional-error-recursion}
\end{equation}
\end{lemma}
Lemma~\ref{lem:local_one_step} shows that the directional error
contracts locally, up to a stochastic term of order $d\eta_t^2$.
Since this recursion is valid only while the iterates remain near
$v^\star$, we next control their exit probability.

\paragraph{Localization. } Let $\tau=\inf\{t\geq n_0:\|v_t-v^\star\|_2>r_0\}$ denote the first exit
time from the stability neighborhood.
To localize the stochastic recursion, we multiply the increments in
\eqref{eq:stochastic-error-decomposition} by
$\1\{t\leq\tau\}$. Since $\{t\leq\tau\}$ is $\mathcal H_{t-1}$-measurable, the
stopping indicator is predictable and preserves the
martingale-difference structure of the fluctuation term. Summing the
resulting stopped recursion reduces the localization argument to
controlling the accumulated martingale fluctuation and second-order
stochastic term. We use the harmonic step size
$
\eta_t=\gamma/(t+t_0),
$
which balances the contracting population drift with summable second-order stochastic fluctuations.
\begin{lemma}[Finite-horizon localization]
\label{lem:finite-horizon-localization}
Under the conditions as in Lemma~\ref{lem:local_one_step} above, fix $n\geq n_0$ and $\delta\in(0,1)$.
There exists a problem-dependent constant $C_{\rm loc}<\infty$,
independent of $d$, $n$, and $\delta$, such that, if
$
n_0\geq C_{\rm loc}\{d+\log(1/\delta)\},
$
then $\Prob(\tau\leq n)\leq\delta$.
\end{lemma}
The explicit form of $C_{\rm loc}$ is given in
Appendix~\ref{app:stochastic_analysis}. 
Lemma~\ref{lem:finite-horizon-localization} shows that an
initialization period of order $d+\log(1/\delta)$ keeps the NBL
iterates within the local stability region with probability at least
$1-\delta$. Within this region, the one-step contraction yields the
following estimation rate.
\begin{proposition}[Localized estimation rate]
\label{prop:localized-estimation}
Suppose that Assumptions~\ref{assum:gaussian_context},
 \ref{assum:sequential_rewards},
\ref{assum:bounded_rewards}, and
\ref{assum:positive-stability} hold. Let
$\eta_t=\gamma/(t+t_0)$, where $2\gamma\lambda_0>1$,
$n_0+1+t_0\geq2\gamma\lambda_0$, and $t_0\leq C_{\rm off}n_0$ for some constant 
$C_{\rm off}<\infty$. Then, for some $C_e<\infty$ independent of $d$ and $t$,
\[
\E\left[
\|v_t-v^\star\|_2^2\1\{\tau>t\}
\right]
\leq
C_e\frac{d}{t+t_0},
\qquad t\geq n_0.
\]
\end{proposition}
Thus, within the local stability region, the squared directional
error is of order $d/t$. We next show that this rate translates
directly into logarithmic cumulative regret through the geometry of
boundary disagreement.
\paragraph{From boundary error to regret.}
The greedy action differs from the optimal action only when
$v_{t-1}$ and $v^\star$ assign different signs to $X_t$. Thus, the
instantaneous regret can be written as
\[
r_t
=
|m_+(X_t)-m_-(X_t)|
\1\left\{
\operatorname{sgn}(v_{t-1}^\top X_t)
\neq
\operatorname{sgn}((v^\star)^\top X_t)
\right\}.
\]
Moreover, because the two conditional mean functions agree on the
optimal decision boundary, the reward gap vanishes as
$(v^\star)^\top X_t$ approaches zero. Under the polynomial-growth
condition on $g'$ and Gaussian contexts, the weighted Gaussian
disagreement bound in Lemma~\ref{lem:gaussian-polynomial-disagreement}
therefore gives
\begin{equation}
\E[r_t\mid\mathcal H_{t-1}]
\leq
C_{\rm gap}\frac{\Delta_\beta}{2}
\|v_{t-1}-v^\star\|_2^2,
\label{eq:error-to-regret}
\end{equation}
where $C_{\rm gap}<\infty$ depends only on the polynomial-growth
constants for $g'$ and on $q$ from Assumption~\ref{assum:polygrowth_g}, and is independent of
$\Delta_\beta$, $d$, and $t$.
\begin{theorem}[Regret of NBL]
\label{thm:regret-rate}
Suppose that Assumptions~\ref{assum:gaussian_context},
\ref{assum:polygrowth_g}, \ref{assum:sequential_rewards},
\ref{assum:bounded_rewards}, and
\ref{assum:positive-stability} hold. Let
$\eta_t=\gamma/(t+t_0)$ with $2\gamma\lambda_0>1$ and suppose that
$t_0\leq C_{\rm off}n_0$ for some fixed $C_{\rm off}<\infty$. There exists a
sufficiently large problem-dependent constant $C_0<\infty$,
independent of $d$ and $n$, such that, with
$n_0=\lceil C_0(d+\log n)\rceil$ and $n_0<n$, Natural Boundary
Learning satisfies
\[
\E[R_n]=O(d\log n),
\]
where the implicit constant is independent of $d$ and $n$.
\end{theorem}
Theorem~\ref{thm:regret-rate} shows that, after an initialization
period of order $d+\log n$, NBL achieves $O(d\log n)$ expected regret
while acting greedily throughout Stage II. The rate follows from the
direct connection between boundary estimation and regret. The
localized squared estimation error is of order $d/t$, while the
reward gap vanishes near the true decision boundary. Consequently,
\eqref{eq:error-to-regret} converts squared boundary error of order
$d/t$ into instantaneous regret of order $d/t$. Summing over time
then gives $\sum_{t=1}^n d/t=O(d\log n).$
\begin{figure*}[t]
\centering
\setlength{\tabcolsep}{1.5pt}

{\small\bfseries Decision stability and stochastic behavior}

\vspace{1mm}
\begin{tabular}{cccc}
{\scriptsize (a)} & {\scriptsize (b)} & {\scriptsize (c)} & {\scriptsize (d)} \\[-2mm]
\includegraphics[width=0.242\textwidth]{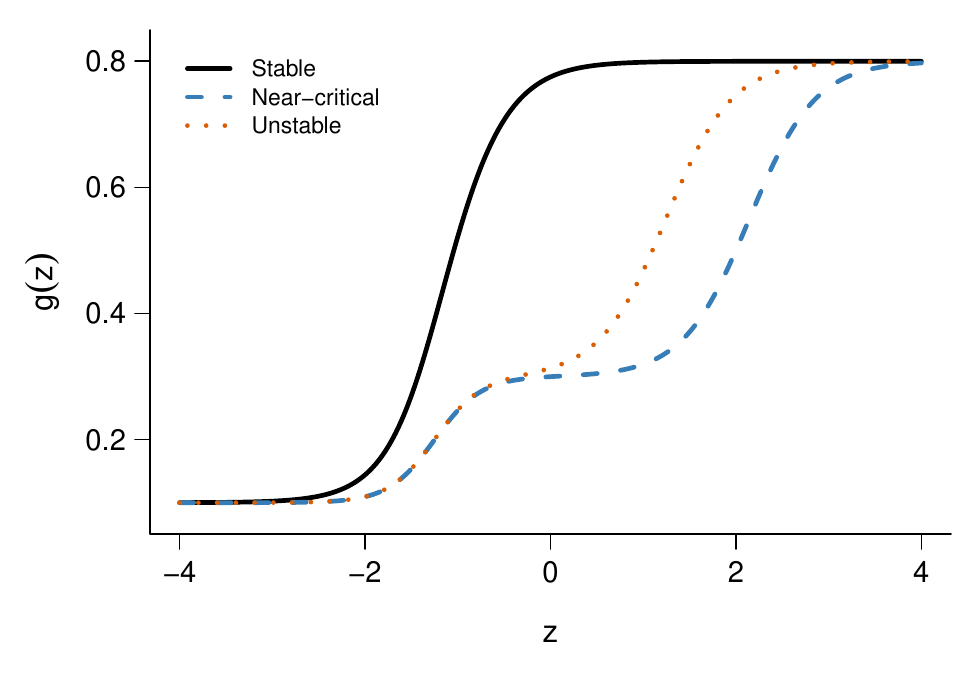}
&
\includegraphics[width=0.242\textwidth]{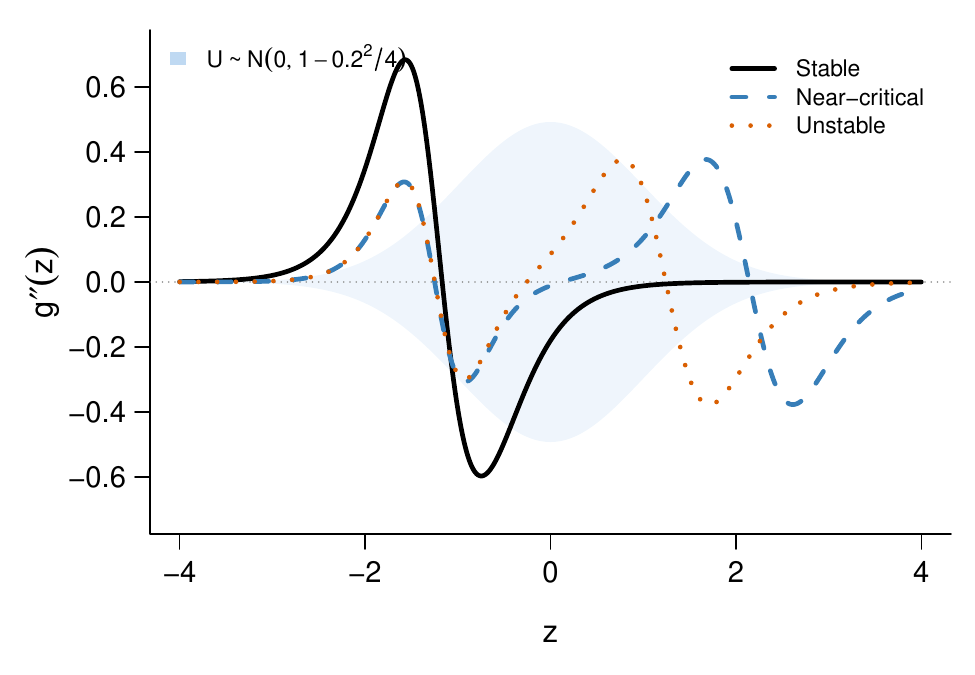}
&
\includegraphics[width=0.242\textwidth]{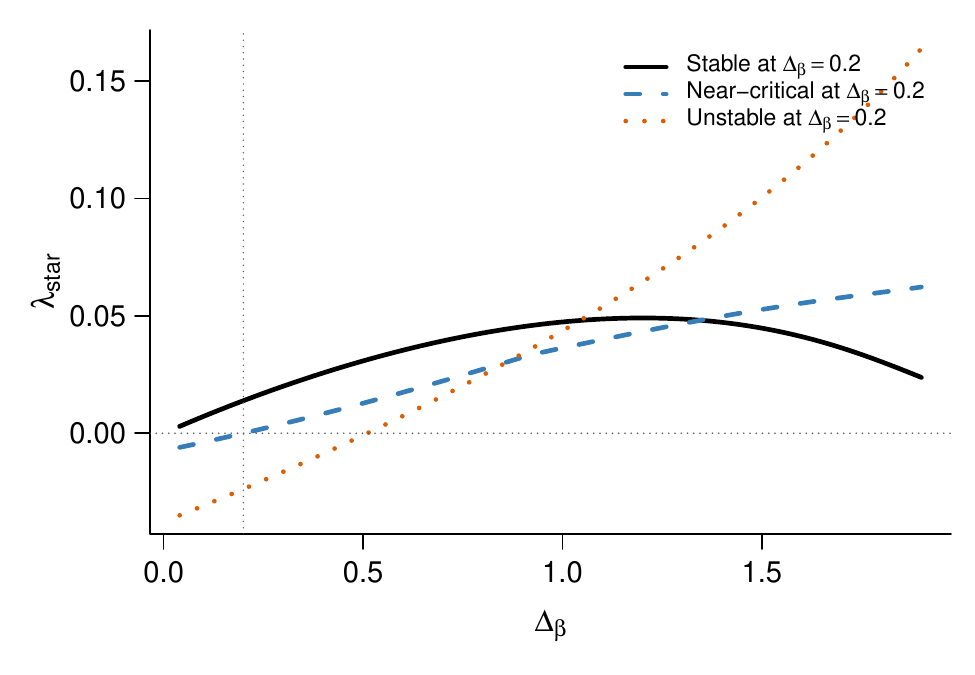}
&
\includegraphics[width=0.242\textwidth]{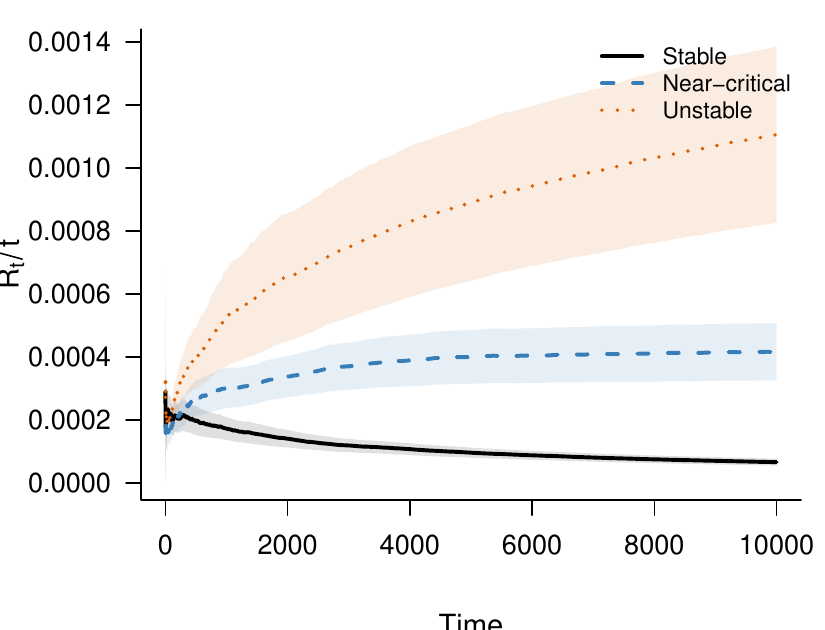}
\end{tabular}

\vspace{2mm}
{\small\bfseries Link misspecification and parametric comparison}

\vspace{1mm}
\begin{tabular}{cccc}
{\scriptsize (e)} & {\scriptsize (f)} & {\scriptsize (g)} & {\scriptsize (h)} \\[-2mm]
\includegraphics[width=0.242\textwidth]{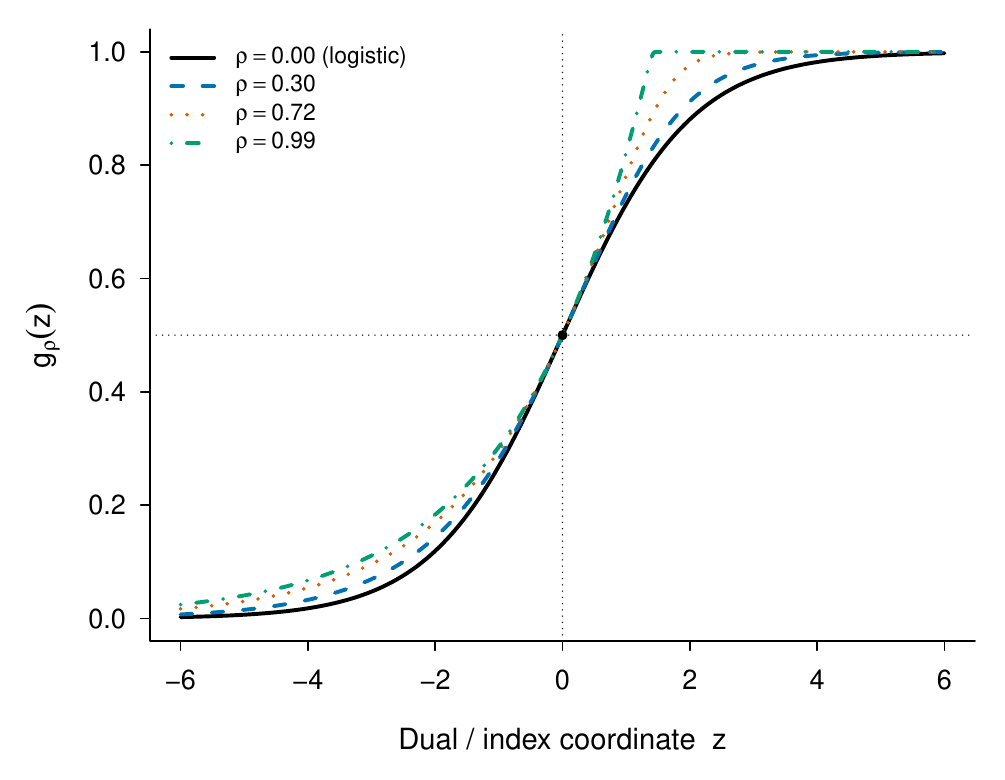}
&
\includegraphics[width=0.242\textwidth]{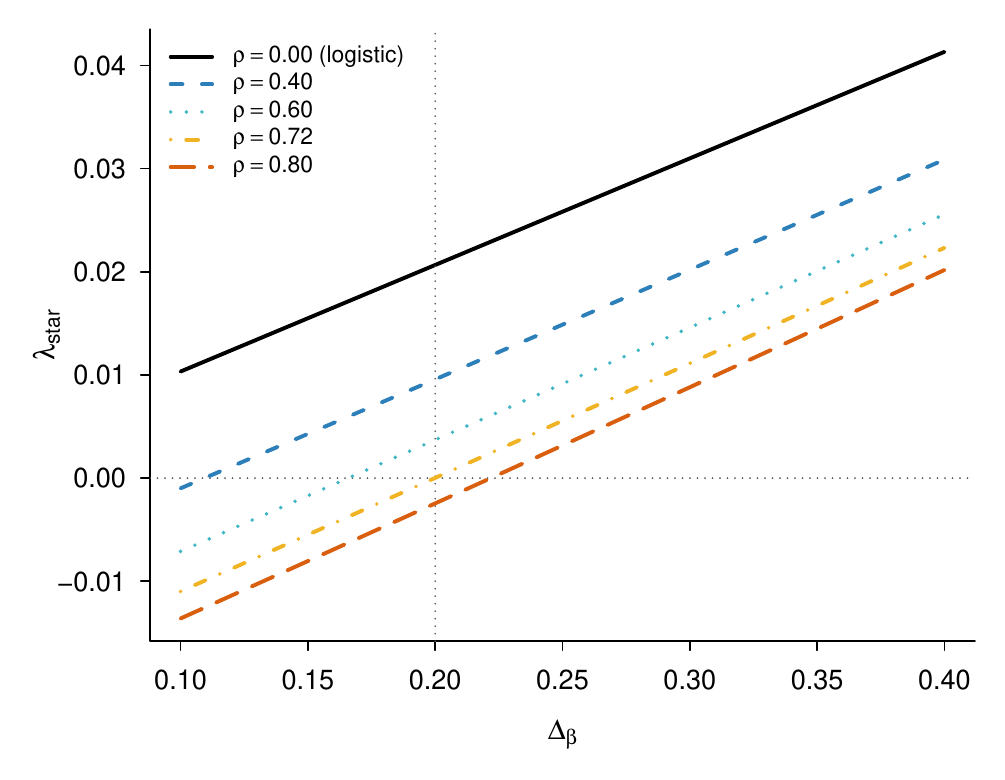}
&
\includegraphics[width=0.242\textwidth]{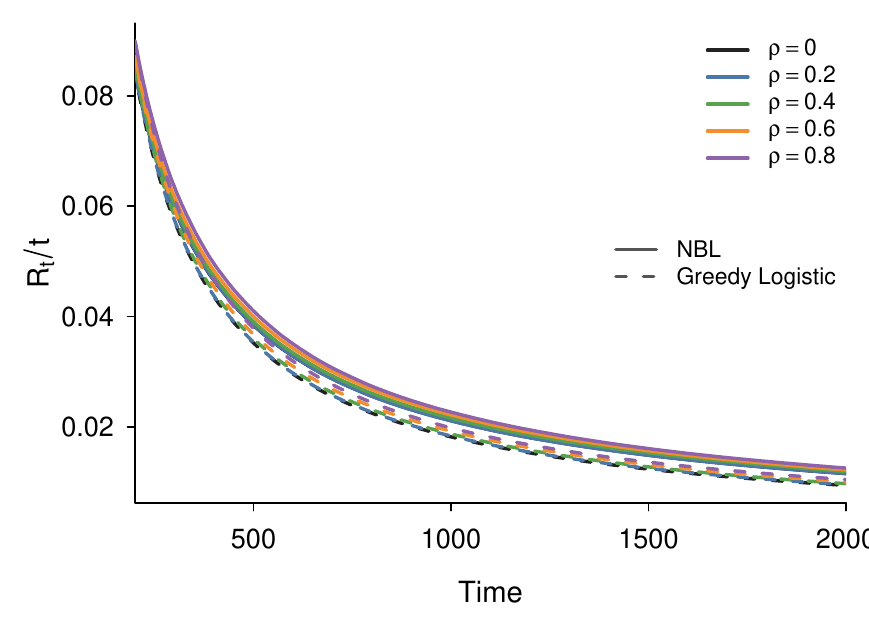}
&
\includegraphics[width=0.242\textwidth]{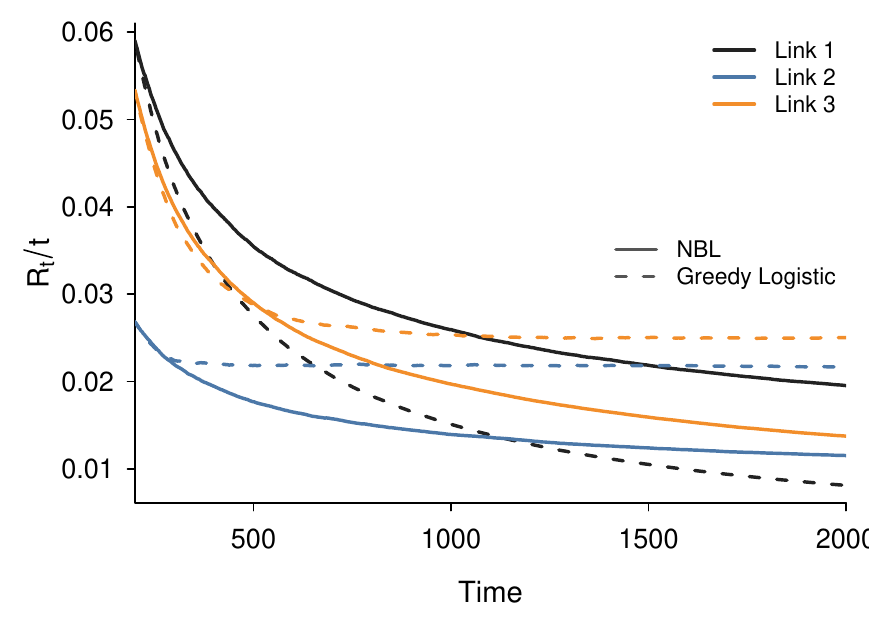}
\end{tabular}

\caption{Numerical illustration of decision stability and link
misspecification. Top: two-transition links yielding stable, near-critical,
and unstable dynamics, their curvature, the resulting decision stability
coefficient $\lambda_\star$, and stochastic NBL regret.
Bottom: the elicitation-induced link family, its decision stability under
weak separation, and comparisons of NBL with Greedy Logistic under link
misspecification.}
\label{fig:numerical-summary}
\end{figure*}
The shared-link model is related to generalized linear contextual
bandits, where a known link permits $\widetilde O(\sqrt{dn})$ regret
using UCB-style exploration \cite{li2017provably}. In contrast, our
link is unknown, but its shared monotonicity makes the optimal decision
boundary identifiable without estimating the link. Our result is also related to exploration-free linear contextual bandits: \cite{bastani2021mostly} obtain logarithmic regret under
covariate diversity, while \cite{kim2024local} establish
$O(\operatorname{poly}\log n)$ regret under local anti-concentration.
For the shared-link single-index model, NBL similarly requires
randomized exploration only initially and achieves
$O(d\log n)$ regret through local stability of the boundary dynamics.

\section{Numerical Illustrations}
\label{sec:numerics}
We use two sets of experiments to illustrate the role of decision stability
and the effect of link misspecification. Throughout, $d=2$; complete
simulation settings and additional experiments are given in
Appendix~\ref{app:numerics}.

\paragraph{Decision stability and stochastic behavior.}
We first consider the two-transition family
$g(z)=p_0+a_1\{1+\tanh[b_1(z-\mu_1)]\}
+a_2\{1+\tanh[b_2(z-\mu_2)]\}$ and select links yielding stable,
near-critical, and unstable dynamics at $\Delta_\beta=0.2$.
Figure~\ref{fig:numerical-summary}(a) shows that all three links are
strictly increasing, but their local geometry differs substantially.
In particular, panel (b) shows differences in curvature over regions
receiving appreciable Gaussian mass. Through the balance between the
sensitivity and curvature terms in the decision stability coefficient in \eqref{eq:decision_stability},
these differences lead to positive, nearly zero, and negative values of
$\lambda_\star$ at $\Delta_\beta=0.2$. Panel (c) shows, however, that instability is largely
confined to weakly separated arms, typically with $\Delta_\beta\lesssim0.2$
in these examples, and disappears as the arm separation increases. The stochastic experiment in
panel (d) reflects the same population geometry: average regret decreases
under stable dynamics, while near-critical and unstable dynamics exhibit
increasingly persistent regret. Additional directional-error diagnostics
and experiments varying $\Delta_\beta$ are in
Appendix~\ref{app:numerics}.

\paragraph{Link misspecification and parametric comparison.}
We next compare NBL with Greedy Logistic, a natural parametric benchmark
for greedy contextual-bandit learning \cite{bastani2021mostly}.
Figure~\ref{fig:numerical-summary}(e) considers the elicitation-induced
family $g_\rho=(\varphi_\rho')^{-1}$, where $\rho=0$ is logistic and
increasing $\rho$ gives controlled departures from logistic geometry.
Panel (f) shows that these departures can also alter decision stability
under weak separation. In panel (g), Greedy Logistic has lower regret when
the link is logistic or nearly logistic, while larger departures from the
logistic link lead to lower regret for NBL. Panel (h) returns to the
two-transition family in (a), now in the well-separated regime
$\Delta_\beta=1$, and similarly shows that NBL can outperform Greedy
Logistic under link misspecification. Together, these comparisons highlight
the parametric efficiency of Greedy Logistic under correct specification
and the advantage of directly learning the decision boundary when the link
is misspecified.
\vspace{-0.2cm}
\section{Conclusion}
\vspace{-0.2cm}
We introduced Natural Boundary Learning for two-arm contextual bandits
with a shared unknown monotone link. Rather than estimating the reward
functions, NBL uses a sequential Stein contrast to learn the optimal
decision boundary directly. Our analysis connects Stein identification,
Riemannian population dynamics, and elicitation geometry through the
decision stability coefficient $\lambda_\star$, and establishes
logarithmic expected regret under local stability. The numerical results
illustrate how link geometry and arm separation jointly determine the
resulting learning dynamics.

Several directions remain open. The present theory relies on isotropic
Gaussian contexts and two arms, and extending boundary learning to more
general context distributions and multiple interacting decision
boundaries is of particular interest. Our guarantees are also local and
require a sufficiently accurate initialization. Developing global or
adaptive initialization guarantees, as well as understanding robustness
to departures from the shared-link assumption, are natural directions
for future work.
\bibliographystyle{plain}
\bibliography{SIB_paper}

\clearpage
\appendix
\thispagestyle{empty}

\onecolumn
\aistatstitle{Appendix for ``Elicitation and Decision Geometry in Single-Index Bandits"}
\section{Proofs and supporting theory}
\subsection{Randomized initialization}

We first establish that the randomized initialization places NBL in
a prescribed neighborhood of $v^\star$ with sufficiently high
probability. Recall that
$\E[W_t\mid\mathcal H_{t-1}]
=(\Delta_\beta\mu_g/2)v^\star$ during the randomized stage.

\begin{proposition}[Randomized initialization]
\label{prop:randomized_initialization}
Suppose that Assumptions~\ref{assum:gaussian_context},
\ref{assum:polygrowth_g}, \ref{assum:sequential_rewards}, and
\ref{assum:bounded_rewards} hold. Then
\begin{equation}
\E\!\left[
\|v_{n_0}-v^\star\|_2^2
\right]
\leq
\frac{16B_Y^2d}
{\Delta_\beta^2\mu_g^2n_0}.
\label{eq:initialization-mse}
\end{equation}
Moreover, there exists a universal constant $C>0$ such that, for
any $\delta\in(0,1)$, with probability at least $1-\delta$,
\begin{equation}
\|v_{n_0}-v^\star\|_2
\leq
\frac{CB_Y}{\Delta_\beta\mu_g}
\sqrt{\frac{d+\log(2/\delta)}{n_0}}.
\label{eq:initialization-high-probability}
\end{equation}
Consequently, for any $r>0$ and $\delta_0\in(0,1)$, there exists a
universal constant $C_{\mathrm{init}}>0$ such that
\[
n_0
\geq
C_{\mathrm{init}}
\frac{B_Y^2\{d+\log(2/\delta_0)\}}
{\Delta_\beta^2\mu_g^2r^2}
\]
implies
$
\Prob(\|v_{n_0}-v^\star\|_2\leq r)\geq1-\delta_0.
$
\end{proposition}

\subsection{Gaussian integration-by-parts identities}
\label{app:gaussian_stein}

We record the Gaussian integration-by-parts identities used throughout
the proofs.

\begin{lemma}[Gaussian Stein identities]
\label{lem:gaussian_stein}
Let $X\sim N(0,I_d)$ and let $\beta\in\R^d$. Suppose that
$f:\R\to\R$ is twice continuously differentiable and satisfies the
integrability conditions required for the expectations below to be
finite. Then
\begin{align}
\E\left[Xf(\beta^\top X)\right]
&=
\E\left[f'(\beta^\top X)\right]\beta,
\label{eq:stein_first}
\\
\E\left[XX^\top f(\beta^\top X)\right]
&=
\E\left[f(\beta^\top X)\right]I_d
+
\E\left[f''(\beta^\top X)\right]\beta\beta^\top.
\label{eq:stein_second}
\end{align}
Equivalently,
\[
\E\left[
(XX^\top-I_d)f(\beta^\top X)
\right]
=
\E\left[f''(\beta^\top X)\right]\beta\beta^\top.
\]
\end{lemma}

\begin{proof}
For the first identity, the multivariate Gaussian integration-by-parts
formula gives
\[
\E[X_j h(X)]
=
\E[\partial_j h(X)].
\]
Taking $h(X)=f(\beta^\top X)$ yields
\[
\partial_j h(X)
=
\beta_j f'(\beta^\top X),
\]
and hence
\[
\E[X_j f(\beta^\top X)]
=
\beta_j\E[f'(\beta^\top X)].
\]
Collecting the coordinates gives \eqref{eq:stein_first}.

For the second identity, applying Gaussian integration by parts twice
gives
\[
\E[(X_jX_k-\delta_{jk})h(X)]
=
\E[\partial_{jk}h(X)].
\]
Since
\[
\partial_{jk}f(\beta^\top X)
=
\beta_j\beta_k f''(\beta^\top X),
\]
we obtain
\[
\E[(X_jX_k-\delta_{jk})f(\beta^\top X)]
=
\beta_j\beta_k\E[f''(\beta^\top X)].
\]
Collecting the entries gives \eqref{eq:stein_second}.
\end{proof}

\subsection{Proof of Proposition~\ref{prop:randomized_initialization}}
\label{app:initialization}

\begin{proof}
Note that $Y_{t,a}$ which we denote as $Y_{t,+}$ and $Y_{t,-}$ denotes the potential rewards for the two arms $\{+1,-1\}$ at time $t$, respectively (these are not to the observed rewards).  
Recall that during the randomized initialization stage,
\[
W_t
=
A_tY_tX_t,
\qquad
t=1,\ldots,n_0,
\]
where $A_t$ is chosen uniformly from $\{-1,+1\}$ using fresh
randomization independent of $X_t$ and $\mathcal H_{t-1}$.
We first identify the conditional mean of $W_t$ given the past.
Since $X_t$ is independent of $\mathcal H_{t-1}$ and the randomized
action is selected uniformly,
\begin{align*}
\E[W_t\mid\mathcal H_{t-1}]
&=
\frac12
\E[Y_{t,+}X_t\mid\mathcal H_{t-1}]
-
\frac12
\E[Y_{t,-}X_t\mid\mathcal H_{t-1}].
\end{align*}
For each $a\in\{-1,+1\}$, the tower property and the conditional-mean
model give
\begin{align*}
\E[Y_{t,a}X_t\mid\mathcal H_{t-1}]
&=
\E\left[
X_t
\E[Y_{t,a}\mid X_t,\mathcal H_{t-1}]
\,\middle|\,
\mathcal H_{t-1}
\right]
\\
&=
\E\left[
X_t g(\beta_a^\top X_t)
\,\middle|\,
\mathcal H_{t-1}
\right].
\end{align*}
Because $X_t$ is independent of $\mathcal H_{t-1}$ and
$X_t\sim N(0,I_d)$, the last expectation is equal to
$
\E[X_tg(\beta_a^\top X_t)]
=
\mu_g\beta_a,
$
where the final equality follows from the first-order Gaussian Stein
identity. Consequently,
\begin{align}
\E[W_t\mid\mathcal H_{t-1}]
&=
\frac{\mu_g}{2}(\beta_+-\beta_-)
=
\frac{\Delta_\beta\mu_g}{2}v^\star.
\label{eq:Wt-conditional-mean-proof}
\end{align}
Define the centered increment
$
D_t
=
W_t-\frac{\Delta_\beta\mu_g}{2}v^\star.
$
By \eqref{eq:Wt-conditional-mean-proof},
$
\E[D_t\mid\mathcal H_{t-1}]
=
0.
$
Thus, $\{D_t\}_{t=1}^{n_0}$ forms a martingale-difference sequence
with respect to the filtration generated by the sequential
observations together with the randomization used to select the
actions. We first record a mean-squared bound for the initialization error.
Since $\{D_t\}_{t=1}^{n_0}$ is a martingale-difference sequence,
the cross terms vanish, and hence
\begin{align*}
\E\left\|
\overline W_{n_0}-\frac{\Delta_\beta\mu_g}{2}v^\star
\right\|_2^2
=
\frac{1}{n_0^2}
\E\left\|
\sum_{t=1}^{n_0}D_t
\right\|_2^2
=
\frac{1}{n_0^2}
\sum_{t=1}^{n_0}
\E\|D_t\|_2^2.
\end{align*}
Moreover, since
$D_t=W_t-\E[W_t\mid\mathcal H_{t-1}]$,
$
\E\|D_t\|_2^2
\leq
\E\|W_t\|_2^2
\leq
B_Y^2\E\|X_t\|_2^2
=
B_Y^2d.
$
Therefore,
\[
\E\left\|
\overline W_{n_0}-\frac{\Delta_\beta\mu_g}{2}v^\star
\right\|_2^2
\leq
\frac{B_Y^2d}{n_0}.
\]
Using
\[
\left\|
\frac{x}{\|x\|_2}
-
\frac{y}{\|y\|_2}
\right\|_2
\leq
\frac{2\|x-y\|_2}{\|y\|_2},
\]
with $x=\overline W_{n_0}$ and
$y=(\Delta_\beta\mu_g/2)v^\star$, we obtain
\[
\E\left[
\|v_{n_0}-v^\star\|_2^2
\right]
\leq
\frac{16B_Y^2d}
{\Delta_\beta^2\mu_g^2n_0}.
\]
This proves \eqref{eq:initialization-mse}. For the high-probability bound note that,
\[
\overline W_{n_0}-\frac{\Delta_\beta\mu_g}{2}v^\star
=
\frac1{n_0}\sum_{t=1}^{n_0}D_t.
\]
We therefore control the empirical initialization moment by
concentrating the martingale sum.
Fix any $u\in\mathbb S^{d-1}$. Under
Assumption~\ref{assum:bounded_rewards},
\[
|u^\top W_t|
=
|Y_t|\,|u^\top X_t|
\leq
B_Y|u^\top X_t|.
\]
Conditional on $\mathcal H_{t-1}$, the random variable
$u^\top X_t$ remains standard Gaussian because $X_t$ is independent
of $\mathcal H_{t-1}$. Hence, for every $s>0$,
\begin{align*}
\Prob\left(
|u^\top W_t|>s
\,\middle|\,
\mathcal H_{t-1}
\right)
\leq
\Prob\left(
B_Y|u^\top X_t|>s
\,\middle|\,
\mathcal H_{t-1}
\right)
=
\Prob\left(
B_Y|Z|>s
\right)
\leq
2\exp\left(
-\frac{s^2}{2B_Y^2}
\right),
\end{align*}
where $Z\sim N(0,1)$. Thus, conditionally on $\mathcal H_{t-1}$, every projection
$u^\top W_t$ has a sub-Gaussian tail with scale of order $B_Y$,
uniformly over $u\in\mathbb S^{d-1}$ and $t\leq n_0$.
The preceding conditional tail bound shows that, conditional on
$\mathcal H_{t-1}$, the random variable $u^\top W_t$ is sub-Gaussian
with scale of order $B_Y$. Since centering a sub-Gaussian random
variable preserves sub-Gaussianity up to a universal constant, there
exists a universal constant $C_0>0$ such that
\[
\E\left[
\exp\left\{
\lambda
\left(
u^\top W_t
-
\E[u^\top W_t\mid\mathcal H_{t-1}]
\right)
\right\}
\,\middle|\,
\mathcal H_{t-1}
\right]
\leq
\exp\left(
C_0\lambda^2B_Y^2
\right),
\qquad
\lambda\in\R.
\]
Recalling that
$
D_t
=
W_t-\E[W_t\mid\mathcal H_{t-1}]
=
W_t-\frac{\Delta_\beta\mu_g}{2}v^\star,
$
we therefore have
\[
\E\left[
\exp\{\lambda u^\top D_t\}
\,\middle|\,
\mathcal H_{t-1}
\right]
\leq
\exp\left(
C_0\lambda^2B_Y^2
\right),
\qquad
\lambda\in\R.
\]
We can now apply the standard exponential martingale argument.
Iterating the preceding conditional moment-generating-function bound
gives
\[
\E\left[
\exp\left\{
\lambda\sum_{t=1}^{n_0}u^\top D_t
\right\}
\right]
\leq
\exp\left(
C_2n_0\lambda^2B_Y^2
\right).
\]
Indeed, conditioning successively on the past yields
\begin{align*}
\E\left[
\exp\left\{
\lambda\sum_{t=1}^{n_0}u^\top D_t
\right\}
\right]
&=
\E\left[
\exp\left\{
\lambda\sum_{t=1}^{n_0-1}u^\top D_t
\right\}
\E\left[
\exp\{\lambda u^\top D_{n_0}\}
\,\middle|\,
\mathcal H_{n_0-1}
\right]
\right]
\\
&\quad\leq
\exp(C_2\lambda^2B_Y^2)
\E\left[
\exp\left\{
\lambda\sum_{t=1}^{n_0-1}u^\top D_t
\right\}
\right],
\end{align*}
and repeating this argument gives the stated bound.
Applying Chernoff's inequality and optimizing over $\lambda$ therefore
gives a universal constant $c_0>0$ such that
\[
\Prob\left(
\left|
\sum_{t=1}^{n_0}u^\top D_t
\right|
>n_0s
\right)
\leq
2\exp\left(
-c_0\frac{n_0s^2}{B_Y^2}
\right).
\]
Equivalently,
\begin{align}
\Prob\left(
\left|
u^\top
\left\{
\overline W_{n_0}
-\frac{\Delta_\beta\mu_g}{2}v^\star
\right\}
\right|
>s
\right)
\leq
2\exp\left(
-c_0\frac{n_0s^2}{B_Y^2}
\right).
\label{eq:init-fixed-direction-proof}
\end{align}
It remains to make this concentration bound uniform over directions
$u\in\mathbb S^{d-1}$. Let $\mathcal N$ be a $1/2$-net of
$\mathbb S^{d-1}$. We may choose $\mathcal N$ so that
$
|\mathcal N|\leq 5^d.
$
Moreover, for every $z\in\R^d$,
$
\|z\|_2
\leq
2\max_{u\in\mathcal N}|u^\top z|.
$
Taking
$
z
=
\overline W_{n_0}-\frac{\Delta_\beta\mu_g}{2}v^\star
$
and applying a union bound together with
\eqref{eq:init-fixed-direction-proof}, we obtain
\begin{align*}
\Prob\left(
\left\|
\overline W_{n_0}
-\frac{\Delta_\beta\mu_g}{2}v^\star
\right\|_2
>2s
\right)
&\leq
\Prob\left(
\max_{u\in\mathcal N}
\left|
u^\top
\left\{
\overline W_{n_0}
-\frac{\Delta_\beta\mu_g}{2}v^\star
\right\}
\right|
>s
\right)
\\
&\qquad\leq
\sum_{u\in\mathcal N}
\Prob\left(
\left|
u^\top
\left\{
\overline W_{n_0}
-\frac{\Delta_\beta\mu_g}{2}v^\star
\right\}
\right|
>s
\right)
\\
&\qquad\leq
2|\mathcal N|
\exp\left(
-c_0\frac{n_0s^2}{B_Y^2}
\right)
\\
&\qquad\leq
2\exp\left(
d\log 5
-
c_0\frac{n_0s^2}{B_Y^2}
\right).
\end{align*}
Choosing
$
s
=
C_3B_Y
\sqrt{
\frac{d+\log(2/\delta)}{n_0}
}
$
for a sufficiently large universal constant $C_3>0$ gives, with
probability at least $1-\delta$,
\begin{align}
\left\|
\overline W_{n_0}
-\frac{\Delta_\beta\mu_g}{2}v^\star
\right\|_2
\leq
C_4B_Y
\sqrt{
\frac{d+\log(2/\delta)}{n_0}
}
\label{eq:init-vector-concentration-proof}
\end{align}
for another universal constant $C_4>0$. We next translate the concentration of the unnormalized moment
$\overline W_{n_0}$ into concentration of its direction. For any
nonzero vectors $x,y\in\R^d$,
\[
\left\|
\frac{x}{\|x\|_2}
-
\frac{y}{\|y\|_2}
\right\|_2
\leq
\frac{2\|x-y\|_2}{\|y\|_2}.
\]
Applying this inequality with
$
x=\overline W_{n_0},
y=(\Delta_\beta\mu_g/2)v^\star,
$
and using $\|v^\star\|_2=1$, we obtain
\[
\|v_{n_0}-v^\star\|_2
\leq
\frac{4}{\Delta_\beta\mu_g}
\left\|
\overline W_{n_0}
-\frac{\Delta_\beta\mu_g}{2}v^\star
\right\|_2.
\]
Combining this inequality with
\eqref{eq:init-vector-concentration-proof} and absorbing numerical
factors into a universal constant $C>0$ gives, with probability at
least $1-\delta$,
\[
\|v_{n_0}-v^\star\|_2
\leq
\frac{CB_Y}{\Delta_\beta\mu_g}
\sqrt{
\frac{d+\log(2/\delta)}{n_0}
}.
\]
Finally, let $r>0$ and $\delta_0\in(0,1)$. The preceding upper bound
is at most $r$ whenever
\[
n_0
\geq
C_{\mathrm{init}}
\frac{
B_Y^2\{d+\log(2/\delta_0)\}
}{
\Delta_\beta^2\mu_g^2r^2
}
\]
for a sufficiently large universal constant
$C_{\mathrm{init}}>0$. Therefore,
\[
\Prob\left(
\|v_{n_0}-v^\star\|_2\leq r
\right)
\geq
1-\delta_0.
\]
This completes the proof.
\end{proof}
\subsection{Population dynamics proofs}
\subsubsection{Conditional mean of the greedy NBL update}
\label{app:conditional_population_field}

\begin{lemma}[Conditional population field]
\label{lem:conditional_population_field}
During Stage II of NBL,
\begin{equation}
\E[W_t\mid\mathcal H_{t-1}]
=
F(v_{t-1}),
\label{eq:app-conditional-F}
\end{equation}
where $F$ is defined in \eqref{eq:Fdef}. Consequently,
\begin{equation}
\E[G_t\mid\mathcal H_{t-1}]
=
h(v_{t-1}),
\label{eq:app-conditional-h}
\end{equation}
where $h(v)=P_v^\perp F(v)$.
\end{lemma}

\begin{proof}
During Stage II, the greedy action rule is
\[
A_t
=
\begin{cases}
+1,&v_{t-1}^\top X_t\geq0,\\
-1,&v_{t-1}^\top X_t<0.
\end{cases}
\]
Since
$
W_t
=
A_tY_tX_t,
$
we may write
\[
W_t
=
Y_{t,+}X_t\1\{v_{t-1}^\top X_t\geq0\}
-
Y_{t,-}X_t\1\{v_{t-1}^\top X_t<0\}.
\]
Taking conditional expectation given $\mathcal H_{t-1}$ and applying
the tower property gives
\begin{align*}
\E[W_t\mid\mathcal H_{t-1}]
&=
\E\Big[
\E\big[
Y_{t,+}X_t\1\{v_{t-1}^\top X_t\geq0\}
-
Y_{t,-}X_t\1\{v_{t-1}^\top X_t<0\}
\,\big|\,
X_t,\mathcal H_{t-1}
\big]
\,\Big|\,
\mathcal H_{t-1}
\Big].
\end{align*}
Because $X_t$ and $v_{t-1}$ are measurable with respect to
$\sigma(X_t,\mathcal H_{t-1})$, while
\[
\E[Y_{t,a}\mid X_t,\mathcal H_{t-1}]
=
g(\beta_a^\top X_t),
\]
the inner conditional expectation equals
$
g(\beta_+^\top X_t)X_t
\1\{v_{t-1}^\top X_t\geq0\}
-
g(\beta_-^\top X_t)X_t
\1\{v_{t-1}^\top X_t<0\}.
$
Therefore,
\begin{align*}
\E[W_t\mid\mathcal H_{t-1}]
&=
\E\left[
g(\beta_+^\top X_t)X_t
\1\{v_{t-1}^\top X_t\geq0\}
-
g(\beta_-^\top X_t)X_t
\1\{v_{t-1}^\top X_t<0\}
\,\middle|\,
\mathcal H_{t-1}
\right].
\end{align*}

Now $v_{t-1}$ is $\mathcal H_{t-1}$-measurable, while $X_t$ is
independent of $\mathcal H_{t-1}$ and has the same distribution as
$X\sim N(0,I_d)$. Conditional on $\mathcal H_{t-1}$, we may therefore
treat $v_{t-1}$ as fixed and integrate only with respect to the
distribution of $X_t$. By the definition of $F$ in
\eqref{eq:Fdef},
$
\E[W_t\mid\mathcal H_{t-1}]
=
F(v_{t-1}),
$
which proves \eqref{eq:app-conditional-F}.

Finally,
$
G_t
=
P_{v_{t-1}}^\perp W_t.
$
Since $v_{t-1}$ is $\mathcal H_{t-1}$-measurable,
$P_{v_{t-1}}^\perp$ is also $\mathcal H_{t-1}$-measurable. Hence
\begin{align*}
\E[G_t\mid\mathcal H_{t-1}]
=
P_{v_{t-1}}^\perp
\E[W_t\mid\mathcal H_{t-1}]
=
P_{v_{t-1}}^\perp F(v_{t-1})
=
h(v_{t-1}),
\end{align*}
which proves \eqref{eq:app-conditional-h}.
\end{proof}

\begin{proof}[Proof of Proposition~\ref{prop:population_equilibrium}]
We show that the true boundary direction is an equilibrium of the
population tangent field:
$
h(v^\star)=0.
$
The key observation is that, at $v^\star$, the two greedy decision
regions are the two halfspaces separated by the hyperplane
orthogonal to $v^\star$. Under the isotropic Gaussian distribution,
these two halfspaces can be paired by reflection. Let
\begin{equation}
Z=(v^\star)^\top X.
\label{eq:Z-def}
\end{equation}
Using the orthogonal decomposition of $X$ along $v^\star$ and its
orthogonal complement,
\[
X
=
P_{v^\star}^\perp X+Zv^\star.
\]
Since $X\sim N(0,I_d)$, we have $Z\sim N(0,1)$. Moreover, $Z$ and
$P_{v^\star}^\perp X$ are jointly Gaussian and uncorrelated, since
\[
\operatorname{Cov}
\left(
P_{v^\star}^\perp X,Z
\right)
=
P_{v^\star}^\perp v^\star
=
0.
\]
Hence $Z$ is independent of $P_{v^\star}^\perp X$. At $v=v^\star$,
the definition of the population field in \eqref{eq:Fdef} gives
\[
F(v^\star)
=
\E\left[
g(\beta_+^\top X)X\1\{Z\geq0\}
\right]
-
\E\left[
g(\beta_-^\top X)X\1\{Z<0\}
\right].
\]
Since $\Prob(Z=0)=0$, we may freely interchange the events
$\{Z\geq0\}$ and $\{Z>0\}$ below. For a realization with $Z>0$,
define its reflection through the hyperplane orthogonal to
$v^\star$ by
$
\widetilde X
=
P_{v^\star}^\perp X-Zv^\star.
$
Since $Z\sim N(0,1)$ is independent of $P_{v^\star}^\perp X$ and
$Z\overset{d}=-Z$, we have
$
\widetilde X\overset{d}=X.
$
Moreover,
\[
(v^\star)^\top X=Z>0,
\qquad
(v^\star)^\top\widetilde X=-Z<0.
\]
Thus, the true greedy rule selects arm $+1$ at $X$ and arm $-1$ at
$\widetilde X$. Write
$
\beta_a=\theta+a(\Delta_\beta/2)v^\star
$
for $a\in\{-1,+1\}$, where
$
\theta^\top v^\star=0.
$
Using the orthogonal decomposition of $X$,
\begin{align*}
\beta_+^\top X
=
\left(\theta+\frac{\Delta_\beta}{2}v^\star\right)^\top
\left(
P_{v^\star}^\perp X+Zv^\star
\right)
=
\theta^\top P_{v^\star}^\perp X
+\frac{\Delta_\beta}{2}Z.
\end{align*}
At the reflected point,
\begin{align*}
\beta_-^\top\widetilde X
=
\left(\theta-\frac{\Delta_\beta}{2}v^\star\right)^\top
\left(
P_{v^\star}^\perp X-Zv^\star
\right)
=
\theta^\top P_{v^\star}^\perp X
+\frac{\Delta_\beta}{2}Z.
\end{align*}
Hence
$
\beta_+^\top X
=
\beta_-^\top\widetilde X,
$
and therefore, by the shared-link structure,
$
g(\beta_+^\top X)
=
g(\beta_-^\top\widetilde X).
$
Thus, each point in the positive halfspace and its reflection in the
negative halfspace have the same selected mean reward.
Because reflection preserves the Gaussian distribution, we may
rewrite the contribution from the negative halfspace over the
positive halfspace. In particular,
\begin{align*}
F(v^\star)
&=
\E\left[
\left\{
g(\beta_+^\top X)X
-
g(\beta_-^\top\widetilde X)\widetilde X
\right\}
\1\{Z>0\}
\right]\\
&=
\E\left[
g\left(
\theta^\top P_{v^\star}^\perp X
+\frac{\Delta_\beta}{2}Z
\right)
(X-\widetilde X)
\1\{Z>0\}
\right].
\end{align*}
By construction,
$
X-\widetilde X
=
2Zv^\star.
$
Therefore,
\begin{align*}
F(v^\star)
&=
2\E\left[
Z
g\left(
\theta^\top P_{v^\star}^\perp X
+\frac{\Delta_\beta}{2}Z
\right)
\1\{Z>0\}
\right]v^\star
=
\kappa^\star v^\star,
\end{align*}
where
\[
\kappa^\star
=
2\E\left[
Z
g\left(
\theta^\top P_{v^\star}^\perp X
+\frac{\Delta_\beta}{2}Z
\right)
\1\{Z>0\}
\right].
\]
Thus, the mean signed observation at the true boundary need not
vanish; rather, it is purely radial. Consequently, it has no
tangential component:
\[
h(v^\star)
=
P_{v^\star}^\perp F(v^\star)
=
\kappa^\star P_{v^\star}^\perp v^\star
=
0.
\]
Hence $v^\star$ is an equilibrium of the population dynamics.
\end{proof}
\subsubsection{Decision stability in two dimensions}
\label{app:local_stability_in_2D}

Before carrying out the general $d$-dimensional stability analysis, it
is useful to visualize the relevant geometry in two dimensions. Suppose
$d=2$ and, after rotating coordinates, let
$v^\star=(1,0)^\top$. A direction on the unit circle can be
parameterized by its signed angle from $v^\star$ as
$
v(\alpha)=(\cos\alpha,\sin\alpha)^\top.
$
The corresponding unit tangent vector is
\[
v_\perp(\alpha)
=
(-\sin\alpha,\cos\alpha)^\top
=
\frac{d}{d\alpha}v(\alpha).
\]
Thus, $v_\perp(\alpha)$ points in the direction of increasing
$\alpha$. Since $h(v(\alpha))$ lies in the one-dimensional tangent space at
$v(\alpha)$, there exists a scalar function $H$ such that
\begin{equation}
h(v(\alpha))
=
H(\alpha)v_\perp(\alpha).
\label{eq:H-alpha}
\end{equation}
Since $v_\perp(\alpha)$ is a unit tangent vector and
$h(v)=P_v^\perp F(v)$,
\[
H(\alpha)
=
v_\perp(\alpha)^\top h(v(\alpha))
=
v_\perp(\alpha)^\top F(v(\alpha)).
\]
Hence, the sign of $H(\alpha)$ determines the direction of the
population motion along the unit circle. Since $v^\star$ is a
population equilibrium, $H(0)=0$.
If $\alpha>0$, then $v(\alpha)$ lies counterclockwise from
$v^\star$, so restoring motion toward $\alpha=0$ requires movement in
the direction of decreasing $\alpha$, and hence $H(\alpha)<0$.
Conversely, if $\alpha<0$, restoring motion requires movement in the
direction of increasing $\alpha$, and hence $H(\alpha)>0$.
Combining the two cases, local attraction toward $v^\star$ is
characterized by
\begin{equation}
\alpha H(\alpha)<0
\qquad
\text{for all sufficiently small }\alpha\neq0.
\label{eq:2d-restoring-condition}
\end{equation}
Figure~\ref{fig:local-stability-2d} illustrates this restoring
geometry.

\begin{figure}[H]
\centering
\begin{tikzpicture}[scale=3.2,>=stealth]

\def\ap{30}
\def\am{-30}

\coordinate (O)     at (0,0);
\coordinate (Vstar) at (1,0);
\coordinate (Vp)    at ({cos(\ap)},{sin(\ap)});
\coordinate (Vm)    at ({cos(\am)},{sin(\am)});

\draw[thick]
({cos(55)},{sin(55)})
arc[start angle=55,end angle=-55,radius=1];

\draw[->,thick] (O) -- (Vstar)
    node[right=2pt] {$v^\star$};

\draw[->,thick] (O) -- (Vp)
    node[pos=0.72,above left=-1pt] {$v(\alpha_+)$};

\draw[->,thick] (O) -- (Vm)
    node[pos=0.72,below left=-1pt] {$v(\alpha_-)$};

\draw[->]
(0.35,0)
arc[start angle=0,end angle=\ap,radius=0.35];

\node at ({0.45*cos(15)},{0.45*sin(15)})
{$\alpha_+$};

\draw[->]
(0.42,0)
arc[start angle=0,end angle=\am,radius=0.42];

\node at ({0.53*cos(-15)},{0.53*sin(-15)})
{$\alpha_-$};

\draw[->,thin]
(Vp) -- ++({-0.32*sin(\ap)},{0.32*cos(\ap)})
node[above left=-1pt] {$v_\perp(\alpha_+)$};

\draw[->,very thick]
(Vp) -- ++({0.28*sin(\ap)},{-0.28*cos(\ap)})
node[below right=-1pt] {$h(v(\alpha_+))$};

\draw[<-,thin]
(Vm) -- ++({0.32*sin(\am)},{-0.32*cos(\am)})
node[above right=8pt] {$v_\perp(\alpha_-)$};

\coordinate (HmStart) at
($(Vm)+({0.04*sin(\am)},{-0.04*cos(\am)})$);

\draw[->,very thick]
(HmStart) -- ++({-0.28*sin(\am)},{0.28*cos(\am)})
node[below right=-0.5pt] {$h(v(\alpha_-))$};

\fill (O) circle (0.7pt);

\end{tikzpicture}
\caption{
Local stability geometry in two dimensions.
The tangent vector $v_\perp(\alpha)$ points in the direction of
increasing $\alpha$. For $\alpha_+>0$, a restoring population field
points opposite to $v_\perp(\alpha_+)$, toward decreasing $\alpha$.
For $\alpha_-<0$, it points in the same direction as
$v_\perp(\alpha_-)$, toward increasing $\alpha$. In both cases the
population field moves the candidate direction back toward
$v^\star$.
}
\label{fig:local-stability-2d}
\end{figure}
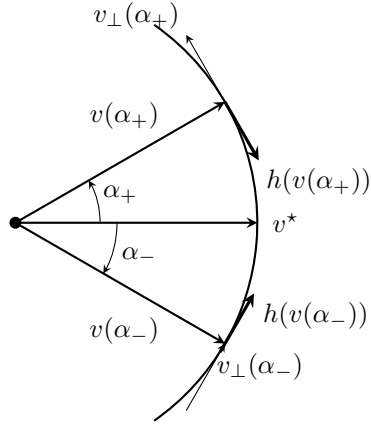

The two-dimensional picture suggests that the relevant object is the
first-order change in the population tangent field when $v$ is
perturbed away from $v^\star$. We next derive this local behavior in
arbitrary dimension.

\subsubsection{Proof of the local population Jacobian}
\label{app:local_jacobian}

\begin{proof}
The proof proceeds in three steps. First, we characterize admissible
first-order perturbations of $v^\star$ on the sphere and define the
directional derivative of the population tangent field. Second, we
decompose this derivative into the change in the tangent projection
and the change in the population vector induced by the moving greedy
decision boundary. Third, we simplify these two contributions using
Gaussian Stein identities and the equilibrium representation of
$F(v^\star)$, yielding the Jacobian in
Proposition~\ref{prop:local_jacobian}.

\paragraph{Step 1: First-order perturbations on the sphere.}

We have shown that $v^\star$ is an equilibrium of the population
tangent field,
$
h(v^\star)=0.
$
However, equilibrium alone does not imply stability. We therefore
study what happens to the population field when the current boundary
direction is moved a small amount away from $v^\star$ while remaining
on the sphere.

For $\epsilon>0$, consider a first-order perturbation of $v^\star$ of
the form
\[
v(\epsilon)
=
v^\star+\epsilon u+o(\epsilon),
\qquad
\epsilon\to0.
\]
Because the boundary direction is constrained to lie on
$\mathbb S^{d-1}$, we must have
$
\|v(\epsilon)\|_2^2=1.
$
Expanding the squared norm gives
\begin{align*}
\|v(\epsilon)\|_2^2
&=
\left\langle
v^\star+\epsilon u+o(\epsilon),
v^\star+\epsilon u+o(\epsilon)
\right\rangle\\
&=
\|v^\star\|_2^2
+
2\epsilon(v^\star)^\top u
+
o(\epsilon)\\
&=
1+2\epsilon(v^\star)^\top u+o(\epsilon).
\end{align*}
Therefore, for the perturbation to remain on the sphere to first
order, we must have
$
(v^\star)^\top u=0.
$
Thus, the admissible first-order perturbations are exactly the vectors
in the tangent space
\[
T_{v^\star}\mathbb S^{d-1}
=
\left\{
u\in\R^d:(v^\star)^\top u=0
\right\}.
\]

The question we want to answer is therefore the following: if we move
a small amount away from $v^\star$ in a tangent direction $u$, how
does the population tangent field $h(v)$ change? To make this precise,
for a fixed $u\in T_{v^\star}\mathbb S^{d-1}$, consider the normalized
path
\[
v_u(\epsilon)
=
\frac{v^\star+\epsilon u}
{\|v^\star+\epsilon u\|_2}.
\]
By construction, $v_u(\epsilon)\in\mathbb S^{d-1}$ and
$v_u(0)=v^\star$. Moreover, since $(v^\star)^\top u=0$,
\[
\left.
\frac{d}{d\epsilon}v_u(\epsilon)
\right|_{\epsilon=0}
=
u.
\]
Hence the path moves away from $v^\star$ initially in the direction
$u$. We now expand the population field along this path. A first-order
Taylor expansion gives
\begin{align*}
h(v_u(\epsilon))
&=
h(v_u(0))
+
\epsilon
\left.
\frac{d}{d\epsilon}
h(v_u(\epsilon))
\right|_{\epsilon=0}
+
o(\epsilon)\\
&=
h(v^\star)
+
\epsilon
\left.
\frac{d}{d\epsilon}
h(v_u(\epsilon))
\right|_{\epsilon=0}
+
o(\epsilon).
\end{align*}
This motivates defining the directional derivative of $h$ at
$v^\star$ in the tangent direction $u$ by
\[
Dh(v^\star)[u]
:=
\left.
\frac{d}{d\epsilon}
h(v_u(\epsilon))
\right|_{\epsilon=0}.
\]
Since $h(v^\star)=0$, the expansion reduces to
\begin{equation}
h(v_u(\epsilon))
=
\epsilon Dh(v^\star)[u]+o(\epsilon).
\label{eq:app-h-linearization}
\end{equation}
Equation~\eqref{eq:app-h-linearization} gives the local population
behavior around the equilibrium. For $\epsilon>0$, the perturbation
from $v^\star$ is, to first order, in the direction $u$. A population
field that moves the direction back toward $v^\star$ should therefore
have a component opposite to $u$. This leads us to study
$
\left\langle u,Dh(v^\star)[u]\right\rangle.
$
If this quantity is negative for every nonzero tangent direction
$u$, then the first-order population field points back toward
$v^\star$.

\paragraph{Step 2: Decomposition of the local derivative.}

Recall that
$
h(v)=P_v^\perp F(v),
P_v^\perp=I_d-vv^\top.
$
Therefore,
\[
Dh(v^\star)[u]
=
\left.
\frac{d}{d\epsilon}
\left\{
P_{v_u(\epsilon)}^\perp
F(v_u(\epsilon))
\right\}
\right|_{\epsilon=0}.
\]
Applying the product rule gives
\begin{equation}
Dh(v^\star)[u]
=
DP_{v^\star}^\perp[u]F(v^\star)
+
P_{v^\star}^\perp DF(v^\star)[u].
\label{eq:app-Dh-product}
\end{equation}
Here
\[
DP_{v^\star}^\perp[u]
=
\left.
\frac{d}{d\epsilon}
P_{v_u(\epsilon)}^\perp
\right|_{\epsilon=0}
\]
and
\[
DF(v^\star)[u]
=
\left.
\frac{d}{d\epsilon}
F(v_u(\epsilon))
\right|_{\epsilon=0}.
\]
The two terms in \eqref{eq:app-Dh-product} arise for different
reasons. The first describes how the tangent projection changes as
the candidate direction rotates while the population vector is held
fixed at $F(v^\star)$. The second describes how the population vector
$F(v)$ itself changes because rotating $v$ changes the two greedy
decision regions. We calculate these terms separately.

\paragraph{Change in the tangent projection.}
Since
$
P_v^\perp=I_d-vv^\top,
$
we have
\[
DP_{v^\star}^\perp[u]
=
-\left\{
u(v^\star)^\top+v^\star u^\top
\right\}.
\]
By Proposition~\ref{prop:population_equilibrium},
$
F(v^\star)=\kappa^\star v^\star.
$
Using $(v^\star)^\top v^\star=1$ and
$u^\top v^\star=0$, we therefore obtain
\begin{align*}
DP_{v^\star}^\perp[u]F(v^\star)
=
-\left\{
u(v^\star)^\top+v^\star u^\top
\right\}
\kappa^\star v^\star
=
-\kappa^\star u.
\end{align*}
Thus,
\begin{equation}
DP_{v^\star}^\perp[u]F(v^\star)
=
-\kappa^\star u.
\label{eq:app-projection-derivative}
\end{equation}

\paragraph{Change in the population vector.}

It remains to calculate $DF(v^\star)[u]$. Recall from
\eqref{eq:Fdef} that
\[
F(v)=F_+(v)-F_-(v),
\]
where
\[
F_+(v)
=
\E\left[
g(\beta_+^\top X)X
\1\{v^\top X\geq0\}
\right]
\]
and
\[
F_-(v)
=
\E\left[
g(\beta_-^\top X)X
\1\{v^\top X<0\}
\right].
\]
Thus, the dependence on $v$ enters only through the two halfspaces.
To determine how these halfspaces change when $v$ is perturbed away
from $v^\star$, fix a nonzero tangent direction
$u\in T_{v^\star}\mathbb S^{d-1}$ and consider the spherical path
\[
v_\epsilon
=
\cos(\epsilon\|u\|_2)v^\star
+
\sin(\epsilon\|u\|_2)
\frac{u}{\|u\|_2},
\qquad
\epsilon\geq0.
\]
Because $u\perp v^\star$, this path remains on the sphere:
$\|v_\epsilon\|_2=1$. Moreover,
\[
v_0=v^\star,
\qquad
\left.
\frac{d}{d\epsilon}v_\epsilon
\right|_{\epsilon=0}
=
u.
\]
Recall that
$
Z=(v^\star)^\top X.
$
Using the orthogonal decomposition of $X$ along $v^\star$,
\[
X
=
P_{v^\star}^\perp X
+
Zv^\star.
\]
Since $X\sim N(0,I_d)$, we have $Z\sim N(0,1)$, and $Z$ is
independent of $P_{v^\star}^\perp X$. Also, because
$u\perp v^\star$,
$
u^\top X
=
u^\top P_{v^\star}^\perp X.
$
We can therefore write
\begin{align*}
v_\epsilon^\top X
&=
\left\{
\cos(\epsilon\|u\|_2)v^\star
+
\sin(\epsilon\|u\|_2)
\frac{u}{\|u\|_2}
\right\}^\top X\\
&=
\cos(\epsilon\|u\|_2)Z
+
\frac{\sin(\epsilon\|u\|_2)}
{\|u\|_2}
u^\top P_{v^\star}^\perp X.
\end{align*}
For sufficiently small $\epsilon$,
$\cos(\epsilon\|u\|_2)>0$. Therefore,
\begin{equation}
v_\epsilon^\top X\geq0
\quad\Longleftrightarrow\quad
Z
\geq
-
\frac{
u^\top P_{v^\star}^\perp X
}{
\|u\|_2
}
\tan(\epsilon\|u\|_2).
\label{eq:app-moving-halfspace}
\end{equation}
We first differentiate $F_+(v_\epsilon)$. Conditional on
$P_{v^\star}^\perp X$, the quantity
$u^\top P_{v^\star}^\perp X$ is fixed, whereas $Z$ remains standard
normal. Hence
\begin{align*}
&\E\left[
g(\beta_+^\top X)X
\1\{v_\epsilon^\top X\geq0\}
\,\middle|\,
P_{v^\star}^\perp X
\right]\\
&\quad=
\int_{
-\frac{
u^\top P_{v^\star}^\perp X
}{
\|u\|_2
}
\tan(\epsilon\|u\|_2)
}^{\infty}
g\left\{
\beta_+^\top
\left(
P_{v^\star}^\perp X+zv^\star
\right)
\right\}
\left(
P_{v^\star}^\perp X+zv^\star
\right)
\phi_0(z)\,dz,
\end{align*}
where $\phi_0$ denotes the standard normal density.
Only the lower limit of this integral depends on $\epsilon$. Its
derivative at $\epsilon=0$ is
\begin{align*}
\left.
\frac{d}{d\epsilon}
\left\{
-
\frac{
u^\top P_{v^\star}^\perp X
}{
\|u\|_2
}
\tan(\epsilon\|u\|_2)
\right\}
\right|_{\epsilon=0}
=
-u^\top P_{v^\star}^\perp X.
\end{align*}
Therefore, by the fundamental theorem of calculus and the chain rule,
\begin{align*}
\left.
\frac{d}{d\epsilon}
\E\left[
g(\beta_+^\top X)X
\1\{v_\epsilon^\top X\geq0\}
\,\middle|\,
P_{v^\star}^\perp X
\right]
\right|_{\epsilon=0}
=
\phi_0(0)
\left(
u^\top P_{v^\star}^\perp X
\right)
g\left(
\beta_+^\top P_{v^\star}^\perp X
\right)
P_{v^\star}^\perp X.
\end{align*}
To pass from this conditional derivative to the derivative of
$F_+(v_\epsilon)$, we interchange differentiation and expectation.
For $\epsilon$ in a sufficiently small neighborhood of zero, the
derivative of the conditional integral is obtained by evaluating the
integrand at the moving boundary
\[
z=
-
\frac{
u^\top P_{v^\star}^\perp X
}{
\|u\|_2
}
\tan(\epsilon\|u\|_2)
\]
and multiplying by the derivative of this boundary. Since
$\tan(\epsilon\|u\|_2)$ and
$\sec^2(\epsilon\|u\|_2)$ are uniformly bounded for sufficiently
small $\epsilon$, Assumption~\ref{assum:polygrowth_g} implies that
the norm of this derivative is bounded by a polynomial in
$\|X\|_2$, uniformly over such $\epsilon$. Because $X$ is Gaussian,
all polynomial moments of $\|X\|_2$ are finite. Dominated convergence
therefore permits differentiation under the expectation.

Recall that
$
\beta_+=\theta+(\Delta_\beta/2)v^\star.
$
Since $P_{v^\star}^\perp X$ is orthogonal to $v^\star$,
\[
\beta_+^\top P_{v^\star}^\perp X
=
\theta^\top P_{v^\star}^\perp X.
\]
Taking expectation therefore gives
\begin{align}
DF_+(v^\star)[u]
&=
\phi_0(0)
\E\left[
g\left(
\theta^\top P_{v^\star}^\perp X
\right)
\left(
u^\top P_{v^\star}^\perp X
\right)
P_{v^\star}^\perp X
\right]
\nonumber\\
&=
\phi_0(0)
\E\left[
g\left(
\theta^\top P_{v^\star}^\perp X
\right)
(P_{v^\star}^\perp X)
(P_{v^\star}^\perp X)^\top
\right]u.
\label{eq:app-boundary-Fplus}
\end{align}

We next differentiate $F_-(v_\epsilon)$. From
\eqref{eq:app-moving-halfspace},
\[
v_\epsilon^\top X<0
\quad\Longleftrightarrow\quad
Z
<
-
\frac{
u^\top P_{v^\star}^\perp X
}{
\|u\|_2
}
\tan(\epsilon\|u\|_2).
\]
Conditional on $P_{v^\star}^\perp X$, we therefore have
\begin{align*}
&\E\left[
g(\beta_-^\top X)X
\1\{v_\epsilon^\top X<0\}
\,\middle|\,
P_{v^\star}^\perp X
\right]\\
&\quad=
\int_{-\infty}^{
-\frac{
u^\top P_{v^\star}^\perp X
}{
\|u\|_2
}
\tan(\epsilon\|u\|_2)
}
g\left\{
\beta_-^\top
\left(
P_{v^\star}^\perp X+zv^\star
\right)
\right\}
\left(
P_{v^\star}^\perp X+zv^\star
\right)
\phi_0(z)\,dz.
\end{align*}
Here the moving boundary is the upper limit of integration.
Consequently, the fundamental theorem of calculus gives the opposite
sign. Since
$
\beta_-=\theta-(\Delta_\beta/2)v^\star
$
and hence
$
\beta_-^\top P_{v^\star}^\perp X
=
\theta^\top P_{v^\star}^\perp X,
$
we obtain
\begin{align}
DF_-(v^\star)[u]
&=
-\phi_0(0)
\E\left[
g\left(
\theta^\top P_{v^\star}^\perp X
\right)
\left(
u^\top P_{v^\star}^\perp X
\right)
P_{v^\star}^\perp X
\right]
\nonumber\\
&=
-\phi_0(0)
\E\left[
g\left(
\theta^\top P_{v^\star}^\perp X
\right)
(P_{v^\star}^\perp X)
(P_{v^\star}^\perp X)^\top
\right]u.
\label{eq:app-boundary-Fminus}
\end{align}
Finally, because $F(v)=F_+(v)-F_-(v)$, the two boundary
contributions in \eqref{eq:app-boundary-Fplus} and
\eqref{eq:app-boundary-Fminus} add. Hence
\begin{equation}
DF(v^\star)[u]
=
2\phi_0(0)
\E\left[
g\left(
\theta^\top P_{v^\star}^\perp X
\right)
(P_{v^\star}^\perp X)
(P_{v^\star}^\perp X)^\top
\right]u.
\label{eq:app-DF}
\end{equation}

\paragraph{Step 3: Simplifying the local derivative.}

We next simplify the matrix expectation in
\eqref{eq:app-DF}. Since $\theta\perp v^\star$,
$
P_{v^\star}^\perp\theta=\theta
$
and therefore
$
\theta^\top P_{v^\star}^\perp X
=
\theta^\top X.
$
Moreover,
\begin{align*}
\E\left[
g\left(
\theta^\top P_{v^\star}^\perp X
\right)
(P_{v^\star}^\perp X)
(P_{v^\star}^\perp X)^\top
\right]
=
P_{v^\star}^\perp
\E\left[
g(\theta^\top X)XX^\top
\right]
P_{v^\star}^\perp.
\end{align*}
Applying the second-order Gaussian Stein identity in
Lemma~\ref{lem:gaussian_stein} with
$f(z)=g(z)$ and index direction $\theta$ gives
\[
\E\left[
g(\theta^\top X)XX^\top
\right]
=
\E[g(\theta^\top X)]I_d
+
\E[g''(\theta^\top X)]
\theta\theta^\top.
\]
Define
\begin{equation}
m_0
=
\E[g(\theta^\top X)].
\label{eq:app-m0}
\end{equation}
Recall that $m_2=\E[g''(\theta^\top X)]$. It follows that
\begin{align*}
\E\left[
g\left(
\theta^\top P_{v^\star}^\perp X
\right)
(P_{v^\star}^\perp X)
(P_{v^\star}^\perp X)^\top
\right]
=
m_0P_{v^\star}^\perp
+
m_2\theta\theta^\top.
\end{align*}
Substituting this identity into \eqref{eq:app-DF} gives
\[
DF(v^\star)[u]
=
2\phi_0(0)
\left\{
m_0P_{v^\star}^\perp
+
m_2\theta\theta^\top
\right\}u.
\]
Because $u$ is tangent at $v^\star$,
$P_{v^\star}^\perp u=u$. Hence
\[
P_{v^\star}^\perp DF(v^\star)[u]
=
2\phi_0(0)
\left\{
m_0u
+
m_2\theta(\theta^\top u)
\right\}.
\]
Combining this with
\eqref{eq:app-Dh-product} and
\eqref{eq:app-projection-derivative} gives
\begin{equation}
Dh(v^\star)[u]
=
\left\{
-\kappa^\star+2\phi_0(0)m_0
\right\}u
+
2\phi_0(0)m_2\theta(\theta^\top u).
\label{eq:app-Dh-pre-kappa}
\end{equation}

It remains to simplify the coefficient involving $\kappa^\star$.
From the proof of Proposition~\ref{prop:population_equilibrium},
\[
\kappa^\star
=
2\E\left[
Z
g\left(
\theta^\top P_{v^\star}^\perp X
+\frac{\Delta_\beta}{2}Z
\right)
\1\{Z>0\}
\right].
\]
Using the law of iterated expectations,
\begin{align*}
\kappa^\star
&=
2\E\Bigg[
\E\left[
Z
g\left(
\theta^\top P_{v^\star}^\perp X
+\frac{\Delta_\beta}{2}Z
\right)
\1\{Z>0\}
\,\middle|\,
P_{v^\star}^\perp X
\right]
\Bigg].
\end{align*}
Since $Z\sim N(0,1)$ is independent of
$P_{v^\star}^\perp X$, conditional on
$P_{v^\star}^\perp X$ the scalar
$\theta^\top P_{v^\star}^\perp X$ is fixed and $Z$ retains its
standard normal distribution. Therefore,
\begin{align*}
&\E\left[
Z
g\left(
\theta^\top P_{v^\star}^\perp X
+\frac{\Delta_\beta}{2}Z
\right)
\1\{Z>0\}
\,\middle|\,
P_{v^\star}^\perp X
\right]=
\int_0^\infty
z
g\left(
\theta^\top P_{v^\star}^\perp X
+\frac{\Delta_\beta}{2}z
\right)
\phi_0(z)\,dz.
\end{align*}
Since $\phi_0'(z)=-z\phi_0(z)$,
\begin{align*}
\int_0^\infty
z
g\left(
\theta^\top P_{v^\star}^\perp X
+\frac{\Delta_\beta}{2}z
\right)
\phi_0(z)\,dz
=
-\int_0^\infty
g\left(
\theta^\top P_{v^\star}^\perp X
+\frac{\Delta_\beta}{2}z
\right)
\phi_0'(z)\,dz.
\end{align*}
Integration by parts gives
\begin{align*}
-\int_0^\infty
g\left(
\theta^\top P_{v^\star}^\perp X
+\frac{\Delta_\beta}{2}z
\right)
\phi_0'(z)\,dz
&=
-\left[
g\left(
\theta^\top P_{v^\star}^\perp X
+\frac{\Delta_\beta}{2}z
\right)
\phi_0(z)
\right]_{z=0}^{z=\infty}\\
&\qquad\quad+
\frac{\Delta_\beta}{2}
\int_0^\infty
g'\left(
\theta^\top P_{v^\star}^\perp X
+\frac{\Delta_\beta}{2}z
\right)
\phi_0(z)\,dz.
\end{align*}
Under Assumption~\ref{assum:polygrowth_g}, $g$ grows at most
polynomially, whereas the Gaussian density decays faster than any
polynomial. Therefore,
\[
\lim_{z\to\infty}
g\left(
\theta^\top P_{v^\star}^\perp X
+\frac{\Delta_\beta}{2}z
\right)\phi_0(z)
=
0.
\]
At the lower endpoint,
$
g\left(
\theta^\top P_{v^\star}^\perp X
\right)\phi_0(0)
$
remains. Hence
\begin{align*}
&\int_0^\infty
z
g\left(
\theta^\top P_{v^\star}^\perp X
+\frac{\Delta_\beta}{2}z
\right)
\phi_0(z)\,dz\\
&\qquad=
\phi_0(0)
g\left(
\theta^\top P_{v^\star}^\perp X
\right)
+
\frac{\Delta_\beta}{2}
\int_0^\infty
g'\left(
\theta^\top P_{v^\star}^\perp X
+\frac{\Delta_\beta}{2}z
\right)
\phi_0(z)\,dz.
\end{align*}
Taking expectation over $P_{v^\star}^\perp X$ gives
\[
\kappa^\star
=
2\phi_0(0)m_0
+
\Delta_\beta
\E\left[
g'\left(
\theta^\top P_{v^\star}^\perp X
+\frac{\Delta_\beta}{2}Z
\right)
\1\{Z>0\}
\right].
\]
Since $\theta^\top P_{v^\star}^\perp X=\theta^\top X=U$, the
definition of $\mu^\star$ gives
\begin{equation}
\kappa^\star
=
2\phi_0(0)m_0
+
\Delta_\beta\mu^\star.
\label{eq:app-kappa-identity}
\end{equation}
Substituting \eqref{eq:app-kappa-identity} into
\eqref{eq:app-Dh-pre-kappa}, the terms involving $m_0$ cancel:
\begin{align*}
Dh(v^\star)[u]
&=
\left\{
-\left(
2\phi_0(0)m_0+\Delta_\beta\mu^\star
\right)
+
2\phi_0(0)m_0
\right\}u
+
2\phi_0(0)m_2\theta(\theta^\top u)\\
&=
-\Delta_\beta\mu^\star u
+
2\phi_0(0)m_2\theta(\theta^\top u).
\end{align*}
Thus,
\[
Dh(v^\star)[u]
=
-\Delta_\beta\mu^\star u
+
2\phi_0(0)m_2\theta(\theta^\top u),
\]
which proves Proposition~\ref{prop:local_jacobian}.
\end{proof}

\subsection{Proof of the local stability corollary}
\label{app:local_stability}

\begin{proof}[Proof of Corollary~\ref{cor:local_stability}]
We first show that the derivative of the population tangent field at
$v^\star$ is strictly contracting along every tangent direction.
Recall from Proposition~\ref{prop:local_jacobian} that, for every
$u\in T_{v^\star}\mathbb S^{d-1}$,
$
Dh(v^\star)[u]
=
-\Delta_\beta\mu^\star u
+
2\phi_0(0)m_2\theta(\theta^\top u).
$
Taking the inner product with $u$ gives
\begin{align*}
\left\langle
u,Dh(v^\star)[u]
\right\rangle
&=
-\Delta_\beta\mu^\star\|u\|_2^2
+
2\phi_0(0)m_2(\theta^\top u)^2\\
&\leq
-\Delta_\beta\mu^\star\|u\|_2^2
+
2\phi_0(0)(m_2)_+(\theta^\top u)^2.
\end{align*}
By Cauchy--Schwarz and
$\|\theta\|_2^2=1-\Delta_\beta^2/4$,
\[
(\theta^\top u)^2
\leq
\|\theta\|_2^2\|u\|_2^2
=
\left(1-\frac{\Delta_\beta^2}{4}\right)\|u\|_2^2.
\]
Consequently,
\begin{align*}
\left\langle
u,Dh(v^\star)[u]
\right\rangle
&\leq
-\left\{
\Delta_\beta\mu^\star
-
2\phi_0(0)(m_2)_+
\left(1-\frac{\Delta_\beta^2}{4}\right)
\right\}
\|u\|_2^2
=
-\lambda_\star\|u\|_2^2,
\end{align*}
where $\lambda_\star$ is the decision stability coefficient defined
in \eqref{eq:decision_stability}. Thus, whenever
$\lambda_\star>0$, the linearized population field is strictly
contracting in every nonzero tangent direction at $v^\star$.

We next show that this infinitesimal contraction extends to a
neighborhood of $v^\star$ on the sphere. Let
$v\in\mathbb S^{d-1}$ and write
$
\Delta=v-v^\star.
$
Decompose $\Delta$ into its tangent and normal components at
$v^\star$:
\[
\Delta
=
u+\{(v^\star)^\top\Delta\}v^\star,
\qquad
u=P_{v^\star}^\perp\Delta
\in T_{v^\star}\mathbb S^{d-1}.
\]
Since both $v$ and $v^\star$ are unit vectors,
$
\|\Delta\|_2^2
=
\|v-v^\star\|_2^2
=
2-2v^\top v^\star.
$
Hence
\[
(v^\star)^\top\Delta
=
(v^\star)^\top(v-v^\star)
=
v^\top v^\star-1
=
-\frac12\|\Delta\|_2^2.
\]
Therefore, the decomposition above can be written exactly as
\begin{equation}
\Delta
=
u-\frac12\|\Delta\|_2^2v^\star.
\label{eq:app-delta-decomposition}
\end{equation}
Because $u\perp v^\star$, it follows from
\eqref{eq:app-delta-decomposition} that
\begin{align}
\|\Delta\|_2^2
&=
\|u\|_2^2
+
\frac14\|\Delta\|_2^4,
\nonumber\\
\|u\|_2^2
&=
\|\Delta\|_2^2
-
\frac14\|\Delta\|_2^4.
\label{eq:app-tangent-delta-relation}
\end{align}
In particular,
\[
\|u\|_2^2
=
\|\Delta\|_2^2
+
o(\|\Delta\|_2^2)
\qquad
\text{as }v\to v^\star.
\]

We now justify the first-order expansion of the population tangent
field used below. For
$w\in T_{v^\star}\mathbb S^{d-1}$ with $\|w\|_2<1$, define the local
parametrization
\[
q(w)
=
\sqrt{1-\|w\|_2^2}\,v^\star+w.
\]
Then $q(w)\in\mathbb S^{d-1}$, $q(0)=v^\star$, and
$Dq(0)[w]=w$. Moreover, writing
$Z=(v^\star)^\top X$ and using
$w^\top X=w^\top P_{v^\star}^\perp X$, we have
\[
q(w)^\top X
=
\sqrt{1-\|w\|_2^2}\,Z
+
w^\top P_{v^\star}^\perp X.
\]
Thus,
\[
q(w)^\top X\geq0
\quad\Longleftrightarrow\quad
Z
\geq
-
\frac{
w^\top P_{v^\star}^\perp X
}{
\sqrt{1-\|w\|_2^2}
}.
\]
Conditional on $P_{v^\star}^\perp X$, the right-hand side is a
smoothly varying boundary for the one-dimensional Gaussian integral
in $Z$. Differentiating this representation with respect to $w$,
the resulting derivatives are obtained by evaluating the integrand
at this moving boundary and multiplying by derivatives of the
boundary. For $w$ in a sufficiently small neighborhood of zero,
Assumption~\ref{assum:polygrowth_g} implies that these derivatives
are bounded uniformly by an integrable polynomial function of
$\|X\|_2$. Since $X$ is Gaussian, dominated convergence permits
differentiation under the expectation and shows that $F\circ q$ is
continuously differentiable near zero. Since
$P_{q(w)}^\perp$ is smooth in $w$, it follows that
$h\circ q$ is also continuously differentiable near zero.

For $v$ sufficiently close to $v^\star$, we have
$v^\top v^\star>0$. Since
$
u=P_{v^\star}^\perp v,
$
the unit-norm constraint gives
\[
v
=
\sqrt{1-\|u\|_2^2}\,v^\star+u
=
q(u).
\]
Hence, using $h(v^\star)=0$ and the differentiability of $h\circ q$
at zero,
\begin{equation}
h(v)
=
Dh(v^\star)[u]+r(v),
\qquad
\|r(v)\|_2=o(\|u\|_2).
\label{eq:app-h-local-expansion-u}
\end{equation}
By \eqref{eq:app-tangent-delta-relation},
$\|u\|_2/\|\Delta\|_2\to1$ as $v\to v^\star$. Therefore,
\begin{equation}
h(v)
=
Dh(v^\star)[u]+r(v),
\qquad
\|r(v)\|_2=o(\|\Delta\|_2).
\label{eq:app-h-local-expansion}
\end{equation}

Using \eqref{eq:app-delta-decomposition} and
\eqref{eq:app-h-local-expansion}, we obtain
\begin{align*}
\langle \Delta,h(v)\rangle
&=
\left\langle
u-\frac12\|\Delta\|_2^2v^\star,
Dh(v^\star)[u]+r(v)
\right\rangle\\
&=
\left\langle
u,Dh(v^\star)[u]
\right\rangle
+
\langle u,r(v)\rangle
-
\frac12\|\Delta\|_2^2
\left\langle
v^\star,Dh(v^\star)[u]
\right\rangle
-
\frac12\|\Delta\|_2^2
\langle v^\star,r(v)\rangle.
\end{align*}
We now bound the three remainder terms. First,
\[
|\langle u,r(v)\rangle|
\leq
\|u\|_2\|r(v)\|_2
=
o(\|\Delta\|_2^2).
\]
Next, because $Dh(v^\star)$ is a linear map on the
finite-dimensional tangent space, its operator norm is finite.
Hence, for some constant $C<\infty$,
$
\|Dh(v^\star)[u]\|_2
\leq
C\|u\|_2
\leq
C\|\Delta\|_2,
$
and therefore
\[
\left|
\frac12\|\Delta\|_2^2
\left\langle
v^\star,Dh(v^\star)[u]
\right\rangle
\right|
\leq
\frac{C}{2}\|\Delta\|_2^3
=
o(\|\Delta\|_2^2).
\]
Finally,
\[
\left|
\frac12\|\Delta\|_2^2
\langle v^\star,r(v)\rangle
\right|
\leq
\frac12\|\Delta\|_2^2\|r(v)\|_2
=
o(\|\Delta\|_2^2).
\]
Combining these bounds yields
\begin{equation}
\langle \Delta,h(v)\rangle
=
\left\langle
u,Dh(v^\star)[u]
\right\rangle
+
o(\|\Delta\|_2^2).
\label{eq:app-local-inner-product}
\end{equation}
Using the infinitesimal contraction established above together with
\eqref{eq:app-tangent-delta-relation},
\begin{align*}
\langle \Delta,h(v)\rangle
&\leq
-\lambda_\star\|u\|_2^2
+
o(\|\Delta\|_2^2)
=
-\lambda_\star\|\Delta\|_2^2
+
o(\|\Delta\|_2^2).
\end{align*}
Equivalently,
\[
\left\langle
v-v^\star,h(v)
\right\rangle
\leq
-\lambda_\star\|v-v^\star\|_2^2
+
o(\|v-v^\star\|_2^2)
\]
as $v\to v^\star$ on $\mathbb S^{d-1}$.

Because $\lambda_\star>0$, there exists $r_0>0$ such that whenever
$v\in\mathbb S^{d-1}$ satisfies
$\|v-v^\star\|_2\leq r_0$, the remainder term satisfies
\[
\left|
o(\|v-v^\star\|_2^2)
\right|
\leq
\frac{\lambda_\star}{2}
\|v-v^\star\|_2^2.
\]
Therefore,
\[
\left\langle
v-v^\star,h(v)
\right\rangle
\leq
-\frac{\lambda_\star}{2}
\|v-v^\star\|_2^2.
\]
Thus, the claimed neighborhood contraction holds with, for example,
$
\lambda_0=\lambda_\star/2>0.
$
\end{proof}

\section{Additional Details on Elicitation Geometry}
\label{app:elicitation}

This appendix provides additional details for the elicitation
interpretation developed in Section~\ref{sec:elicitation}. We first
formalize the Bregman representation of the conditional mean, then
describe the associated primal--dual geometry. We next show that a
strictly increasing response link itself induces a convex potential,
and finally relate this geometry explicitly to the decision stability
coefficient and the self-concordance condition.

\subsection{Bregman elicitation of the conditional mean}
\label{app:bregman_mean}
For completeness, we briefly recall the relevant definitions; a more
detailed treatment of consistency and elicitability is given in
\cite{gneiting2011making}.
Recall that $\mathcal F$ denotes a class of distributions and
$T:\mathcal F\to\R$ the statistical functional of interest. A loss
$\ell(r,y)$ is strictly consistent for $T$ if, for every
$F\in\mathcal F$,
\[
T(F)
=
\arg\min_r \E_F[\ell(r,Y)],
\]
whenever the expectation is well defined. A functional admitting a
strictly consistent loss is called elicitable.
We specialize to the mean functional. Let $\mathcal I\subseteq\R$ be
an interval containing the relevant reward and prediction values, and
let $\varphi:\mathcal I\to\R$ be differentiable and strictly convex.
The Bregman loss \cite{savage1971elicitation} generated by $\varphi$ is
\begin{equation}
\ell_\varphi(r,y)
=
\varphi(y)-\varphi(r)-\varphi'(r)(y-r).
\label{eq:app_bregman_loss}
\end{equation}
The following standard calculation shows that every such loss elicits
the mean.

\begin{lemma}[Bregman consistency for the mean]
\label{lem:bregman_mean}
Let $F\in\mathcal F$ have finite mean
$\mu_F=\E_F[Y]\in\mathcal I$, and suppose that the expectations below
are finite. Then
\[
\E_F[\ell_\varphi(r,Y)]
-
\E_F[\ell_\varphi(\mu_F,Y)]
=
\ell_\varphi(r,\mu_F).
\]
Consequently, if $\varphi$ is strictly convex,
\[
\mu_F
=
\arg\min_{r\in\mathcal I}
\E_F[\ell_\varphi(r,Y)].
\]
\end{lemma}

\begin{proof}
By \eqref{eq:app_bregman_loss},
\[
\E_F[\ell_\varphi(r,Y)]
=
\E_F[\varphi(Y)]
-
\varphi(r)
-
\varphi'(r)(\mu_F-r),
\]
whereas
$
\E_F[\ell_\varphi(\mu_F,Y)]
=
\E_F[\varphi(Y)]-\varphi(\mu_F)
$.
Subtracting gives
\[
\E_F[\ell_\varphi(r,Y)]
-
\E_F[\ell_\varphi(\mu_F,Y)]
=
\varphi(\mu_F)-\varphi(r)-\varphi'(r)(\mu_F-r)
=
\ell_\varphi(r,\mu_F).
\]
Strict convexity implies that the final expression is nonnegative and
vanishes only when $r=\mu_F$.
\end{proof}

Applying Lemma~\ref{lem:bregman_mean} to the conditional distribution
$F_{a,x}$ gives
$
T(F_{a,x})
=
m_a(x)
=
\arg\min_r
\E_{F_{a,x}}[\ell_\varphi(r,Y)],
$
which is the representation used in Section~\ref{sec:elicitation}.
If $\varphi$ is twice continuously differentiable with
$\varphi''(r)>0$, define
$
w(r)=\varphi''(r).
$
Differentiating \eqref{eq:app_bregman_loss} with respect to the
prediction gives
\begin{equation}
\frac{\partial}{\partial r}\ell_\varphi(r,y)
=
w(r)(r-y).
\label{eq:app_bregman_derivative}
\end{equation}
Thus, $w(r)$ locally weights the usual mean identification function
$r-y$. Equivalently, a second-order expansion around $r$ gives
\[
\ell_\varphi(r,y)
=
\frac12 w(r)(y-r)^2
+
o\{(y-r)^2\},
\qquad y\to r,
\]
so $w=\varphi''$ describes the local geometry of the loss on the
mean-prediction scale.
Lemma~\ref{lem:bregman_mean} also makes clear that the mean does not
determine a unique loss: different strictly convex potentials can
elicit the same mean functional while inducing different local
geometries.
\subsection{Primal--dual geometry and response maps}
\label{app:primal_dual}

We next make precise the primal--dual representation used in
Section~\ref{sec:elicitation}. Let
$\mathcal I\subseteq\R$ be an open interval and suppose that
$\varphi:\mathcal I\to\R$ is differentiable and strictly convex.
Its convex conjugate is given by~\cite[Section 7.3.3]{deisenroth20mathml},
\begin{equation}
\varphi^*(\vartheta)
=
\sup_{r\in\mathcal I}
\{r\vartheta-\varphi(r)\}.
\label{eq:app_convex_conjugate}
\end{equation}
For a fixed $\vartheta$, an interior maximizer of
\eqref{eq:app_convex_conjugate} satisfies
$
\vartheta-\varphi'(r)=0.
$
Thus, the primal and dual coordinates are related by
$
\vartheta=\varphi'(r).
$
Since $\varphi$ is strictly convex, $\varphi'$ is strictly increasing
and therefore one-to-one. Let
$
\mathcal J=\varphi'(\mathcal I)
$
denote its range. For $\vartheta\in\mathcal J$,
$
r
=
(\varphi')^{-1}(\vartheta).
$
Under standard Legendre-type regularity conditions, the conjugate is
differentiable on $\mathcal J$ and
$
(\varphi^*)'(\vartheta)
=
(\varphi')^{-1}(\vartheta).
$
Thus, the response map induced by $\varphi$ may equivalently be
written as
$
g_\varphi(\vartheta)
=
(\varphi')^{-1}(\vartheta)
=
(\varphi^*)'(\vartheta).
$
Geometrically, $\vartheta=\varphi'(r)$ is the slope of the tangent to
the graph of $\varphi$ at the primal coordinate $r$. Strict convexity
ensures that different primal coordinates have different tangent
slopes. The inverse map $(\varphi')^{-1}$ therefore recovers the
primal coordinate from its tangent slope. For the shared-link model,
suppose that the conditional mean functional satisfies
\begin{equation}
\varphi'\{T(F_{a,x})\}
=
\varphi'\{m_a(x)\}
=
\beta_a^\top x.
\label{eq:app_dual_single_index}
\end{equation}
Whenever $\beta_a^\top x\in\mathcal J$, inversion gives
\begin{equation}
m_a(x)
=
(\varphi')^{-1}(\beta_a^\top x)
=
g_\varphi(\beta_a^\top x).
\label{eq:app_primal_single_index}
\end{equation}
A common potential $\varphi$ across actions therefore induces a
common response geometry through $g_\varphi$, while the linear
representations $\beta_a^\top x$ remain action-specific.
There is a domain issue in this representation that is particularly
relevant under our Gaussian context model. Since
$
X_t\sim N(0,I_d)
$
and
$
\|\beta_a\|_2=1,
$
we have
$
\beta_a^\top X_t\sim N(0,1),
$
so the dual coordinate has support $\R$. Consequently, for
\eqref{eq:app_dual_single_index} to hold on the full support of the
index, we require
\begin{equation}
\varphi'(\mathcal I)=\R.
\label{eq:app_full_dual_domain}
\end{equation}
Under \eqref{eq:app_full_dual_domain},
$
g_\varphi=(\varphi')^{-1}
$
is defined on all of $\R$. Several familiar convex potentials illustrate both the primal--dual
construction and this domain requirement.
\begin{table}[ht]
\centering
\caption{Examples of convex potentials, Bregman weights, dual
coordinates, and induced response maps.}
\label{tab:app_bregman_examples}
\begin{tabular}{llll}
\hline
$\varphi(r)$
&
$w(r)=\varphi''(r)$
&
$\vartheta=\varphi'(r)$
&
$g_\varphi(\vartheta)$
\\
\hline
$\frac12r^2$
&
$1$
&
$r$
&
$\vartheta$
\\[2mm]

$r\log r-r$
&
$\frac1r$
&
$\log r$
&
$\exp(\vartheta)$
\\[2mm]

$-\log r$
&
$\frac1{r^2}$
&
$-\frac1r$
&
$-\frac1\vartheta$
\\[2mm]

$r\log r+(1-r)\log(1-r)$
&
$\frac{1}{r(1-r)}$
&
$\log\!\left(\frac{r}{1-r}\right)$
&
$\displaystyle
\frac{\exp(\vartheta)}
{1+\exp(\vartheta)}
$
\\
\hline
\end{tabular}
\end{table}

The quadratic potential and binary negative entropy have dual domain
$\R$, producing the identity and logistic response maps,
respectively. Negative entropy produces the exponential response map,
while Burg entropy produces an inverse-type response map on a
restricted dual domain. Thus, the primal--dual construction itself is
more general than the Gaussian-index model considered in this paper:
compatibility additionally requires the dual domain to contain the
support of the linear index.
The construction above begins with a convex potential $\varphi$ and
derives the corresponding response map
$g_\varphi=(\varphi')^{-1}$. Our model is specified in the opposite
direction, beginning with an unknown strictly increasing link $g$.
We next show that such a link itself induces a convex potential and
therefore admits the same primal-dual interpretation.

\subsection{From a monotone response link to a convex potential}
\label{app:link_to_potential}

The preceding discussion begins with a convex potential $\varphi$ and
constructs the response map
$
g_\varphi=(\varphi')^{-1}.
$
For the model in this paper, however, it is also useful to consider the
reverse direction. We begin with an unknown strictly increasing link
$g$ and ask whether it induces a convex potential having $g$ as its
primal-recovery map.
Let
$
\mathcal M=g(\R)
$
denote the range of $g$. Since $g$ is strictly increasing, it is
one-to-one and therefore admits an inverse
$
g^{-1}:\mathcal M\to\R.
$
Fix an arbitrary reference point $m_0\in\mathcal M$ and define
\begin{equation}
\varphi_g(m)
=
\int_{m_0}^m g^{-1}(s)\,ds,
\qquad
m\in\mathcal M.
\label{eq:app_potential_from_link}
\end{equation}
Then
$
\varphi_g'(m)
=
g^{-1}(m).
$
Because $g^{-1}$ is strictly increasing, $\varphi_g$ is strictly
convex. Moreover,
$
\varphi_g'\{g(\vartheta)\}
=
g^{-1}\{g(\vartheta)\}
=
\vartheta,
$
and hence
$
g
=
(\varphi_g')^{-1}.
$
This gives the following representation result.

\begin{proposition}[Convex potential induced by a monotone link]
\label{prop:link_induced_potential}
Let $g:\R\to\mathcal M$ be continuous and strictly increasing, where
$\mathcal M=g(\R)$. Define $\varphi_g$ by
\eqref{eq:app_potential_from_link}. Then $\varphi_g$ is differentiable
and strictly convex on $\mathcal M$, with
$
\varphi_g'=g^{-1}.
$
Consequently,
$
g=(\varphi_g')^{-1}.
$
Therefore, the shared-link model
$
m_a(x)=g(\beta_a^\top x)
$
can equivalently be written in the dual-coordinate form
$
\varphi_g'\{m_a(x)\}
=
\beta_a^\top x.
$
\end{proposition}

\begin{proof}
Continuity and strict monotonicity of $g$ imply that its inverse
$g^{-1}$ is continuous and strictly increasing on $\mathcal M$.
The fundamental theorem of calculus applied to
\eqref{eq:app_potential_from_link} therefore gives
\[
\varphi_g'(m)=g^{-1}(m).
\]
Since the derivative of $\varphi_g$ is strictly increasing,
$\varphi_g$ is strictly convex. Finally,
\[
\varphi_g'\{g(\vartheta)\}
=
g^{-1}\{g(\vartheta)\}
=
\vartheta,
\]
which establishes
$g=(\varphi_g')^{-1}$ and the stated dual-coordinate representation.
\end{proof}

The choice of reference point $m_0$ affects $\varphi_g$ only through
an additive constant and therefore does not affect either its
derivative or the induced response map. More generally, Bregman
divergences are invariant to affine modifications of their generating
potential. For the particular dual-coordinate representation above,
however, the normalization
$\varphi_g'=g^{-1}$ fixes the linear component. Thus, the geometry
relevant to the associated Bregman loss is determined by the
curvature of the potential rather than by its additive normalization.

\subsection{Elicitation geometry and decision stability}
\label{app:elicitation_stability}

The preceding construction also makes precise the relationship between
the geometry of the convex potential and the shape of the induced
response link. Recall that $g=(\varphi')^{-1}$, so that
$\varphi'\{g(z)\}=z$. Differentiating gives
$\varphi''\{g(z)\}g'(z)=1$, and hence
$
g'(z)=1/\varphi''\{g(z)\}.
$
If $\varphi$ is three times differentiable, differentiating once more
gives
$
g''(z)
=
-\varphi'''\{g(z)\}/\{\varphi''\{g(z)\}\}^3.
$
Thus, the slope and curvature of the response link are determined by
the local curvature and variation of the elicitation geometry. In
particular, smaller Bregman curvature corresponds to greater
dual-to-primal sensitivity, while variation in this curvature
determines how the sensitivity changes along the index.
These identities connect the elicitation geometry directly to the
local population dynamics of NBL. Recall that
$U=\theta^\top X$ and $Z=(v^\star)^\top X$. Then
\begin{equation}
\mu^\star
=
\E\left[
\frac{\1\{Z>0\}}
{\varphi''\left\{
g\left(U+\frac{\Delta_\beta}{2}Z\right)
\right\}}
\right],
\label{eq:app_mustar_elicitation}
\end{equation}
while
\begin{equation}
m_2
=
-
\E\left[
\frac{\varphi'''\{g(U)\}}
{\{\varphi''\{g(U)\}\}^3}
\right].
\label{eq:app_m2_elicitation}
\end{equation}
The two quantities therefore capture different aspects of the same
elicitation geometry. The quantity $\mu^\star$ is governed by the
dual-to-primal sensitivity $1/\varphi''(m)$. Since
$dm/dz=g'(z)=1/\varphi''(m)$, it measures how strongly a local change
in the dual coordinate is translated into a change in the conditional
mean. By contrast, $m_2$ captures the average variation of this
sensitivity across the population through $g''$.
Substituting \eqref{eq:app_mustar_elicitation} and
\eqref{eq:app_m2_elicitation} into the decision stability coefficient
gives
\begin{align}
\lambda_\star
={}&
\Delta_\beta
\E\left[
\frac{\1\{Z>0\}}
{\varphi''\left\{
g\left(U+\frac{\Delta_\beta}{2}Z\right)
\right\}}
\right]
-
2\phi_0(0)
\left(1-\frac{\Delta_\beta^2}{4}\right)
\left(
-
\E\left[
\frac{\varphi'''\{g(U)\}}
{\{\varphi''\{g(U)\}\}^3}
\right]
\right)_+.
\label{eq:elicitation-stability-coefficient}
\end{align}
Thus, the decision stability coefficient depends jointly on the
separation of the arm-specific indices and on the geometry through
which this separation is translated from the dual scale to the reward
scale. The first term captures the restoring signal associated with
dual-to-primal sensitivity, while the second captures the possible
effect of variation in this sensitivity across the population.
Importantly, the latter depends on the signed population average
$m_2=\E[g''(U)]$, so local curvature may cancel across different
regions of the index space. When $m_2>0$, this average curvature
effect can weaken the restoring component, whereas when $m_2\leq0$,
it cannot oppose the restoring movement.

\subsection{Self-concordance and decision stability}
\label{app:self_concordance}

We next relate the self-concordance condition used in
Section~\ref{sec:elicitation} to the curvature of the induced convex
potential. Recall that
$
g=(\varphi')^{-1}.
$
Differentiating the identity
$
\varphi'\{g(z)\}=z
$
gives
\[
g'(z)
=
\frac{1}{\varphi''\{g(z)\}},
\]
and differentiating once more yields
\[
g''(z)
=
-
\frac{\varphi'''\{g(z)\}}
{\{\varphi''\{g(z)\}\}^3}.
\]
Consequently, the generalized self-concordance condition
$
|g''(z)|\leq Rg'(z),
 z\in\R,
$
is equivalent, with $m=g(z)$, to
$
|\varphi'''(m)|
\leq
R\{\varphi''(m)\}^2.
$
Equivalently,
\[
\left|
\frac{d}{dm}\frac{1}{\varphi''(m)}
\right|
\leq R.
\]
Thus, self-concordance controls how rapidly the local
dual-to-primal sensitivity $1/\varphi''(m)$ can vary over the primal
mean space. The following proof shows how this control bounds the
curvature contribution to the decision stability coefficient relative
to its restoring component.
\begin{proof}[Proof of Proposition~\ref{prop:self-concordance-stability}]
Recall that
$
\lambda_\star
=
\Delta_\beta\mu^\star
-
2\phi_0(0)
\left(1-({\Delta_\beta^2}/{4})\right)
(m_2)_+.
$
Let $U=\theta^\top X$ and $Z=(v^\star)^\top X$. Then
$
U\sim N(0,1-\Delta_\beta^2/4)
$,
$Z\sim N(0,1)$, and $U\perp Z$, while
$
\mu^\star
=
\E[
g'\left(U+({\Delta_\beta}/{2})Z\right)
\1\{Z>0\}
],
m_2=\E[g''(U)].
$
We first control the curvature term. By self-concordance,
$|g''(z)|\leq Rg'(z)$ for every $z\in\R$. Therefore,
\begin{equation}
(m_2)_+
\leq
|m_2|
\leq
\E[|g''(U)|]
\leq
R\E[g'(U)].
\label{eq:self-concordance-m2-bound}
\end{equation}

It remains to relate $\E[g'(U)]$ to $\mu^\star$. Since $g'>0$,
the self-concordance condition implies
$
|d\log g'(z)/dz|\leq R.
$
Hence, for any $z_1,z_2\in\R$,
\[
|\log g'(z_1)-\log g'(z_2)|
\leq
R|z_1-z_2|.
\]
Taking
$
z_1=U+(\Delta_\beta/2)Z
$
and $z_2=U$, on $\{Z>0\}$ we obtain
\[
\log g'\left(U+\frac{\Delta_\beta}{2}Z\right)
\geq
\log g'(U)-\frac{R\Delta_\beta}{2}Z,
\]
and therefore
\[
g'\left(U+\frac{\Delta_\beta}{2}Z\right)
\geq
\exp\left(-\frac{R\Delta_\beta}{2}Z\right)g'(U),
\qquad Z>0.
\]
Using the independence of $U$ and $Z$,
\begin{align}
\mu^\star
&=
\E\left[
g'\left(U+\frac{\Delta_\beta}{2}Z\right)
\1\{Z>0\}
\right]
\geq
\E[g'(U)]
\E\left[
\exp\left(-\frac{R\Delta_\beta}{2}Z\right)
\1\{Z>0\}
\right]
=
\kappa\left(\frac{R\Delta_\beta}{2}\right)
\E[g'(U)],
\label{eq:self-concordance-mu-bound}
\end{align}
where, for $a\geq0$,
$
\kappa(a)
=
\E[e^{-aZ}\1\{Z>0\}]
=
e^{a^2/2}\Phi(-a).
$
The final equality follows by completing the square in the standard
normal integral. It follows from
\eqref{eq:self-concordance-mu-bound} that
$
\E[g'(U)]
\leq
{\mu^\star}/
{(\kappa(R\Delta_\beta/2))}.
$
Combining this with
\eqref{eq:self-concordance-m2-bound} gives
\begin{equation}
(m_2)_+
\leq
\frac{R}{\kappa(R\Delta_\beta/2)}
\mu^\star.
\label{eq:self-concordance-curvature-control}
\end{equation}
Substituting \eqref{eq:self-concordance-curvature-control} into the
definition of $\lambda_\star$ yields
\[
\lambda_\star
\geq
2\mu^\star
\left\{
\frac{\Delta_\beta}{2}
-
\phi_0(0)
\left(1-\frac{\Delta_\beta^2}{4}\right)
\frac{R}{\kappa(R\Delta_\beta/2)}
\right\}.
\]
Since $\mu^\star>0$, the right-hand side is strictly positive whenever
\[
\frac{\Delta_\beta}{2}
>
\phi_0(0)
\left(1-\frac{\Delta_\beta^2}{4}\right)
\frac{R}{\kappa(R\Delta_\beta/2)}.
\]
Thus $\lambda_\star>0$ under the stated condition, which proves the
result.
\end{proof}
\section{Proofs for the stochastic analysis and regret}
\label{app:stochastic_analysis}

We first record the geometric normalization fact used in deriving
\eqref{eq:pathwise-error-recursion}.

\begin{lemma}[Effect of normalization]
\label{lem:normalization}
Let $x,u\in\R^d$ satisfy $\|x\|_2\geq1$ and $\|u\|_2=1$. Then
\[
\left\|
\frac{x}{\|x\|_2}-u
\right\|_2
\leq
\|x-u\|_2.
\]
\end{lemma}

\begin{proof}
Let $r=\|x\|_2\geq1$. Since $\|u\|_2=1$,
\[
\left\|\frac{x}{r}-u\right\|_2^2
=
2-\frac{2x^\top u}{r},
\qquad
\|x-u\|_2^2
=
r^2+1-2x^\top u.
\]
Therefore,
\begin{align*}
\|x-u\|_2^2
-
\left\|
\frac{x}{r}-u
\right\|_2^2
=
r^2-1-2x^\top u\left(1-\frac1r\right)
=
(r-1)
\left(
r+1-\frac{2x^\top u}{r}
\right).
\end{align*}
By Cauchy--Schwarz, $x^\top u/r\leq1$, so the second factor is at
least $r-1\geq0$. Hence the difference is nonnegative, proving the
claim.
\end{proof}
Since $G_t\in T_{v_{t-1}}\mathbb S^{d-1}$,
$v_{t-1}^\top G_t=0$, and hence
$
\|v_{t-1}+\eta_tG_t\|_2^2
=
1+\eta_t^2\|G_t\|_2^2
\geq1.
$
Applying Lemma~\ref{lem:normalization} with
$x=v_{t-1}+\eta_tG_t$ and $u=v^\star$ therefore gives
\begin{align*}
e_t
=
\left\|
\frac{v_{t-1}+\eta_tG_t}
{\|v_{t-1}+\eta_tG_t\|_2}
-v^\star
\right\|_2^2
\leq
\|v_{t-1}+\eta_tG_t-v^\star\|_2^2
=
e_{t-1}
+
2\eta_t
\langle v_{t-1}-v^\star,G_t\rangle
+
\eta_t^2\|G_t\|_2^2,
\end{align*}
which gives \eqref{eq:pathwise-error-recursion}.
For completeness, define
$
\xi_t=G_t-h(v_{t-1}).
$
By Lemma~\ref{lem:conditional_population_field},
$
\E[\xi_t\mid\mathcal H_{t-1}]=0.
$ Substituting $G_t=h(v_{t-1})+\xi_t$ into the preceding inequality
gives the stochastic decomposition
\eqref{eq:stochastic-error-decomposition}.
\begin{proof}[Proof of Lemma~\ref{lem:local_one_step}]
Taking conditional expectation in
\eqref{eq:pathwise-error-recursion} given $\mathcal H_{t-1}$ and
using Lemma~\ref{lem:conditional_population_field} gives
\begin{align}
\E[e_t\mid\mathcal H_{t-1}]
\leq{}&
e_{t-1}
+
2\eta_t
\left\langle
v_{t-1}-v^\star,h(v_{t-1})
\right\rangle
+
\eta_t^2
\E[\|G_t\|_2^2\mid\mathcal H_{t-1}].
\label{eq:app-conditional-error-before-bounds}
\end{align}
Here $v_{t-1}$ and $\eta_t$ are
$\mathcal H_{t-1}$-measurable.
It remains to control the conditional second moment of the stochastic
update. Since
$
G_t=P_{v_{t-1}}^\perp W_t
$
and $P_{v_{t-1}}^\perp$ is an orthogonal projection,
$
\|G_t\|_2\leq\|W_t\|_2.
$
Moreover, $W_t=A_tY_tX_t$ and $A_t\in\{-1,+1\}$, so
\[
\|G_t\|_2^2
\leq
Y_t^2\|X_t\|_2^2
\leq
B_Y^2\|X_t\|_2^2.
\]
Since $X_t$ is independent of $\mathcal H_{t-1}$ and
$X_t\sim N(0,I_d)$,
\[
\E[\|G_t\|_2^2\mid\mathcal H_{t-1}]
\leq
B_Y^2\E\|X_t\|_2^2
=
B_Y^2d.
\]
Finally, if $\|v_{t-1}-v^\star\|_2\leq r_0$, then
Corollary~\ref{cor:local_stability} gives
\[
\left\langle
v_{t-1}-v^\star,h(v_{t-1})
\right\rangle
\leq
-\lambda_0\|v_{t-1}-v^\star\|_2^2
=
-\lambda_0e_{t-1}.
\]
Substituting these bounds into
\eqref{eq:app-conditional-error-before-bounds} yields
\[
\E[e_t\mid\mathcal H_{t-1}]
\leq
(1-2\lambda_0\eta_t)e_{t-1}
+
B_Y^2d\,\eta_t^2,
\]
as claimed.
\end{proof}
\subsection{Finite-horizon localization}
\label{app:localization}

The local contraction in Lemma~\ref{lem:local_one_step} holds only
while the iterate remains within the stability neighborhood of
$v^\star$. Recall the first exit time
$
\tau
=
\inf\left\{
t\geq n_0:
\|v_t-v^\star\|_2>r_0
\right\}.
$
We now control the stochastic trajectory up to this exit time.
For $t>n_0$, \eqref{eq:stochastic-error-decomposition} and
Corollary~\ref{cor:local_stability} imply
\begin{align}
(e_t-e_{t-1})\1\{t\leq\tau\}
\leq{}&
-2\lambda_0\eta_t e_{t-1}\1\{t\leq\tau\}
\nonumber\\
&+
2\eta_t
\langle v_{t-1}-v^\star,\xi_t\rangle
\1\{t\leq\tau\}
+
\eta_t^2\|G_t\|_2^2\1\{t\leq\tau\}.
\label{eq:stopped_one_step_inequality}
\end{align}
Since $\{t\leq\tau\}\in\mathcal H_{t-1}$, the stopping indicator is
predictable and therefore preserves the martingale-difference
structure of the fluctuation term.

For $t\geq n_0$, define
\begin{align}
M_t
&=
2\sum_{s=n_0+1}^t
\eta_s
\left\langle
v_{s-1}-v^\star,\xi_s
\right\rangle
\1\{s\leq\tau\},
\label{eq:localization-martingale}
\\
Q_t
&=
\sum_{s=n_0+1}^t
\eta_s^2\|G_s\|_2^2
\1\{s\leq\tau\}.
\label{eq:localization-quadratic}
\end{align}
Summing \eqref{eq:stopped_one_step_inequality} from
$s=n_0+1$ to $t$ gives
\begin{align}
e_{t\wedge\tau}
\leq
e_{n_0}
-
2\lambda_0
\sum_{s=n_0+1}^t
\eta_s e_{s-1}\1\{s\leq\tau\}
+
M_t+Q_t
\leq
e_{n_0}+M_t+Q_t,
\label{eq:stopped-error-bound}
\end{align}
where the last inequality follows by dropping the nonpositive drift
term.

We use the harmonic step size
\begin{equation}
\eta_t=\frac{\gamma}{t+t_0}.
\label{eq:step_size}
\end{equation}
Then
$
\sum_{t\geq1}\eta_t=\infty
$
and
$
\sum_{t\geq1}\eta_t^2<\infty,
$
so the contractive drift can accumulate while the second-order
stochastic fluctuations remain summable. The same choice will also
yield the $t^{-1}$ localized estimation rate below.
\subsection{Concentration of the stopped stochastic terms}
\label{app:stopped-stochastic-control}

The following lemma controls the martingale and quadratic stochastic
terms in \eqref{eq:stopped-error-bound}.

\begin{lemma}[Control of the stopped stochastic terms]
\label{lem:stopped-stochastic-control}
Let $\eta_t=\gamma/(t+t_0)$. There exist universal constants
$C_M,C_Q>0$ such that, for any $\delta\in(0,1)$ and any finite
horizon $n$, if
$
n_0+t_0
\geq
C_M B_Y^2\gamma^2r_0^{-2}\log(3/\delta),
$
then
\[
\Prob\left(
\max_{n_0\leq t\leq n}M_t>\frac{r_0^2}{4}
\right)
\leq\frac{\delta}{3}.
\]
Similarly, if
$
n_0+t_0
\geq
C_Q B_Y^2\gamma^2r_0^{-2}
\left\{
d+\log(3/\delta)
\right\},
$
then
\[
\Prob\left(
Q_n>\frac{r_0^2}{4}
\right)
\leq\frac{\delta}{3}.
\]
\end{lemma}

The proof of Lemma~\ref{lem:stopped-stochastic-control} uses the
following three concentration results.

\begin{lemma}[Conditional tail-to-MGF bound]
\label{lem:tail_to_mgf}
Let $U$ be a real-valued random variable and let $\mathcal G$ be a
$\sigma$-field. Suppose that, almost surely,
\[
\Prob\left(
|U|\geq x
\,\middle|\,
\mathcal G
\right)
\leq
2\exp\left(
-\frac{x^2}{2\sigma^2}
\right),
\qquad x>0,
\]
where $\sigma$ is a nonnegative $\mathcal G$-measurable random
variable. Then there exists a universal constant $C_{\rm sg}>0$ such
that
\[
\E\left[
\exp\left\{
\lambda
\left(
U-\E[U\mid\mathcal G]
\right)
\right\}
\,\middle|\,
\mathcal G
\right]
\leq
\exp\left(
C_{\rm sg}\lambda^2\sigma^2
\right),
\qquad \lambda\in\R.
\]
\end{lemma}

\begin{proof}
Condition on $\mathcal G$ throughout, so that $\sigma$ may be
regarded as fixed. The assumed tail bound first implies moment
bounds for $U$. For every $p\geq1$, using the tail-integration
formula,
\begin{align*}
\E[|U|^p\mid\mathcal G]
&=
p\int_0^\infty
x^{p-1}
\Prob(|U|\geq x\mid\mathcal G)\,dx
\leq
2p\int_0^\infty
x^{p-1}
\exp\left(-\frac{x^2}{2\sigma^2}\right)\,dx.
\end{align*}
With the change of variables $z=x^2/(2\sigma^2)$, the right-hand
side is
$
p(2\sigma^2)^{p/2}\Gamma(p/2).
$
Hence, using the standard bound
$\Gamma(p/2)^{1/p}\leq C\sqrt{p}$,
$
\left(
\E[|U|^p\mid\mathcal G]
\right)^{1/p}
\leq
C_1\sigma\sqrt{p}
$
for a universal constant $C_1>0$.

Now define the conditionally centered variable
$
Z
=
U-\E[U\mid\mathcal G].
$
By the conditional Minkowski inequality,
\begin{align*}
\left(
\E[|Z|^p\mid\mathcal G]
\right)^{1/p}
&\leq
\left(
\E[|U|^p\mid\mathcal G]
\right)^{1/p}
+
|\E[U\mid\mathcal G]|
\leq
\left(
\E[|U|^p\mid\mathcal G]
\right)^{1/p}
+
\E[|U|\mid\mathcal G].
\end{align*}
Applying the preceding moment bound with $p$ and with $p=1$ gives
\[
\left(
\E[|Z|^p\mid\mathcal G]
\right)^{1/p}
\leq
C_2\sigma\sqrt{p},
\qquad p\geq1,
\]
for another universal constant $C_2>0$. Moreover,
$\E[Z\mid\mathcal G]=0$. The standard equivalence between the
moment and moment-generating-function characterizations of centered
sub-Gaussian random variables therefore yields
\[
\E[e^{\lambda Z}\mid\mathcal G]
\leq
\exp(C_{\rm sg}\lambda^2\sigma^2),
\qquad \lambda\in\R,
\]
for a universal constant $C_{\rm sg}>0$; see, for example,
\cite[Section~2.6]{vershynin2018high}.
\end{proof}

\begin{lemma}[Maximal inequality for conditionally sub-Gaussian
martingales]
\label{lem:maximal-subgaussian-martingale}
Let $\{D_s,\mathcal H_s\}$ be a martingale-difference sequence
satisfying
\[
\E[\exp(\lambda D_s)\mid\mathcal H_{s-1}]
\leq
\exp\left(\frac{\lambda^2\sigma_s^2}{2}\right)
\]
for every $\lambda\in\R$, where $\sigma_s^2$ is deterministic.
Then, for every $x>0$ and $T\geq1$,
\[
\Prob\left(
\max_{1\leq t\leq T}
\sum_{s=1}^{t}D_s
\geq x
\right)
\leq
\exp\left\{
-\frac{x^2}{
2\sum_{s=1}^{T}\sigma_s^2
}
\right\}.
\]
\end{lemma}

\begin{proof}
Let
\[
S_t=\sum_{s=1}^tD_s,
\qquad
V_t=\sum_{s=1}^t\sigma_s^2.
\]
For any $\lambda>0$, define
\[
L_t(\lambda)
=
\exp\left\{
\lambda S_t-\frac{\lambda^2}{2}V_t
\right\}.
\]
We first verify that $\{L_t(\lambda)\}$ is a nonnegative
supermartingale. Since
$S_t=S_{t-1}+D_t$ and
$V_t=V_{t-1}+\sigma_t^2$,
\[
L_t(\lambda)
=
L_{t-1}(\lambda)
\exp\left\{
\lambda D_t-\frac{\lambda^2}{2}\sigma_t^2
\right\}.
\]
Hence, conditioning on $\mathcal H_{t-1}$,
\begin{align*}
\E[L_t(\lambda)\mid\mathcal H_{t-1}]
&=
L_{t-1}(\lambda)
e^{-\lambda^2\sigma_t^2/2}
\E[e^{\lambda D_t}\mid\mathcal H_{t-1}]
\leq
L_{t-1}(\lambda),
\end{align*}
where the inequality follows from the conditional
moment-generating-function assumption. Thus
$\{L_t(\lambda)\}$ is a nonnegative supermartingale, with
$L_0(\lambda)=1$. By Ville's inequality for nonnegative
supermartingales,
\[
\Prob\left(
\max_{1\leq t\leq T}L_t(\lambda)\geq u
\right)
\leq
\frac{1}{u},
\qquad u>0.
\]

Now suppose that $\max_{1\leq t\leq T}S_t\geq x.$
Then there exists some $t^\star\leq T$ such that
$S_{t^\star}\geq x$. Since $V_{t^\star}\leq V_T$,
\begin{align*}
L_{t^\star}(\lambda)
&=
\exp\left\{
\lambda S_{t^\star}
-\frac{\lambda^2}{2}V_{t^\star}
\right\}
\geq
\exp\left\{
\lambda x-\frac{\lambda^2}{2}V_T
\right\}.
\end{align*}
Therefore,
\[
\left\{
\max_{1\leq t\leq T}S_t\geq x
\right\}
\subseteq
\left\{
\max_{1\leq t\leq T}L_t(\lambda)
\geq
\exp\left(
\lambda x-\frac{\lambda^2}{2}V_T
\right)
\right\}.
\]
Applying Ville's inequality gives
\[
\Prob\left(
\max_{1\leq t\leq T}S_t\geq x
\right)
\leq
\exp\left\{
-\lambda x+\frac{\lambda^2}{2}V_T
\right\}.
\]
This holds for every $\lambda>0$. If $V_T>0$, the right-hand side is
minimized at
$
\lambda={x}/{V_T},
$
which yields
\[
\Prob\left(
\max_{1\leq t\leq T}S_t\geq x
\right)
\leq
\exp\left\{
-\frac{x^2}{2V_T}
\right\}.
\]
Since $V_T=\sum_{s=1}^T\sigma_s^2$, this proves the claim.
If $V_T=0$, then $\sigma_s^2=0$ for every $s\leq T$. The
moment-generating-function assumption implies
$
\E[e^{\lambda D_s}\mid\mathcal H_{s-1}]\leq1
$
for every $\lambda\in\R$, which forces $D_s=0$ almost surely.
Hence the conclusion is trivial.
\end{proof}

\begin{lemma}[Weighted chi-square concentration]
\label{lem:weighted_chisquare}
Let $Z_1,\ldots,Z_m$ be independent $N(0,I_d)$ random vectors, and
let $a_1,\ldots,a_m\geq0$ be deterministic. Then, for every $x>0$,
\[
\Prob\left(
\sum_{s=1}^m
a_s\{\|Z_s\|_2^2-d\}
\geq
2\sqrt{
d x\sum_{s=1}^m a_s^2
}
+
2x\max_{1\leq s\leq m}a_s
\right)
\leq
e^{-x}.
\]
\end{lemma}

\begin{proof}
Write
$
Z_s=(Z_{s1},\ldots,Z_{sd})^\top.
$
Then
\[
\sum_{s=1}^m
a_s\{\|Z_s\|_2^2-d\}
=
\sum_{s=1}^m\sum_{j=1}^d
a_s(Z_{sj}^2-1).
\]
The random variables $\{Z_{sj}\}$ are independent standard Gaussian.
Applying the weighted chi-square inequality of
\cite[Lemma~1]{laurent2000adaptive} to the $md$ variables with
coefficient $a_s$ repeated $d$ times gives
\[
\Prob\left(
\sum_{s=1}^m\sum_{j=1}^d
a_s(Z_{sj}^2-1)
\geq
2\sqrt{
x\sum_{s=1}^m\sum_{j=1}^d a_s^2
}
+
2x\max_s a_s
\right)
\leq e^{-x}.
\]
Since
\[
\sum_{s=1}^m\sum_{j=1}^d a_s^2
=
d\sum_{s=1}^m a_s^2,
\]
the claimed bound follows.
\end{proof}

\begin{proof}[Proof of Lemma~\ref{lem:stopped-stochastic-control}]
We control the martingale term $M_t$ and the quadratic term $Q_t$
separately.

\paragraph{Control of $M_t$.}
Recall from \eqref{eq:localization-martingale} that
$
M_t
=
2\sum_{s=n_0+1}^{t}
\eta_s
\left\langle
v_{s-1}-v^\star,\xi_s
\right\rangle
\1\{s\leq\tau\}.
$
Define the increments
\[
\Delta M_s
=
M_s-M_{s-1}
=
2\eta_s
\left\langle
v_{s-1}-v^\star,\xi_s
\right\rangle
\1\{s\leq\tau\}.
\]
We first verify that $\{M_t\}_{t\geq n_0}$ is a martingale with
respect to $\{\mathcal H_t\}$. Since
$
\E[\xi_s\mid\mathcal H_{s-1}]=0
$
and $\{s\leq\tau\}$ is $\mathcal H_{s-1}$-measurable, we have
\begin{align*}
\E[\Delta M_s\mid\mathcal H_{s-1}]
&=
2\eta_s\1\{s\leq\tau\}
\E\left[
\left\langle
v_{s-1}-v^\star,\xi_s
\right\rangle
\,\middle|\,
\mathcal H_{s-1}
\right]\\
&=
2\eta_s\1\{s\leq\tau\}
\left\langle
v_{s-1}-v^\star,
\E[\xi_s\mid\mathcal H_{s-1}]
\right\rangle\\
&=0.
\end{align*}
Thus $\{\Delta M_s\}$ is a martingale-difference sequence, and hence
$\{M_t\}_{t\geq n_0}$ is a martingale. To control its increments, define
$
U_s
=
\left\langle
v_{s-1}-v^\star,G_s
\right\rangle.
$
Since
$
\E[G_s\mid\mathcal H_{s-1}]
=
h(v_{s-1}),
$
we have
\[
U_s-\E[U_s\mid\mathcal H_{s-1}]
=
\left\langle
v_{s-1}-v^\star,\xi_s
\right\rangle.
\]
Therefore,
\[
\Delta M_s
=
2\eta_s
\left\{
U_s-\E[U_s\mid\mathcal H_{s-1}]
\right\}
\1\{s\leq\tau\}.
\]
Moreover, by the definition of $G_s$,
\begin{align*}
U_s
&=
\left\langle
v_{s-1}-v^\star,
P_{v_{s-1}}^\perp W_s
\right\rangle
=
\left\langle
P_{v_{s-1}}^\perp(v_{s-1}-v^\star),
W_s
\right\rangle.
\end{align*}
Using $W_s=A_sY_sX_s$, $A_s\in\{-1,+1\}$, and
$|Y_s|\leq B_Y$,
\begin{equation}
|U_s|
\leq
B_Y
\left|
\left\langle
P_{v_{s-1}}^\perp(v_{s-1}-v^\star),
X_s
\right\rangle
\right|.
\label{eq:Us-pointwise-bound}
\end{equation}
Conditional on $\mathcal H_{s-1}$,
$
P_{v_{s-1}}^\perp(v_{s-1}-v^\star)
$
is fixed, while $X_s\sim N(0,I_d)$ is independent of
$\mathcal H_{s-1}$. Hence
\[
\left.
\left\langle
P_{v_{s-1}}^\perp(v_{s-1}-v^\star),
X_s
\right\rangle
\,\right|\,
\mathcal H_{s-1}
\sim
N\left(
0,
\left\|
P_{v_{s-1}}^\perp(v_{s-1}-v^\star)
\right\|_2^2
\right).
\]
Consequently, \eqref{eq:Us-pointwise-bound} and the standard Gaussian
tail bound imply that, for every $x>0$,
\[
\Prob\left(
|U_s|\geq x
\,\middle|\,
\mathcal H_{s-1}
\right)
\leq
2\exp\left\{
-\frac{x^2}
{2B_Y^2
\|P_{v_{s-1}}^\perp(v_{s-1}-v^\star)\|_2^2}
\right\}.
\]
Applying Lemma~\ref{lem:tail_to_mgf} conditionally on
$\mathcal H_{s-1}$ with
$
\sigma_s
=
B_Y
\left\|
P_{v_{s-1}}^\perp(v_{s-1}-v^\star)
\right\|_2
$
gives
\[
\E\left[
\exp\left\{
\lambda
\left(
U_s-\E[U_s\mid\mathcal H_{s-1}]
\right)
\right\}
\,\middle|\,
\mathcal H_{s-1}
\right]
\leq
\exp\left(
C_{\rm sg}\lambda^2\sigma_s^2
\right).
\]
Since $\1\{s\leq\tau\}$ is $\mathcal H_{s-1}$-measurable, replacing
$\lambda$ above by
$
2\lambda\eta_s\1\{s\leq\tau\}
$
and using
$
\|
P_{v_{s-1}}^\perp(v_{s-1}-v^\star)
\|_2
\leq
\|v_{s-1}-v^\star\|_2
\leq r_0
\, \text{on }\, \{s\leq\tau\},
$
gives
\begin{align*}
\E\left[
e^{\lambda\Delta M_s}
\,\middle|\,
\mathcal H_{s-1}
\right]
&\leq
\exp\left\{
4C_{\rm sg}\lambda^2
B_Y^2\eta_s^2
\left\|
P_{v_{s-1}}^\perp(v_{s-1}-v^\star)
\right\|_2^2
\1\{s\leq\tau\}
\right\}
\leq
\exp\left\{
4C_{\rm sg}\lambda^2
B_Y^2r_0^2\eta_s^2
\right\}.
\end{align*}
Thus $M_t$ is a martingale with conditionally sub-Gaussian increments.
Lemma~\ref{lem:maximal-subgaussian-martingale} therefore gives, for
every $x>0$,
\[
\Prob\left(
\max_{n_0\leq t\leq n}M_t\geq x
\right)
\leq
\exp\left\{
-\frac{x^2}
{16C_{\rm sg}B_Y^2r_0^2
\sum_{s=n_0+1}^n\eta_s^2}
\right\}.
\]
For
$
\eta_s=\gamma/(s+t_0),
$
we have
\begin{align*}
\sum_{s=n_0+1}^n\eta_s^2
&\leq
\gamma^2
\sum_{s=n_0+1}^{\infty}
\frac{1}{(s+t_0)^2}
\leq
\gamma^2
\int_{n_0+t_0}^{\infty}\frac{du}{u^2}
=
\frac{\gamma^2}{n_0+t_0}.
\end{align*}
Therefore,
\[
\Prob\left(
\max_{n_0\leq t\leq n}M_t\geq x
\right)
\leq
\exp\left\{
-\frac{(n_0+t_0)x^2}
{16C_{\rm sg}B_Y^2r_0^2\gamma^2}
\right\}.
\]
Taking $x=r_0^2/4$ gives
\[
\Prob\left(
\max_{n_0\leq t\leq n}M_t
\geq\frac{r_0^2}{4}
\right)
\leq
\exp\left\{
-\frac{(n_0+t_0)r_0^2}
{256C_{\rm sg}B_Y^2\gamma^2}
\right\}.
\]
Hence
$
n_0+t_0
\geq
256C_{\rm sg}
B_Y^2\gamma^2r_0^{-2}
\log\left(\frac{3}{\delta}\right)
$
is sufficient to ensure
\[
\Prob\left(
\max_{n_0\leq t\leq n}M_t
>\frac{r_0^2}{4}
\right)
\leq
\frac{\delta}{3}.
\]
Thus the first claim holds with, for example,
$
C_M=256C_{\rm sg}.
$

\paragraph{Control of $Q_t$.}
Recall from \eqref{eq:localization-quadratic} that
$
Q_t
=
\sum_{s=n_0+1}^t
\eta_s^2
\|G_s\|_2^2
\1\{s\leq\tau\}.
$
Since every summand is nonnegative, $Q_t$ is nondecreasing in $t$,
and therefore
$
\max_{n_0\leq t\leq n}Q_t=Q_n.
$
Moreover,
\[
\|G_s\|_2
=
\|P_{v_{s-1}}^\perp W_s\|_2
\leq
\|W_s\|_2
=
|Y_s|\|X_s\|_2
\leq
B_Y\|X_s\|_2.
\]
Hence, pathwise,
\begin{equation}
Q_n
\leq
B_Y^2
\sum_{s=n_0+1}^n
\eta_s^2\|X_s\|_2^2.
\label{eq:QT-gaussian-bound}
\end{equation}
Notice that the stopping indicator has disappeared from the
right-hand side. This allows us to use directly the independence of
the Gaussian covariates. Applying
Lemma~\ref{lem:weighted_chisquare} with
$a_s=\eta_s^2$ gives, for every $x>0$, with probability at least
$1-e^{-x}$,
\begin{align}
\sum_{s=n_0+1}^n
\eta_s^2\|X_s\|_2^2
\leq{}&
d\sum_{s=n_0+1}^n\eta_s^2
+
2\sqrt{
xd\sum_{s=n_0+1}^n\eta_s^4
}
+
2x
\max_{n_0+1\leq s\leq n}\eta_s^2.
\label{eq:weighted-chi-Q}
\end{align}
For the harmonic step size,
\[
\sum_{s=n_0+1}^n\eta_s^2
\leq
\frac{\gamma^2}{n_0+t_0}.
\]
Similarly,
\begin{align*}
\sum_{s=n_0+1}^n\eta_s^4
&\leq
\gamma^4
\sum_{s=n_0+1}^{\infty}
\frac{1}{(s+t_0)^4}
\leq
\gamma^4
\int_{n_0+t_0}^{\infty}\frac{du}{u^4}
=
\frac{\gamma^4}{3(n_0+t_0)^3}
\leq
\frac{\gamma^4}{(n_0+t_0)^3},
\end{align*}
and
\[
\max_{n_0+1\leq s\leq n}\eta_s^2
\leq
\frac{\gamma^2}{(n_0+t_0)^2}.
\]
Substituting these bounds into \eqref{eq:weighted-chi-Q}, we obtain,
with probability at least $1-e^{-x}$,
\[
Q_n
\leq
B_Y^2\gamma^2
\left\{
\frac{d}{n_0+t_0}
+
\frac{2\sqrt{xd}}{(n_0+t_0)^{3/2}}
+
\frac{2x}{(n_0+t_0)^2}
\right\}.
\]
Since $n_0+t_0\geq1$,
\[
\frac{2\sqrt{xd}}{(n_0+t_0)^{3/2}}
\leq
\frac{2\sqrt{xd}}{n_0+t_0}
\leq
\frac{d+x}{n_0+t_0},
\]
where the last inequality follows from
$
2\sqrt{xd}\leq d+x.
$
Also,
\[
\frac{2x}{(n_0+t_0)^2}
\leq
\frac{2x}{n_0+t_0}.
\]
Consequently,
\[
Q_n
\leq
\frac{
B_Y^2\gamma^2
}{n_0+t_0}
\{2d+3x\}
\leq
\frac{
3B_Y^2\gamma^2
}{n_0+t_0}
(d+x)
\]
with probability at least $1-e^{-x}$. Taking
$
x=\log({3}/{\delta})
$
therefore gives, with probability at least $1-\delta/3$,
\[
Q_n
\leq
\frac{
3B_Y^2\gamma^2
}{
n_0+t_0
}
\left\{
d+\log\left(\frac{3}{\delta}\right)
\right\}.
\]
Thus, if
$
n_0+t_0
\geq
12B_Y^2\gamma^2r_0^{-2}
\{
d+\log\left({3}/{\delta}\right)
\},
$
then
$
Q_n\leq r_0^2/4
$
with probability at least $1-\delta/3$. Equivalently,
\[
\Prob\left(
Q_n>\frac{r_0^2}{4}
\right)
\leq
\frac{\delta}{3}.
\]
Hence the second claim holds with, for example,
$
C_Q=12.
$
Combining the two preceding arguments proves
Lemma~\ref{lem:stopped-stochastic-control}.
\end{proof}
\begin{proof}[Proof of Lemma~\ref{lem:finite-horizon-localization}]
Define the events
\[
\mathcal E_{\rm init}
=
\left\{
e_{n_0}\leq\frac{r_0^2}{4}
\right\},
\qquad
\mathcal E_M
=
\left\{
\max_{n_0\leq t\leq n}M_t
\leq\frac{r_0^2}{4}
\right\},
\qquad
\mathcal E_Q
=
\left\{
Q_n\leq\frac{r_0^2}{4}
\right\}.
\]
We first show that each of these events occurs with high probability. By Proposition~\ref{prop:randomized_initialization}, applied with
$r=r_0/2$ and $\delta_0=\delta/3$, there exists a universal constant
$C_{\rm init}>0$ such that
\[
n_0
\geq
\frac{C_{\rm init}B_Y^2}
{\Delta_\beta^2\mu_g^2r_0^2}
\left(
d+\log\frac{3}{\delta}
\right)
\]
implies
\[
\Prob\left(
\|v_{n_0}-v^\star\|_2
\leq\frac{r_0}{2}
\right)
\geq
1-\frac{\delta}{3},
\]
where $C_{\rm init}$ absorbs the numerical constants arising from
$r=r_0/2$, from
$c=\Delta_\beta/2$, and from replacing
$\log(2/\delta_0)=\log(6/\delta)$ by a constant multiple of
$\log(3/\delta)$. Therefore,
\begin{equation}
\Prob(\mathcal E_{\rm init}^c)
\leq\frac{\delta}{3}.
\label{eq:init-bound-localization}
\end{equation}
Next, let
$
C_{\rm quad}
=
\max\{C_M,C_Q\}.
$
Since $t_0\geq0$,
$
n_0+t_0\geq n_0,
$
and since
$
d+\log(3/\delta)\geq\log(3/\delta),
$
Lemma~\ref{lem:stopped-stochastic-control} implies that
\[
n_0
\geq
C_{\rm quad}B_Y^2\gamma^2r_0^{-2}
\left\{
d+\log\frac{3}{\delta}
\right\}
\]
is sufficient for
\begin{equation}
\Prob(\mathcal E_M^c)
\leq
\frac{\delta}{3},
\qquad
\Prob(\mathcal E_Q^c)
\leq
\frac{\delta}{3}.
\label{eq:stochastic-good-events}
\end{equation}

Consequently, there exists a problem-dependent constant
$C_{\rm loc}<\infty$, independent of $d$, $n$, and $\delta$, such
that
$
n_0
\geq
C_{\rm loc}
\left\{
d+\log(1/\delta)
\right\}
$
implies both \eqref{eq:init-bound-localization} and
\eqref{eq:stochastic-good-events}. For example, $C_{\rm loc}$ may
be chosen as a sufficiently large universal multiple of
$
({B_Y^2}/{r_0^2})
\max\{
({1}/{\Delta_\beta^2\mu_g^2}),
\gamma^2
\}.
$
We now show that
\[
\mathcal E_{\rm init}
\cap
\mathcal E_M
\cap
\mathcal E_Q
\subseteq
\{\tau>n\}.
\]
From the stopped recursion \eqref{eq:stopped-error-bound},
\[
e_{t\wedge\tau}
\leq
e_{n_0}+M_t+Q_t,
\qquad
n_0\leq t\leq n.
\]
Since $Q_t\leq Q_n$, on
$\mathcal E_{\rm init}\cap\mathcal E_M\cap\mathcal E_Q$ we have
\[
e_{t\wedge\tau}
\leq
\frac{r_0^2}{4}
+
\frac{r_0^2}{4}
+
\frac{r_0^2}{4}
=
\frac{3r_0^2}{4},
\qquad
n_0\leq t\leq n.
\]
Suppose, to the contrary, that $\tau\leq n$. Taking $t=\tau$ gives
$
e_\tau
\leq
{3r_0^2}/{4}
<
r_0^2.
$
However, by the definition
$
\tau
=
\inf\left\{
t\geq n_0:
\|v_t-v^\star\|_2>r_0
\right\},
$
the event $\{\tau<\infty\}$ implies
$e_\tau
=
\|v_\tau-v^\star\|_2^2
>
r_0^2,
$
which is a contradiction. Therefore,
$\tau>n$ on
$\mathcal E_{\rm init}\cap\mathcal E_M\cap\mathcal E_Q$.
Finally, by the union bound and
\eqref{eq:init-bound-localization}--\eqref{eq:stochastic-good-events},
\begin{align*}
\Prob(\tau\leq n)
&\leq
\Prob(\mathcal E_{\rm init}^c)
+
\Prob(\mathcal E_M^c)
+
\Prob(\mathcal E_Q^c)
\leq
\delta.
\end{align*}
This proves the claim.
\end{proof}
\subsection{Localized estimation rate}
\label{app:localized-estimation}

We first record a deterministic recursion lemma used to convert the
local one-step contraction into a $t^{-1}$ estimation rate.

\begin{lemma}[Deterministic stochastic approximation recursion]
\label{lem:sa-recursion}
Let $\alpha>1$ and $b>0$, and suppose that a nonnegative sequence
$\{a_t\}_{t\geq n_0}$ satisfies
\[
a_t
\leq
\left(
1-\frac{\alpha}{t+t_0}
\right)a_{t-1}
+
\frac{b}{(t+t_0)^2},
\qquad t>n_0.
\]
If $n_0+1+t_0\geq\alpha$, then
\[
a_t
\leq
\frac{C}{t+t_0},
\qquad t\geq n_0,
\]
where
\[
C
=
\max\left\{
(n_0+t_0)a_{n_0},
\frac{b}{\alpha-1}
\right\}.
\]
\end{lemma}

\begin{proof}
We proceed by induction. By the definition of $C$,
\[
a_{n_0}
\leq
\frac{C}{n_0+t_0}.
\]
Suppose that, for some $t>n_0$,
\[
a_{t-1}
\leq
\frac{C}{t-1+t_0}.
\]
Since $n_0+1+t_0\geq\alpha$, the coefficient
$1-\alpha/(t+t_0)$ is nonnegative. Therefore,
\[
a_t
\leq
\left(
1-\frac{\alpha}{t+t_0}
\right)
\frac{C}{t-1+t_0}
+
\frac{b}{(t+t_0)^2}.
\]
It remains to show that the right-hand side is at most
$C/(t+t_0)$. This is equivalent to
\[
\frac{b}{t+t_0}
\leq
C
\left\{
1-
\frac{t+t_0-\alpha}{t-1+t_0}
\right\}
=
C\frac{\alpha-1}{t-1+t_0}.
\]
By the definition of $C$,
$C(\alpha-1)\geq b$, and hence
\[
C\frac{\alpha-1}{t-1+t_0}
\geq
\frac{b}{t-1+t_0}
\geq
\frac{b}{t+t_0}.
\]
Thus $a_t\leq C/(t+t_0)$, completing the induction.
\end{proof}
Whenever $\|v_{t-1}-v^\star\|_2\leq r_0$, the local one-step
contraction in Lemma~\ref{lem:local_one_step} gives
\[
\E[e_t\mid\mathcal H_{t-1}]
\leq
(1-2\lambda_0\eta_t)e_{t-1}
+
B_Y^2d\,\eta_t^2.
\]
For $\eta_t=\gamma/(t+t_0)$, this becomes
\begin{equation}
\E[e_t\mid\mathcal H_{t-1}]
\leq
\left(
1-\frac{2\gamma\lambda_0}{t+t_0}
\right)e_{t-1}
+
\frac{B_Y^2\gamma^2d}{(t+t_0)^2}.
\label{eq:local-recursion-regret}
\end{equation}

The following result uses this recursion together with the randomized
initialization bound to obtain the localized estimation rate needed
for the regret analysis.

\begin{proof}[Proof of Proposition~\ref{prop:localized-estimation}]
Let
$
a_t
=
\E\left[
e_t\1\{\tau>t\}
\right].
$
For $t>n_0$, since
$
\{\tau>t\}\subseteq\{\tau\geq t\},
$
we have
$
a_t
\leq
\E\left[
e_t\1\{\tau\geq t\}
\right].
$
The event $\{\tau\geq t\}=\{t\leq\tau\}$ is
$\mathcal H_{t-1}$-measurable, and on this event
$\|v_{t-1}-v^\star\|_2\leq r_0$. Therefore,
using \eqref{eq:local-recursion-regret} and the tower property,
\begin{align*}
a_t
&\leq
\E\left[
\1\{\tau\geq t\}
\E[e_t\mid\mathcal H_{t-1}]
\right]\\
&\leq
\left(
1-\frac{2\gamma\lambda_0}{t+t_0}
\right)
\E\left[
e_{t-1}\1\{\tau\geq t\}
\right]
+
\frac{B_Y^2\gamma^2d}{(t+t_0)^2}
\Prob(\tau\geq t).
\end{align*}
Since
$
\{\tau\geq t\}=\{\tau>t-1\}
$
and $\Prob(\tau\geq t)\leq1$, it follows that
\[
a_t
\leq
\left(
1-\frac{2\gamma\lambda_0}{t+t_0}
\right)a_{t-1}
+
\frac{B_Y^2\gamma^2d}{(t+t_0)^2}.
\]
Applying Lemma~\ref{lem:sa-recursion} with
$
\alpha=2\gamma\lambda_0,
b=B_Y^2\gamma^2d,
$
gives
$
a_t
\leq
{C}/{(t+t_0)},
$
where
\[
C
=
\max\left\{
(n_0+t_0)a_{n_0},
\frac{B_Y^2\gamma^2d}
{2\gamma\lambda_0-1}
\right\}.
\]
It remains to control the first term in $C$. Since
$
a_{n_0}\leq\E[e_{n_0}],
$
the expectation bound in
Proposition~\ref{prop:randomized_initialization} gives
\[
a_{n_0}
\leq
C_{\rm init}'
\frac{B_Y^2d}
{\Delta_\beta^2\mu_g^2n_0}
\]
for a universal constant $C_{\rm init}'>0$. Since
$t_0\leq C_{\rm off}n_0$,
\begin{align*}
(n_0+t_0)a_{n_0}
\leq
C_{\rm init}'
\frac{B_Y^2d}{\Delta_\beta^2\mu_g^2}
\frac{n_0+t_0}{n_0}
\leq
C_{\rm init}'(1+C_{\rm off})
\frac{B_Y^2}{\Delta_\beta^2\mu_g^2}\,d.
\end{align*}
Consequently,
$
C
\leq
C_e d,
$
where we may take
\begin{equation}
C_e
=
\max\left\{
C_{\rm init}'(1+C_{\rm off})
\frac{B_Y^2}{\Delta_\beta^2\mu_g^2},
\frac{B_Y^2\gamma^2}
{2\gamma\lambda_0-1}
\right\}.
\label{eq:Ce-explicit}
\end{equation}
Therefore,
\[
\E\left[
\|v_t-v^\star\|_2^2\1\{\tau>t\}
\right]
\leq
C_e\frac{d}{t+t_0},
\qquad t\geq n_0.
\]
Finally, for $t>n_0$,
$
\{\tau\geq t\}=\{\tau>t-1\},
$
so applying the preceding bound at time $t-1$ gives
\[
\E\left[
\|v_{t-1}-v^\star\|_2^2\1\{\tau\geq t\}
\right]
\leq
C_e\frac{d}{t-1+t_0}.
\]
This proves the claim.
\end{proof}
\subsection{From boundary error to regret}
\label{app:error-to-regret}

We next show that, under Gaussian contexts, disagreement between two
nearby decision boundaries has quadratically small probability mass
after weighting by the distance to the optimal boundary.

\begin{lemma}[Gaussian polynomially weighted disagreement]
\label{lem:gaussian-polynomial-disagreement}
Let $X\sim N(0,I_d)$ and let $v,v^\star\in\mathbb S^{d-1}$.
Suppose that $\theta^\top v^\star=0$ and $\|\theta\|_2\leq1$.
For every fixed $q\geq0$, there exists a constant $C_q<\infty$
such that
\begin{align*}
&\E\left[
|(v^\star)^\top X|
\left\{
1+|\theta^\top X|^q
+|(v^\star)^\top X|^q
\right\}
\1\left\{
\operatorname{sgn}(v^\top X)
\neq
\operatorname{sgn}((v^\star)^\top X)
\right\}
\right]
\leq
C_q\|v-v^\star\|_2^2.
\end{align*}
\end{lemma}
\begin{proof}[Proof of Lemma~\ref{lem:gaussian-polynomial-disagreement}]
Throughout the proof, $C_q$ denotes a finite positive constant
depending only on $q$, whose value may change from line to line.
Let $\alpha\in[0,\pi]$ denote the angle between $v$ and $v^\star$.
By rotational invariance, choose an orthonormal basis such that
$v^\star=e_1$ and
$v=\cos(\alpha)e_1+\sin(\alpha)e_2$.
Since $(X_1,X_2)\sim N(0,I_2)$, write
$X_1=R\cos\Phi$ and $X_2=R\sin\Phi$, where
$\Phi\sim\operatorname{Unif}[0,2\pi)$, $R\perp\Phi$, and
$R^2\sim\chi_2^2$. Hence
$(v^\star)^\top X=R\cos\Phi$ and
$v^\top X=R\cos(\Phi-\alpha)$.

Since $R>0$ almost surely, the disagreement event is
\[
\mathcal D_v
=
\left\{
\operatorname{sgn}(v^\top X)
\neq
\operatorname{sgn}((v^\star)^\top X)
\right\}
=
\left\{
\operatorname{sgn}(\cos\Phi)
\neq
\operatorname{sgn}(\cos(\Phi-\alpha))
\right\}.
\]
Decompose $\theta=\theta_\parallel+\theta_\perp$, where
$\theta_\parallel\in\operatorname{span}(e_1,e_2)$ and
$\theta_\perp$ is orthogonal to this plane, and define
$U=\theta_\perp^\top X$. Since $\|\theta\|_2\leq1$,
$|\theta_\parallel^\top X|\leq R$, and hence
$|\theta^\top X|^q\leq C_q\{R^q+|U|^q\}$. Also,
$|(v^\star)^\top X|^q\leq R^q$. Therefore,
\begin{align*}
&|(v^\star)^\top X|
\left\{
1+|\theta^\top X|^q
+|(v^\star)^\top X|^q
\right\}
\1\{\mathcal D_v\} \leq
C_q R|\cos\Phi|
\left\{
1+R^q+|U|^q
\right\}
\1\{\mathcal D_v\}.
\end{align*}

Now $U$ is independent of $(R,\Phi)$, while $R$ is independent of
$\Phi$. Since $\mathcal D_v$ depends only on $\Phi$,
\begin{align*}
&\E\left[
|(v^\star)^\top X|
\left\{
1+|\theta^\top X|^q
+|(v^\star)^\top X|^q
\right\}
\1\{\mathcal D_v\}
\right]\leq
C_q
\E\left[
R\{1+R^q+|U|^q\}
\right]
\E\left[
|\cos\Phi|\1\{\mathcal D_v\}
\right].
\end{align*}
Since $R$ is the norm of a two-dimensional standard Gaussian vector
and $U\sim N(0,\|\theta_\perp\|_2^2)$ with
$\|\theta_\perp\|_2\leq1$, all of the preceding moments are bounded
by constants depending only on $q$. Hence
$\E[R\{1+R^q+|U|^q\}]\leq C_q$.

It remains to control the angular factor. For
$\alpha\in[0,\pi]$, the signs of $\cos\phi$ and
$\cos(\phi-\alpha)$ differ precisely on two angular wedges of length
$\alpha$, which may be represented modulo $2\pi$ by
$(\pi/2,\pi/2+\alpha)$ and
$(3\pi/2,3\pi/2+\alpha)$. Therefore,
\[
\int_{\mathcal D_\alpha}|\cos\phi|\,d\phi
=
2\int_{\pi/2}^{\pi/2+\alpha}|\cos\phi|\,d\phi
=
2(1-\cos\alpha).
\]
Consequently,
$\E[|\cos\Phi|\1\{\mathcal D_v\}]
=(1-\cos\alpha)/\pi$. Thus
\[
\E\left[
|(v^\star)^\top X|
\left\{
1+|\theta^\top X|^q
+|(v^\star)^\top X|^q
\right\}
\1\{\mathcal D_v\}
\right]
\leq
C_q(1-\cos\alpha).
\]
Finally,
$\|v-v^\star\|_2^2=2(1-\cos\alpha)$, and therefore
\[
\E\left[
|(v^\star)^\top X|
\left\{
1+|\theta^\top X|^q
+|(v^\star)^\top X|^q
\right\}
\1\{\mathcal D_v\}
\right]
\leq
C_q\|v-v^\star\|_2^2,
\]
after absorbing numerical constants into $C_q$.
\end{proof}
Using
$\beta_+=\theta+(\Delta_\beta/2)v^\star$ and
$\beta_-=\theta-(\Delta_\beta/2)v^\star$, the mean-value theorem gives
\[
|m_+(X)-m_-(X)|
\leq
\Delta_\beta |(v^\star)^\top X|
\sup_{u\in I_X}|g'(u)|,
\]
where $I_X$ is the interval between $\beta_-^\top X$ and
$\beta_+^\top X$. Since $\Delta_\beta\leq2$, every $u\in I_X$
satisfies
$|u|\leq|\theta^\top X|+|(v^\star)^\top X|$.
Assumption~\ref{assum:polygrowth_g} therefore implies
\[
|m_+(X)-m_-(X)|
\leq
C\Delta_\beta |(v^\star)^\top X|
\left\{
1+|\theta^\top X|^q+|(v^\star)^\top X|^q
\right\},
\]
where $C<\infty$ depends only on the polynomial-growth constants for
$g'$ and on $q$.

Returning to the sequential problem, conditional on
$\mathcal H_{t-1}$, the direction $v_{t-1}$ is fixed, while
$X_t\sim N(0,I_d)$ is independent of $\mathcal H_{t-1}$.
Applying Lemma~\ref{lem:gaussian-polynomial-disagreement} with
$v=v_{t-1}$ therefore gives
\[
\E[r_t\mid\mathcal H_{t-1}]
\leq
C_{\mathrm{gap}}\frac{\Delta_\beta}{2}
\|v_{t-1}-v^\star\|_2^2,
\]
where the numerical factor $2$ is absorbed into
$C_{\mathrm{gap}}<\infty$. The constant $C_{\mathrm{gap}}$ depends
only on the polynomial-growth constants for $g'$ and on $q$, and is
independent of $\Delta_\beta$, $d$, and $t$.

Combining this instantaneous regret bound with the localized
estimation rate and the finite-horizon localization result yields the
following cumulative regret guarantee.

\begin{theorem}[Regret of NBL]
Suppose that Assumptions~\ref{assum:gaussian_context},
\ref{assum:polygrowth_g}, \ref{assum:sequential_rewards},
\ref{assum:bounded_rewards}, and
\ref{assum:positive-stability} hold. Let $r_0>0$ and $\lambda_0>0$
be the local stability constants from
Corollary~\ref{cor:local_stability}, so that
$
\langle v-v^\star,h(v)\rangle
\leq
-\lambda_0\|v-v^\star\|_2^2
$
whenever $\|v-v^\star\|_2\leq r_0$. Let
$\eta_t=\gamma/(t+t_0)$, where $2\gamma\lambda_0>1$.
Fix a horizon $n\geq2$. There exists a sufficiently large
problem-dependent constant $C_0<\infty$, independent of $d$ and $n$,
such that, with
$
n_0=\lceil C_0(d+\log n)\rceil,
$
the following holds. Suppose that $t_0\leq C_{\rm off}n_0$ for some fixed
$C_{\rm off}<\infty$. If $n_0<n$, then
\[
\E[R_n]
\leq
2B_Yn_0
+
C_{\mathrm{gap}}\frac{\Delta_\beta}{2}C_e d
\left\{
1+
\log\left(
\frac{n+t_0}{n_0+t_0}
\right)
\right\}
+
2B_Yn^{-3},
\]
where $C_e$ is the constant in
Proposition~\ref{prop:localized-estimation}. Consequently,
\[
\E[R_n]=O(d\log n),
\]
where the implicit constant is independent of $d$ and $n$.
\end{theorem}

\begin{proof}[Proof of Theorem~\ref{thm:regret-rate}]
Set $\delta=n^{-4}$. We first verify that
$n_0=\lceil C_0(d+\log n)\rceil$ can be chosen to satisfy all
lower-bound requirements. By
Lemma~\ref{lem:finite-horizon-localization}, it is sufficient that
\[
n_0\geq C_{\rm loc}\{d+\log(1/\delta)\}.
\]
Since $\delta=n^{-4}$,
$d+\log(1/\delta)=d+4\log n\leq4(d+\log n)$, so the localization
requirement is satisfied whenever $C_0\geq4C_{\rm loc}$.
Proposition~\ref{prop:localized-estimation} also requires
$n_0+1+t_0\geq2\gamma\lambda_0$. Since $d+\log n$ is bounded away
from zero for $d\geq1$ and $n\geq2$, this condition is also satisfied
by enlarging $C_0$, if necessary. Thus $C_0$ may be chosen as a
finite problem-dependent constant, independent of $d$ and $n$, such
that all lower-bound requirements on $n_0$ hold.

For $t>n_0$, decompose
\[
\E[r_t]
=
\E[r_t\1\{\tau\geq t\}]
+
\E[r_t\1\{\tau<t\}].
\]
Since $\{\tau\geq t\}=\{\tau>t-1\}$ is
$\mathcal H_{t-1}$-measurable, the preceding bound gives
\begin{align*}
\E[r_t\1\{\tau\geq t\}]
&=
\E\left[
\1\{\tau\geq t\}
\E[r_t\mid\mathcal H_{t-1}]
\right]\\
&\leq
C_{\rm gap}\frac{\Delta_\beta}{2}
\E\left[
\|v_{t-1}-v^\star\|_2^2
\1\{\tau\geq t\}
\right]\\
&\leq
C_{\rm gap}\frac{\Delta_\beta}{2}C_e
\frac{d}{t-1+t_0},
\end{align*}
where the last inequality follows from
Proposition~\ref{prop:localized-estimation}.

On the complementary event, bounded rewards imply $r_t\leq2B_Y$.
Since $\{\tau<t\}\subseteq\{\tau\leq n\}$ for $t\leq n$,
Lemma~\ref{lem:finite-horizon-localization} gives
\[
\E[r_t\1\{\tau<t\}]
\leq
2B_Y\Prob(\tau\leq n)
\leq
2B_Yn^{-4}.
\]
Hence, for every $t>n_0$,
\[
\E[r_t]
\leq
C_{\rm gap}\frac{\Delta_\beta}{2}C_e
\frac{d}{t-1+t_0}
+
2B_Yn^{-4}.
\]
Summing over the boundary-learning stage,
\begin{align*}
\sum_{t=n_0+1}^n\E[r_t]
&\leq
C_{\rm gap}\frac{\Delta_\beta}{2}C_e d
\sum_{t=n_0+1}^n\frac{1}{t-1+t_0}
+
2B_Y(n-n_0)n^{-4}\\
&\leq
C_{\rm gap}\frac{\Delta_\beta}{2}C_e d
\left\{
1+
\log\left(
\frac{n+t_0}{n_0+t_0}
\right)
\right\}
+
2B_Yn^{-3}.
\end{align*}
During randomized initialization, $r_t\leq2B_Y$, so
$\sum_{t=1}^{n_0}\E[r_t]\leq2B_Yn_0$. Therefore,
\[
\E[R_n]
\leq
2B_Yn_0
+
C_{\rm gap}\frac{\Delta_\beta}{2}C_e d
\left\{
1+
\log\left(
\frac{n+t_0}{n_0+t_0}
\right)
\right\}
+
2B_Yn^{-3}.
\]

Finally, $n_0=O(d+\log n)$. Since $t_0\leq C_{\rm off}n_0$ and $n_0<n$,
$t_0\leq C_{\rm off}n$, and hence
\[
1+
\log\left(
\frac{n+t_0}{n_0+t_0}
\right)
\leq
1+\log(1+C_{\rm off})+\log n
=
O(\log n).
\]
Thus
$\E[R_n]=O(d+\log n)+O(d\log n)+O(n^{-3})$.
Since $d\geq1$ and $n\geq2$, $d+\log n=O(d\log n)$, and therefore
\[
\E[R_n]=O(d\log n).
\]
\end{proof}
\section{Numerical Experiments}
\label{app:numerics}
\subsection{Simulation Setup}
\label{app:simulation-setup}

Unless otherwise stated, all numerical experiments use $d=2$ with
isotropic Gaussian contexts
$
X_t\sim N(0,I_2).
$
We parameterize the two arm-specific index directions as
\[
\beta_a
=
\theta
+
a\frac{\Delta_\beta}{2}v^\star,
\qquad
a\in\{-1,+1\},
\]
where $\|v^\star\|_2=1$, $\theta^\top v^\star=0$, and
$\|\theta\|_2^2=1-\Delta_\beta^2/4$. Thus
$\|\beta_+-\beta_-\|_2=\Delta_\beta$, and the optimal action is
$
a^\star(x)=\operatorname{sgn}\{(v^\star)^\top x\}.
$
The conditional mean reward is
$
m_a(x)=g(\beta_a^\top x),
$
with the shared strictly increasing link $g$ varied across experiments.
For the population calculations, we use the decomposition
$
U=\theta^\top X
$
and
$
Z=(v^\star)^\top X,
$
so that
$
U\sim N(0,1-\Delta_\beta^2/4)
$
and $Z\sim N(0,1)$ are independent. The two link-dependent quantities
entering the decision stability coefficient are
\[
\mu^\star
=
\E\!\left[
g'\!\left(U+\frac{\Delta_\beta}{2}Z\right)
\mathbf 1\{Z>0\}
\right],
\qquad
m_2=\E\{g''(U)\},
\]
and we evaluate
\[
\lambda_\star
=
\Delta_\beta\mu^\star
-
2\phi_0(0)(m_2)_+
\left(1-\frac{\Delta_\beta^2}{4}\right).
\]
The expectations defining these population quantities are evaluated
numerically under the corresponding Gaussian distributions.

For the stochastic experiments, NBL is updated using
\[
A_t=\operatorname{sgn}(v_{t-1}^\top X_t),
\qquad
G_t=P_{v_{t-1}}^\perp(A_tY_tX_t),
\qquad
v_t=
\frac{v_{t-1}+\eta_tG_t}
{\|v_{t-1}+\eta_tG_t\|_2},
\]
with harmonic step size
$
\eta_t=\gamma/(t+t_0).
$
Unless otherwise stated, we use $\gamma=10$ and $t_0=500$. In the
experiments designed to examine the local stochastic dynamics, the
initial direction is displaced by $10^\circ$ from $v^\star$. We use
$n=10{,}000$ rounds and average the resulting trajectories over 100
independent Monte Carlo replications.
We report two measures of stochastic performance. The signed angular
error records the oriented displacement of $v_t$ from $v^\star$ and
therefore retains the direction of the local population dynamics. The
decision performance is summarized by average regret
$
R_t/t,
$
where
\[
R_t
=
\sum_{s=1}^t
\left\{
m_{a^\star(X_s)}(X_s)-m_{A_s}(X_s)
\right\}.
\]
For trajectory plots, solid curves show Monte Carlo means and shaded
regions show pointwise approximate 95\% Monte Carlo confidence intervals
for the mean, computed as the Monte Carlo mean plus or minus $1.96$
Monte Carlo standard errors.
The experiments comparing NBL with Greedy Logistic use a separate
common randomized initialization and a shorter horizon; their settings
and implementation are given in
Appendix~\ref{app:logistic-comparison}. Exact specifications of the
elicitation-induced and two-transition link families are given in
Appendices~\ref{app:elicitation-geometry} and
\ref{app:tanh-geometry}, respectively.

\subsection{Link Geometry and Decision Stability}
\label{app:link-geometry}
We use two complementary link constructions to examine the decision
stability condition. The first is an elicitation-induced family that
allows the local geometry of the link to be varied analytically. The
second is a flexible two-transition link used as an unknown
nonparametric regression function. Together, the examples illustrate
that decision stability depends jointly on the geometry of the shared
link and the separation between the arm-specific indices.

\subsubsection{A controlled elicitation-induced family}
\label{app:elicitation-geometry}

We begin with a weighted version of the binary negative-entropy
potential. For $\rho\in(-1,1)$, define
\[
\varphi_\rho(m)
=
(1+\rho)m\log m
+
(1-\rho)(1-m)\log(1-m)
-
\kappa_\rho m,
\qquad m\in(0,1),
\]
where $\kappa_\rho=2\rho(1-\log 2)$. The linear term is chosen so that
$\varphi_\rho'(1/2)=0$. Since affine terms do not affect the associated
Bregman divergence, this normalization changes the dual coordinate
without changing the underlying Bregman geometry.
Differentiating gives
$\varphi_\rho'(m)
=(1+\rho)\log m-(1-\rho)\log(1-m)+2\rho\log 2$
and
$\varphi_\rho''(m)
=(1+\rho)/m+(1-\rho)/(1-m)$.
Since $|\rho|<1$, $\varphi_\rho''(m)>0$ on $(0,1)$, so
$\varphi_\rho$ is strictly convex and $\varphi_\rho'$ is strictly
increasing. The induced response link $g_\rho=(\varphi_\rho')^{-1}$
is therefore characterized by
\[
z
=
(1+\rho)\log g_\rho(z)
-
(1-\rho)\log\{1-g_\rho(z)\}
+
2\rho\log 2,
\]
or, equivalently,
$g_\rho(z)^{1+\rho}/\{1-g_\rho(z)\}^{1-\rho}
=e^z/2^{2\rho}$.
Moreover, $\varphi_\rho'(m)\to-\infty$ as $m\downarrow0$ and
$\varphi_\rho'(m)\to+\infty$ as $m\uparrow1$, so $g_\rho$ is a
strictly increasing map from $\mathbb R$ onto $(0,1)$. When $\rho=0$,
the construction reduces exactly to the logistic link.

The normalization allows the local geometry of the family to be
compared directly. At $m=1/2$, we have
$\varphi_\rho''(1/2)=4$ and
$\varphi_\rho'''(1/2)=-8\rho$. Using
$g_\rho'=1/\varphi_\rho''(g_\rho)$ and
$g_\rho''=-\varphi_\rho'''(g_\rho)/
\{\varphi_\rho''(g_\rho)\}^3$ gives
\[
g_\rho(0)=\frac12,
\qquad
g_\rho'(0)=\frac14,
\qquad
g_\rho''(0)=\frac{\rho}{8}.
\]
Thus the family preserves the value and slope of the response link at
the origin while allowing its local curvature and asymmetry to vary
through $\rho$. The same potential defines a proper Bregman loss that
continues to elicit the ordinary Bernoulli mean, so varying $\rho$
changes the elicitation geometry without changing the target
functional. Figure~\ref{fig:app-elicitation-geometry} shows the corresponding
potential and dual map, while the induced links are shown in
Figure~\ref{fig:numerical-summary}(e) of the main text.

\begin{figure}[h!]
    \centering
    \begin{minipage}[t]{0.42\textwidth}
        \centering
        \includegraphics[width=\textwidth]
        {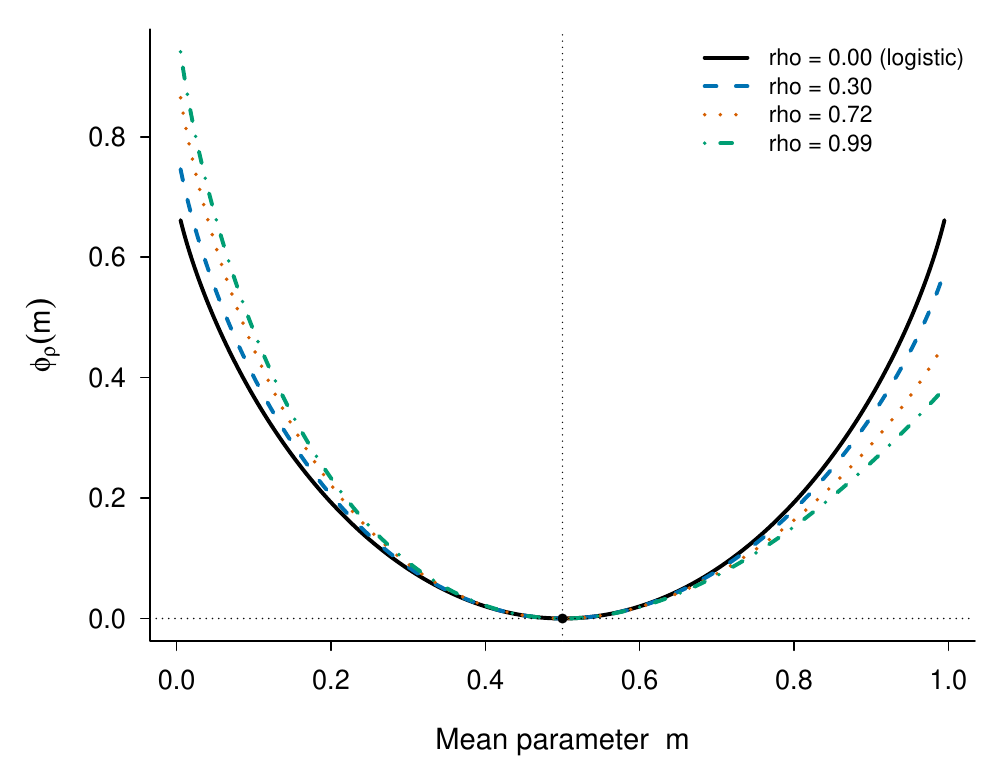}
        \small (a) Elicitation potential
    \end{minipage}
    \hfill
    \begin{minipage}[t]{0.42\textwidth}
        \centering
        \includegraphics[width=\textwidth]
        {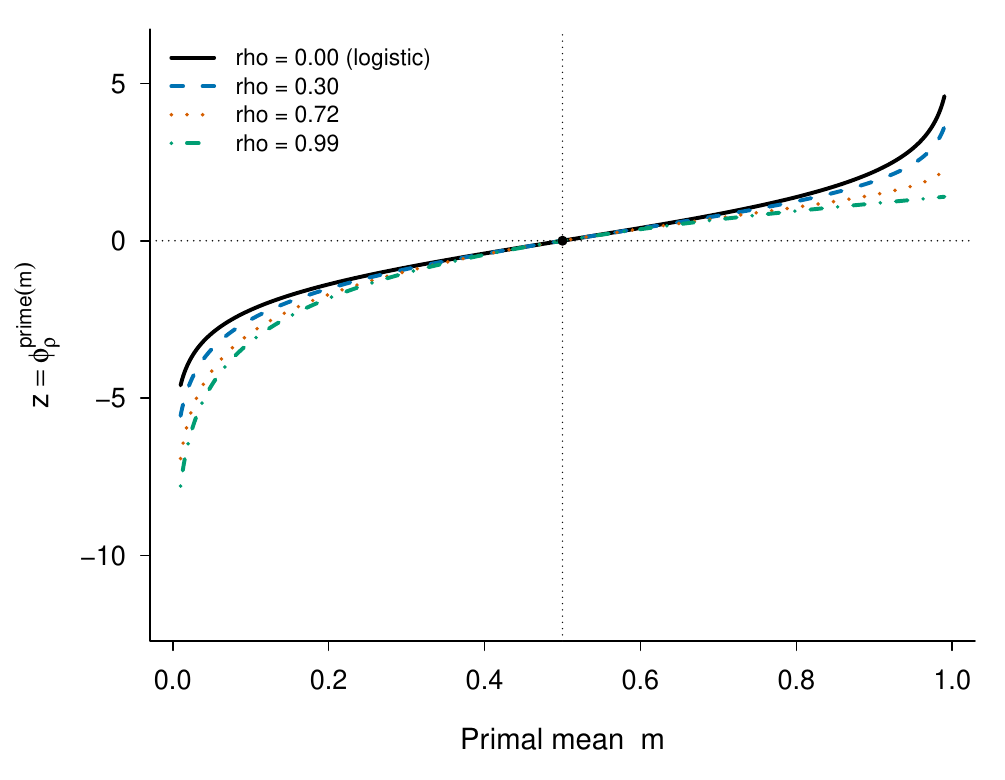}
        \small (b) Dual map
    \end{minipage}
    \caption{
    Elicitation geometry of the asymmetric weighted negative-entropy
    family. Panel (a) shows $\varphi_\rho$, while panel (b) shows the
    dual map $z=\varphi_\rho'(m)$. The induced link
    $g_\rho=(\varphi_\rho')^{-1}$ is shown in
    Figure~\ref{fig:numerical-summary}(e).
    }
    \label{fig:app-elicitation-geometry}
\end{figure}

Because every $g_\rho$ is strictly increasing, changing $\rho$ leaves
the optimal decision boundary unchanged. It can, however, alter the
population dynamics used to learn that boundary. Recall that the
decision stability coefficient is
$\lambda_\star=\Delta_\beta\mu^\star
-2\phi_0(0)(m_2)_+(1-\Delta_\beta^2/4)$.
The first term captures the restoring contribution associated with link
sensitivity, whereas the second captures the destabilizing contribution
from positive population-averaged curvature. The effect of elicitation
geometry can therefore be particularly pronounced under weak arm
separation, where the restoring signal is small while the curvature
contribution need not be.
Figure~\ref{fig:numerical-summary}(f) illustrates this behavior in the
weak-separation regime. At the reference value $\Delta_\beta=0.2$, the
choices $\rho=0.60,0.72,0.80$ give
$\lambda_\star\approx3.70\times10^{-3},
2.15\times10^{-5},-2.45\times10^{-3}$, respectively, providing stable,
near-critical, and unstable examples. As $\Delta_\beta$ increases, the
restoring contribution strengthens, while the factor
$1-\Delta_\beta^2/4$ multiplying the curvature contribution decreases.
The instability in this family is therefore concentrated in the
weak-separation regime rather than being an intrinsic property of the
link.

Figure~\ref{fig:app-stability-components} makes this balance explicit at
$\Delta_\beta=0.2$ as $\rho$ varies. For visual comparison, we plot the
two half-scaled contributions determining the sign of $\lambda_\star$:
the restoring contribution $(\Delta_\beta/2)\mu^\star$ and the curvature
contribution
$\phi_0(0)(m_2)_+(1-\Delta_\beta^2/4)$. The restoring contribution
changes relatively slowly, whereas the positive-curvature contribution
increases substantially with $\rho$. They coincide near $\rho=0.72$,
where $\lambda_\star=0$: beyond this point the curvature contribution
dominates and the true boundary becomes locally unstable.

\begin{figure}[h]
    \centering
    \includegraphics[width=0.42\textwidth]
    {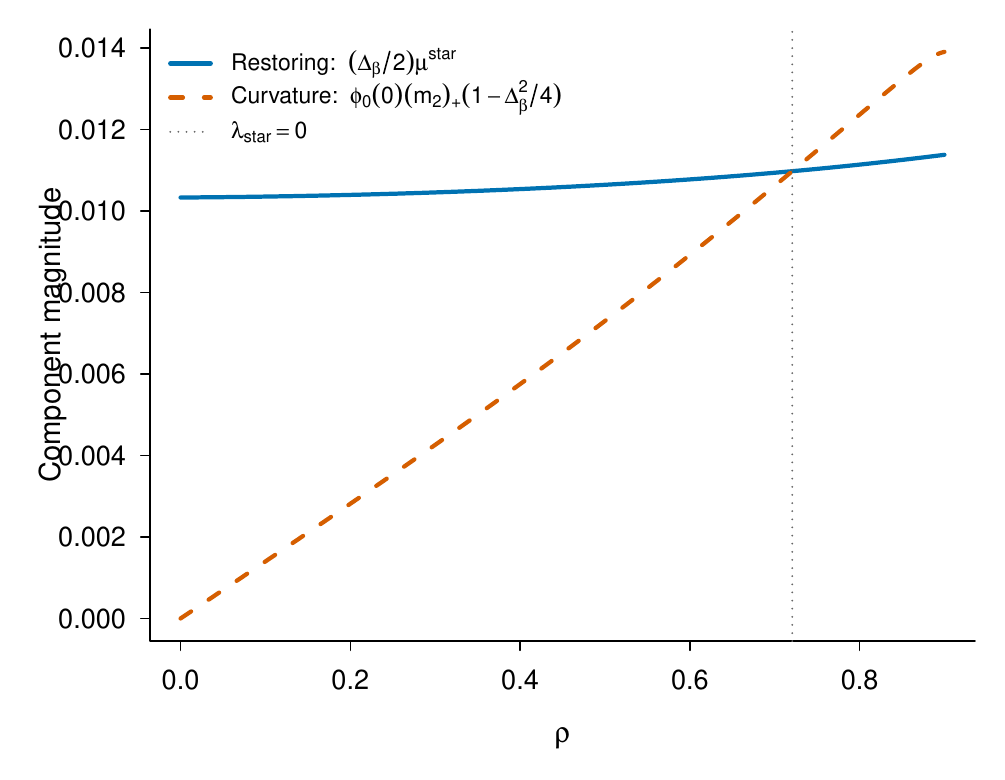}
    \caption{
    Components determining the sign of the decision stability coefficient
    at $\Delta_\beta=0.2$. For visual comparison, the figure shows the
    half-scaled restoring contribution $(\Delta_\beta/2)\mu^\star$ and
    curvature contribution
    $\phi_0(0)(m_2)_+(1-\Delta_\beta^2/4)$.
    Their equality corresponds to $\lambda_\star=0$.
    }
    \label{fig:app-stability-components}
\end{figure}

\subsubsection{A flexible unknown link}
\label{app:tanh-geometry}

The preceding construction varies the link through a known elicitation
family. In our bandit model, however, the shared link $g$ is unknown and
is not assumed to belong to a parametric family. To illustrate decision
stability in this more general setting, we generate rewards from the
flexible two-transition link
\begin{align}\label{eq:link_tanh}
g(z)
=
p_0
+
a_1\{1+\tanh[b_1(z-\mu_1)]\}
+
a_2\{1+\tanh[b_2(z-\mu_2)]\},
\end{align}
with $p_0=0.10$, $a_1=0.10$, $a_2=0.25$, $b_1=2$, $b_2=1.4$, and
$\mu_1=-1.25$, and vary only the location $\mu_2$ of the second
transition. This functional form is used only to generate the
ground-truth regression function; the shared link is treated as unknown
by NBL, which neither specifies nor estimates this parametric form.

Unlike the elicitation-induced family, this construction is not designed
to preserve a common value or slope at a particular index value.
Instead, changing $\mu_2$ moves the second transition across the index
space and produces different sensitivity and curvature profiles. Since
$a_j,b_j>0$, every link in the family is strictly increasing and
therefore induces the same form of linear optimal decision boundary,
while its curvature can change sign across the index space.
At the reference separation $\Delta_\beta=0.2$, we use
$\mu_2=-1.074037$, $2.144072$, and $1.244122$ to obtain stable,
near-critical, and unstable local population dynamics, respectively.
The resulting links and curvature profiles are shown in
Figure~\ref{fig:numerical-summary}(a)--(b). These labels refer only to
the local dynamics at $\Delta_\beta=0.2$ and are not intrinsic
properties of the links.

To understand why links with the same monotonicity can produce different
local dynamics, recall that
$U=\theta^\top X\sim N(0,1-\Delta_\beta^2/4)$ and
$Z=(v^\star)^\top X\sim N(0,1)$ are independent. The link enters the
decision stability coefficient through
$\mu^\star=\E[g'(U+\Delta_\beta Z/2)\mathbf 1\{Z>0\}]$ and
$m_2=\E\{g''(U)\}$. Thus stability is not determined by the pointwise
magnitude of either $g'$ or $g''$. Rather, it depends on the balance
between link sensitivity and curvature after these quantities are
weighted under the relevant population distributions. In particular,
when $m_2\leq0$ the positive-curvature penalty in $\lambda_\star$
vanishes, whereas sufficiently large positive population-averaged
curvature can oppose the restoring contribution associated with
$\mu^\star$.

The same construction allows us to examine the interaction between link
geometry and arm separation more systematically. We vary $\mu_2$ over a
continuum while simultaneously varying $\Delta_\beta$, keeping all
remaining parameters of the link fixed. Figure~\ref{fig:app-tanh-phase}
shows the resulting decision-stability regions. The shaded region
corresponds to $\lambda_\star<0$, while the solid curve marks the
boundary $\lambda_\star=0$. The horizontal dotted line indicates the
reference separation $\Delta_\beta=0.2$ used in the main experiments.
\begin{figure}[t]
    \centering
    \includegraphics[width=0.45\textwidth]
    {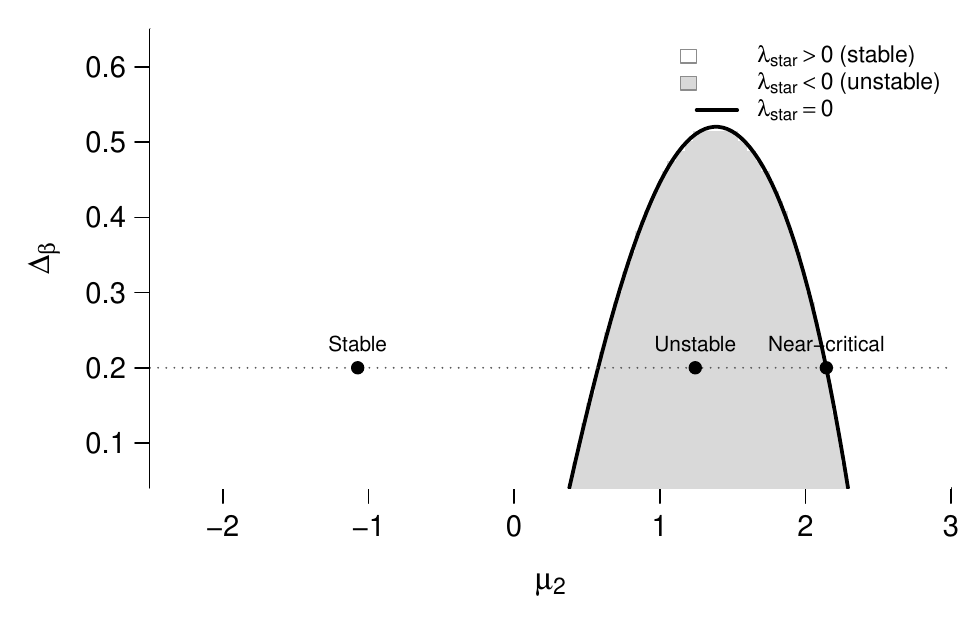}
    \caption{
    Decision-stability phase diagram for the two-transition shared-link
    family. The shaded region corresponds to $\lambda_\star<0$, while
    the solid curve marks the stability boundary $\lambda_\star=0$.
    The dotted horizontal line indicates the reference separation
    $\Delta_\beta=0.2$, and the three points correspond to the
    representative links used in the main numerical experiments.
    }
    \label{fig:app-tanh-phase}
\end{figure}
At the reference separation, moving the second transition through the
index space can move the population dynamics into and out of the
unstable region. This dependence is nonmonotone: instability occurs
only over an intermediate range of $\mu_2$ for this family. The phase
diagram also shows that the stability of a fixed link depends on the
arm separation. As $\Delta_\beta$ increases, the unstable region
contracts and eventually disappears over the displayed family. This
behavior is consistent with Figure~\ref{fig:numerical-summary}(c),
where the three representative links become locally stable as the arm
separation increases.

To see more directly how arm separation changes the population geometry,
we hold fixed the link with $\mu_2=1.244122$, which is locally unstable
at $\Delta_\beta=0.2$, and vary $\Delta_\beta$. The functions $g'$ and
$g''$ remain unchanged, while the population quantities entering
$\lambda_\star$ change with the arm separation. In particular,
$U\sim N(0,1-\Delta_\beta^2/4)$, so increasing $\Delta_\beta$
concentrates the distribution of $U$ around zero.
Figure~\ref{fig:tanh-changing-separation} overlays the fixed sensitivity
and curvature profiles with the distributions of $U$ corresponding to
$\Delta_\beta=0.2$, $1.0$, and $1.8$.
\begin{figure}[h!]
    \centering
    \begin{minipage}[t]{0.42\textwidth}
        \centering
        \includegraphics[width=\textwidth]
{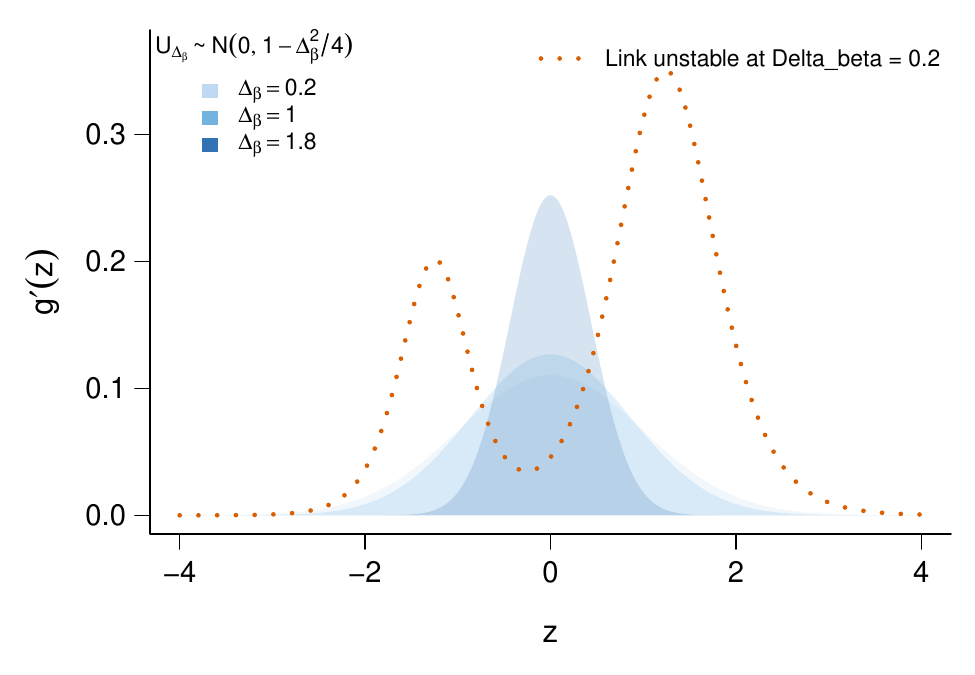}
        \small (a) Link sensitivity $g'(z)$
    \end{minipage}
    \hfill
    \begin{minipage}[t]{0.42\textwidth}
        \centering
        \includegraphics[width=\textwidth]
{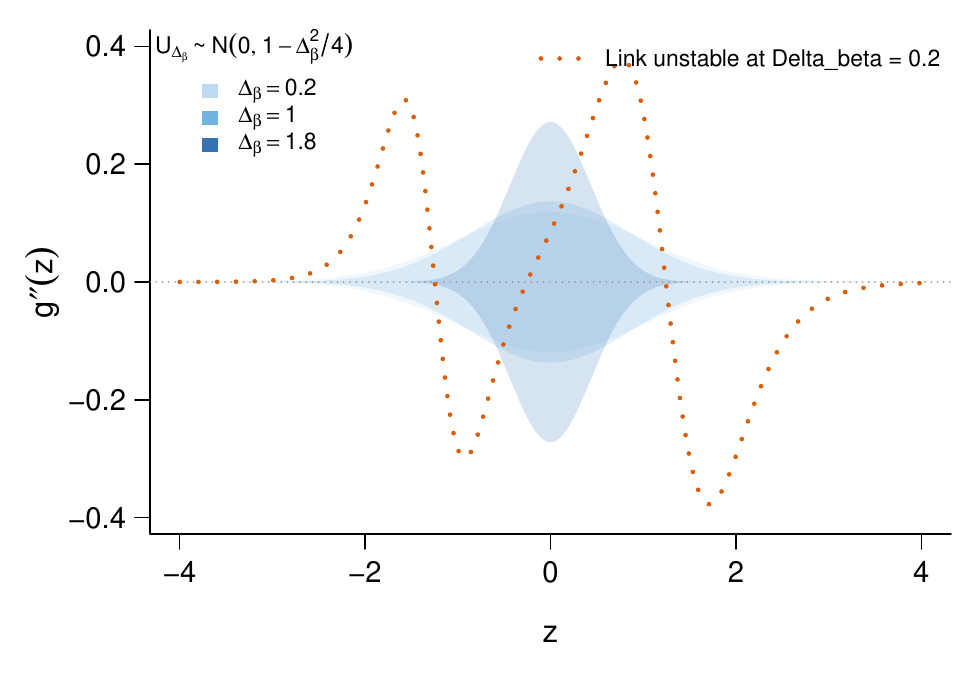}
        \small (b) Link curvature $g''(z)$
    \end{minipage}
    \caption{
    Effect of arm separation on the population weighting of a fixed
    unknown link. The link is locally unstable at
    $\Delta_\beta=0.2$ and is held fixed as $\Delta_\beta$ varies.
    The shaded regions correspond to
    $U\sim N(0,1-\Delta_\beta^2/4)$ for
    $\Delta_\beta=0.2$, $1.0$, and $1.8$, and are rescaled by a
    common factor for visual comparison. In panel (b), the densities
    are reflected about zero only for visual reference.
    }
    \label{fig:tanh-changing-separation}
\end{figure}
For curvature, the interpretation is direct: since
$m_2=\E\{g''(U)\}$, increasing $\Delta_\beta$ places greater weight
on the behavior of $g''$ near zero and less weight on curvature farther
away. Thus the population-averaged curvature can change even though
$g''$ itself remains fixed. The sensitivity panel has a related but
distinct interpretation. Since
$\mu^\star=\E[g'(U+\Delta_\beta Z/2)\mathbf 1\{Z>0\}]$, changing
$\Delta_\beta$ affects $\mu^\star$ both through the distribution of
$U$ and through the shift $\Delta_\beta Z/2$. Thus the density shown
in panel (a) illustrates only one component of the population weighting
entering $\mu^\star$.

These changes occur together with the explicit dependence of
$\lambda_\star$ on arm separation. Increasing $\Delta_\beta$ directly
strengthens the sensitivity contribution through the factor
$\Delta_\beta$ and reduces the multiplier
$1-\Delta_\beta^2/4$ on the curvature contribution, while also changing
$\mu^\star$ and $m_2$ themselves. This explains why the same fixed link
can be locally unstable under weak arm separation and locally stable
when the arm directions are more widely separated.
The next subsection examines how these population-level differences are
reflected in the stochastic NBL recursion.

\subsection{Stochastic Dynamics and Arm Separation}
\label{app:stochastic-nbl}

The preceding illustrations describe the population geometry underlying
NBL. We next examine whether these local stability regimes remain visible
in the stochastic recursion driven by the noisy sequential Stein contrast.
We use the same three two-transition links and first fix
$\Delta_\beta=0.2$, so that they correspond to the stable, near-critical,
and unstable population regimes identified above. To isolate the local
dynamics, each replication is initialized with a signed angular
displacement of $10^\circ$ from $v^\star$. We run NBL for $n=10{,}000$
rounds using $\eta_t=\gamma/(t+t_0)$ with $\gamma=10$ and $t_0=500$,
and average the results over 100 Monte Carlo replications.

Since these experiments use $d=2$, we measure directional error by the
signed angle from $v^\star$ to $v_t$. Writing
$v^\star=(v_1^\star,v_2^\star)^\top$, we define
\[
\alpha_t
=
\operatorname{atan2}
\left(
v_1^\star v_{t,2}-v_2^\star v_{t,1},
(v^\star)^\top v_t
\right),
\]
and report $180\alpha_t/\pi$ in degrees. Thus $\alpha_t=0$ corresponds
to the true boundary direction, while the sign records the direction of
the displacement around $v^\star$.
\begin{figure}[h]
    \centering
    \begin{minipage}[t]{0.48\textwidth}
        \centering
        \includegraphics[width=\textwidth]
        {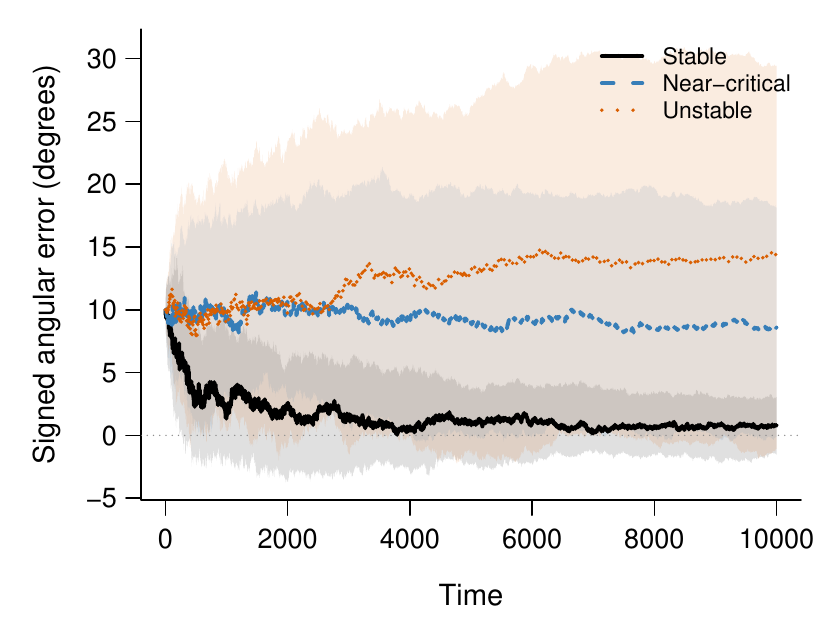}
        \small (a) Varying link geometry
    \end{minipage}
    \hfill
    \begin{minipage}[t]{0.48\textwidth}
        \centering
        \includegraphics[width=\textwidth]
        {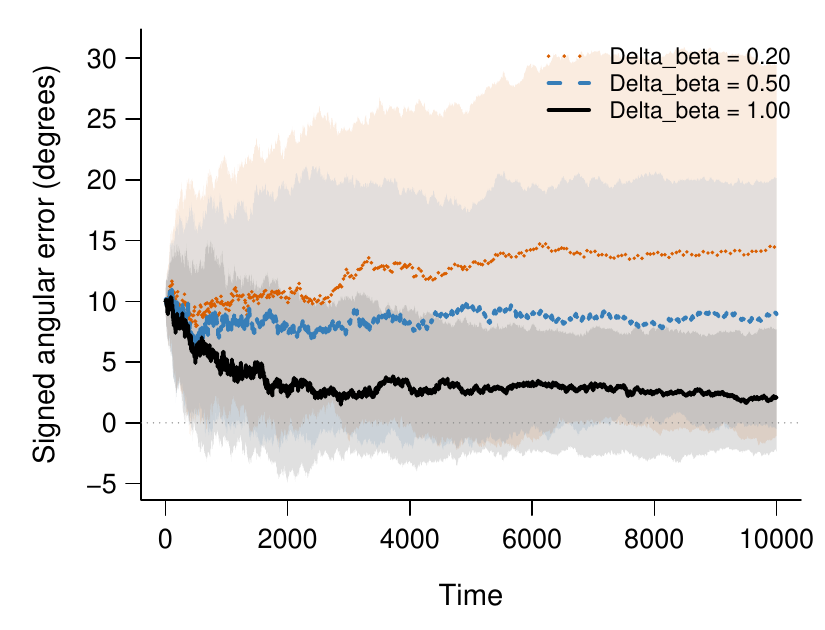}
        \small (b) Varying arm separation
    \end{minipage}
    \caption{
    Signed angular error of NBL over 100 Monte Carlo replications.
    Panel (a) fixes $\Delta_\beta=0.2$ and varies the shared link across
    the stable, near-critical, and unstable population regimes. Panel
    (b) fixes the link with $\mu_2=1.244122$, which is unstable at
    $\Delta_\beta=0.2$, and varies the arm separation. Curves show
    Monte Carlo means and shaded regions show pointwise approximate
    95\% Monte Carlo confidence intervals for the mean trajectories.
    }
    \label{fig:app-stochastic-angle}
\end{figure}
Figure~\ref{fig:app-stochastic-angle} examines the stochastic boundary
dynamics from two complementary directions. Panel (a) fixes
$\Delta_\beta=0.2$ and varies the shared link to be as specified in Section~\ref{app:tanh-geometry}. In the stable regime,
the boundary estimate is driven toward $v^\star$; near criticality, the
angular error remains persistently away from zero; and in the unstable
regime, the estimate moves progressively farther from $v^\star$.
Thus the attracting, near-critical, and repelling behaviors predicted
by the population geometry remain visible under stochastic updating.

Panel (b) instead fixes the link to be same as \eqref{eq:link_tanh} with $\mu_2=1.244122$ and varies the
arm separation. The same link that moves away from $v^\star$ at
$\Delta_\beta=0.2$ exhibits increasingly strong recovery toward the
true boundary direction as $\Delta_\beta$ increases. This agrees with
the population phase diagram in Figure~\ref{fig:app-tanh-phase} and
shows that the dependence of decision stability on arm separation is
also visible in the finite-sample NBL recursion.

The corresponding regret behavior is shown in
Figure~\ref{fig:numerical-summary}(d) for the varying-link experiment.
Under stable dynamics, average regret $R_t/t$ decreases, whereas
near-critical and unstable dynamics exhibit increasingly persistent
regret. Together, the population and stochastic experiments show how
the geometry captured by $\lambda_\star$ translates into both
directional learning of the boundary and its resulting decision
performance.

\subsection{Additional Comparisons with Greedy Logistic}
\label{app:logistic-comparison}

We next compare NBL with Greedy Logistic, a parametric benchmark based on
greedy generalized linear contextual-bandit learning
\cite{bastani2021mostly}. Following a common randomized initialization,
Greedy Logistic fits a separate logistic regression for each arm and
selects the arm with the larger fitted conditional mean. Since the
logistic link is strictly increasing, this is equivalent to selecting
the arm with the larger fitted linear score. In contrast, NBL does not
estimate either arm-specific reward function and instead updates the
decision boundary directly.
We use $d=2$, $\Delta_\beta=1$, $n=2000$, and a randomized
initialization of $n_0=200$ rounds. For NBL, we use
$\eta_t=10/(t+500)$, and results are averaged over 100 Monte Carlo
replications. Unless otherwise stated, we generate contexts independently as
$X_t\sim N(0,I_d)$. Conditional on $X_t$, the potential rewards are
generated independently as
$Y_{t,a}\sim\operatorname{Bernoulli}\{g(\beta_a^\top X_t)\}$ for
$a\in\{-1,+1\}$, so that
$\E(Y_{t,a}\mid X_t)=g(\beta_a^\top X_t)$. The observed reward is
$Y_t=Y_{t,A_t}$. For comparisons between algorithms, each Monte Carlo
replication uses the same generated contexts and potential rewards for
both methods, so that NBL and Greedy Logistic are evaluated on a common
Bernoulli reward realization.  Within each replication, the two methods use the same
 randomized initialization.

We first consider the elicitation-induced family
$g_\rho=(\varphi_\rho')^{-1}$ with
$\rho\in\{0,0.2,0.4,0.6,0.8\}$, where $\rho=0$ is exactly logistic.
This provides a controlled sequence of departures from the working link
used by Greedy Logistic. Figure~\ref{fig:numerical-summary}(g) shows
that Greedy Logistic has lower regret when the link is logistic or close
to logistic, reflecting the benefit of exploiting the correctly or
approximately specified parametric reward model. As the link departs
further from logistic geometry, this advantage diminishes and NBL can
have lower regret.

We then consider the three two-transition links in Section~\ref{app:tanh-geometry} used in the population
experiments. Although these links were selected to have different
stability behavior at $\Delta_\beta=0.2$, all three have positive
decision stability coefficients at the well-separated value
$\Delta_\beta=1$ used here. The comparison therefore isolates link
misspecification without placing NBL in the unstable regime.
Figure~\ref{fig:numerical-summary}(h) shows that Greedy Logistic can
retain an advantage for a non-logistic link that remains sufficiently
well approximated by its working model, while NBL can have lower regret
for larger departures from logistic geometry.

These comparisons illustrate the distinction between parametric reward modeling and direct boundary learning. When the logistic specification is correct or provides a good approximation, Greedy Logistic can exploit the additional structure and achieve lower finite-sample regret. NBL does not obtain the same parametric efficiency, but its regret remains competitive in these experiments despite neither specifying nor estimating the shared link. As the link departs further from the logistic specification, the parametric advantage diminishes, and under sufficiently strong misspecification, directly learning the decision boundary can yield lower regret. Thus NBL pays a relatively modest price for link agnosticism in the settings considered here, while avoiding sensitivity to a particular parametric specification.
\end{document}